\documentclass[a4paper,12pt]{article}
\usepackage[utf8]{inputenc}
\usepackage{amsmath,amsfonts,stmaryrd,amssymb,bbm,mathrsfs,mathabx}

\usepackage{comment}
\usepackage{aligned-overset}
\usepackage{amsthm}
\usepackage{mathtools}
\usepackage{xcolor}
\usepackage{enumitem}
\usepackage{hyperref}
\usepackage{nicematrix}
\usepackage{todonotes}
\usepackage{verbatim}
\usepackage{caption}
\usepackage{subfigure}
\usepackage{graphicx}
\usepackage{bm}
\usepackage{setspace}
\usepackage{float}
\usepackage[toc,page]{appendix}
\usepackage{xparse}

\usepackage{tcolorbox}
\newif\ifhighlighton
\highlightonfalse

\let\underbrace\LaTeXunderbrace

\ifhighlighton
  \newenvironment{highlight}
  {\begin{tcolorbox}[colback=yellow!10, colframe=black, sharp corners]}
  {\end{tcolorbox}}
\else
  \newenvironment{highlight}{}{}
\fi

\usepackage{pdfrender}

\usepackage{geometry}
\hypersetup{
    colorlinks,
    citecolor=blue,
    linkcolor=blue,
    urlcolor=blue
}

\usepackage{notation}

\usepackage[backend=biber, style=numeric,citestyle=numeric]{biblatex}
\definecolor{cc}{RGB}{0,0,0} 
\definecolor{subj}{RGB}{0,0,0} 
\definecolor{rev1}{RGB}{0,0,0} 

\newtheorem{lemma}{Lemma}[section]
\newtheorem{theorem}{Theorem}[section]
\newtheorem{corollary}{Corollary}[section]
\newtheorem{remark}{Remark}[section]
\newtheorem{definition}{Definition}[section]

\newtheorem{example}{Example}[section]

\makeatletter
\DeclareRobustCommand{\varvdots}{%
  \vbox{
    \baselineskip4\p@\lineskiplimit\z@
    \kern-\p@
    \hbox{.}\hbox{.}\hbox{.}
  }
}
\makeatother

\title{Recovering Governing Equations from Solution Data: Identifiability Bounds for Linear and Nonlinear ODEs}
\author{Yang Pan \\ yangpan@mins.ee.ethz.ch \and Helmut Bölcskei\\ hboelcskei@ethz.ch}
\date{\today}

\begin{document} 

\maketitle
\begin{abstract}
\textcolor{rev1}{Learning governing equations from observed solution data is a fundamental challenge in scientific machine learning \cite{bruntonDiscoveringGoverningEquations2016,kovachkiNeuralOperatorLearning2023,longPDENetLearningPDEs2018,rudyDatadrivenDiscoveryPartial2017,raonicConvolutionalNeuralOperators2023}, yet the theoretical conditions under which a ground-truth ODE can be uniquely and stably identified from multiple solution observations remain largely undeveloped, and no quantitative analysis of the sample complexity of such learning tasks exists in the literature. To address this gap, we introduce the Hausdorff distance on solution sets as the natural metric for comparing differential equations, since it captures the worst-case separation between two equations over all admissible initial conditions and thus encodes the minimax structure of the identification problem. We establish identifiability bounds for governing ODEs across a wide class of structure equations — ranging from linear ODEs to nonlinear classes with Lipschitz (Hölder)-continuous vector fields — characterizing precisely when two distinct equations can be distinguished from solution data. Using this metric, we derive metric entropy estimates for the relevant ODE classes and analyze sample complexity bounds, quantifying how many solution observations are needed to reliably recover the governing equation.}
\end{abstract}

\section{Introduction}
Learning physics (PDEs) and discovering governing equations from \textcolor{cc}{solution data} is a topic of significant interest \cite{bruntonDiscoveringGoverningEquations2016,kovachkiNeuralOperatorLearning2023,longPDENetLearningPDEs2018,rudyDatadrivenDiscoveryPartial2017,raonicConvolutionalNeuralOperators2023}. The current literature focuses mostly on algorithms to achieve this goal \cite{bruntonDiscoveringGoverningEquations2016,rudyDatadrivenDiscoveryPartial2017,longPDENetLearningPDEs2018}. Uniqueness issues in finding the equations are discussed in \cite{schollUniquenessProblemPhysical2023,hollerUniquenessStructuredModel2024,schollWelldefinednessPhysicalLaw2023} to a limited extent. Specifically, \cite{bruntonDiscoveringGoverningEquations2016,rudyDatadrivenDiscoveryPartial2017,hollerUniquenessStructuredModel2024} exploit regularization minimization to search for the structure equation with ``minimal energy'' that fits the solution data, with existence and uniqueness analyzed in \cite{hollerUniquenessStructuredModel2024}. Equivalent and practical conditions for one function (solution) to uniquely determine its underlying differential equation are derived in \cite{schollUniquenessProblemPhysical2023,schollWelldefinednessPhysicalLaw2023} for linear/polynomial/algebraic PDEs. However, \textcolor{rev1}{theoretical foundations--under what conditions the ground-truth governing equation can be uniquely and stably identified from multiple solution observations--remain undeveloped.}
Moreover, a quantitative analysis on the complexity in learning equations, e.g., \textcolor{black}{sample complexity} \textcolor{rev1}{(the amount of solution observations)}, does not exist in the literature. A theoretical foundation is thus still missing.

 This paper intends to \textcolor{cc}{fill the gaps mentioned above} for a wide class of ODEs as a starting point for more work to be done in the future for general PDEs. Specifically, we manage to find out when uniqueness is guaranteed in discovering governing equations and provide a quantitative analysis of complexity in such learning tasks. \textcolor{rev1}{To make the identification problem precise, we measure the distance between two ODEs by the Hausdorff distance of their solution sets over a prescribed set of initial conditions. This is the natural choice: it captures the worst-case separation between solution trajectories, reflecting the minimax structure of the identification problem.}

    \textcolor{rev1}{Based on the prescribed distance measure, we set out to consider specific classes of ODEs--linear, Lipschitz, Hölder, and polynomial ODEs, establish identification bounds by proving upper and lower bounds on Hausdorff distance between ODEs by the distance between the underlying structure equations, and then use identification bounds to derive complexity results in learning equations based on sample complexity (i.e., how many samples we need to learn the equation) and metric entropy \cite{tikhomirovEEntropyECapacitySets1993} (how ``rich'' the class of ODEs is in terms of the Hausdorff distance). We summarize in Table \ref{tbl:hd_ex} representative examples of lower and upper bounds on the Hausdorff distance between ODEs for several ODE classes considered in this work.}

\renewcommand{\arraystretch}{2.5}
\begin{table}[h]
\centering
\resizebox{\columnwidth}{!}{
\begin{tabular}{|l|l|l|l|}
    \hline
 ODE class & Hausdorff distance between & Lower bound & Upper bound\\
 \hline
 Linear ODEs  & $\dot{x}=Ax$ ~~~~\&~~ $\dot{x}=\tilde{A}x$ & $\bigo{\norm{A-\tilde{A}}_2}$  & $\bigo{\norm{A-\tilde{A}}_2}$\\
 Lipschitz ODEs & $\dot{x}=f(x)$ ~~\&~~ $\dot{x}=\tilde{f}(x)$ & $\bigo{\normFun[2]{f-\tilde{f}}{L^{\infty}}^2}$&$\bigo{\normFun[2]{f-\tilde{f}}{L^{\infty}}}$ \\
       \hline
\end{tabular}
}
\vspace{0.5em}
\caption{Lower and upper bounds on Hausdorff distance within linear ODE class or Lipschitz ODE class.}
\label{tbl:hd_ex}
\end{table}
 

\textit{Organization of the paper.} The paper is organized as follows. Section \ref{sec:literature} reviews related work in the literature. In Section \ref{sec:setup}, we formalize the setup of learning structure equations of ODEs and introduce Hausdorff distance for ODEs. Section \ref{sec:hd_bound_ode} is devoted to identification results for a wide class of ODEs based on Hausdorff distance. In Section \ref{sec:complexity}, we analyze the complexity of learning ODEs in terms of sample complexity and metric entropy.

\subsection{Notation}

$\Nzero$ and $\Nplus$ denote the set of natural numbers including and excluding $0$, respectively. $\R^{d\times d}$ stands for the set of $d$ by $d$ matrices with real entries.
The transpose of the matrix $A$ is $A^T$. The $N\times N$ identity matrix is $\mathbb{I}_N$. 
For the vector $x\in \R^d$, we let $\norm{x}_2 \defeq \sqrt{\sum_{i=1}^d|x_i|^2}$. For the matrix $A\in \R^{d\times d}$, define $\norm{A}_2 = \sup_{\norm{x}_2=1}\norm{Ax}_2$ and $\norm{A}_F = \sqrt{\sum_{i=1}^d\sum_{j=1}^d |A_{ij}|^2}$. Additionally, denote $\sig_1(A)\geq \sig_2(A)\geq \dots\geq \sig_n(A)$ as the singular values of $A$. Moreover, $\eig_1(A),\eig_2(A),\dots, \eig_n(A)$ are the eigenvalues of $A$ ordered such that $|\eig_1(A)|\geq |\eig_2(A)|\geq \dots\geq |\eig_n(A)|$. $\specRad{A} \defeq |\eig_1(A)|$ is the spectral radius of $A$. For a nonsingular matrix $A$, the condition number induced by $\norm{\cdot}$ is $\conNum{A}{\norm{\cdot}}\defeq \norm{A} \cdot\norm{A^{-1}}$.
    $\log(\cdot)$ refers to the logarithm to base $2$, $\log^{(n)} = \log\circ \cdots \circ \log$ is the $n$-fold iterated logarithm, and $\log^\tau(\cdot)=\left(\log(\cdot)\right)^\tau$, for $\tau\in \R$. 
     For a set $X\subset \R$, we define $C^{k}(X)=\{f:X\rightarrow \R\mid f \text{ has } k\text{-th continuous derivatives}\}$ and $\normFun{f}{C^{k}(X)} = \sum_{i=0}^k \sup_{x\in X} |f^{(i)}(x)|$, \textcolor{cc}{where $f^{(i)}$ denotes the $i$-th derivative of $f$}. When $k=0$, we use $\normFun{f}{L^{\infty}(X)}$ instead of $\normFun{f}{C^{0}(X)}$. For a function $f:\R^d\rightarrow\R$, denote $\operatorname{supp}(f)=\{x\in \R^d\mid f(x)\neq 0\}$ as the support set of $f$. Let $f(\epsilon)$ and $g(\epsilon)$, in both cases for $\epsilon > 0$, be strictly positive for all small enough values of $\epsilon$. We use $f(\epsilon) = \smallo{g(\epsilon)}$ to indicate that $\lim_{\epsilon \rightarrow 0} \frac{f(\epsilon)}{g(\epsilon)}=0$ and we express $\limsup_{\epsilon \rightarrow 0} \frac{f(\epsilon)}{g(\epsilon)} < \infty$ by $f(\epsilon) = \bigo{g(\epsilon)}$. Moreover, we write $f(\epsilon)\lesssim g(\epsilon)$ when $f(\epsilon)\leq \bigo{g(\epsilon)}$, and we write $f(\epsilon)\asymp g(\epsilon)$ when both $f(\epsilon) = \bigo{g(\epsilon)}$ and $g(\epsilon) = \bigo{f(\epsilon)}$. Constants are always understood to be in $\R$ unless explicitly stated otherwise.
    

\section{\textcolor{cc}{Related literature}}\label{sec:literature}
Across science and engineering, ``identifying a system'' can be viewed as the task of turning observations into a usable representation of how a phenomenon behaves.
A common way to frame this topic is the white/grey/black-box spectrum \cite{Wikipedia_si}: white-box models aim to express the internal mechanisms explicitly, grey-box models assume partial mechanistic structure but leave key components or parameters to be learned from data, and black-box models
instead learn an input–output relationship that is aimed for prediction or control.

A wide range of research topics and literature concerns the problem of ``identifying a system'', which can be broadly categorized according to the white-/grey-/black-box spectrum introduced above:

\begin{itemize}
    \item The white-box modeling includes 
    \begin{itemize}
        \item \emph{Sparse equation discovery} methods, such as SINDy \cite{bruntonDiscoveringGoverningEquations2016,rudyDatadrivenDiscoveryPartial2017}, which seek to recover explicit governing equations directly from data by constructing a candidate feature library and promoting sparsity in the identified dynamics.
        \item \emph{Symbolic regression} approaches \cite{udrescu2020aifeynman}, including equation-learning systems such as EQL \cite{martius2016extrapolationEQL,sahoo2018learningEQL}, which search over spaces of analytic expressions to produce closed-form, human-interpretable models.
    \end{itemize}

    \item The grey-box modeling includes 
    \begin{itemize}
        \item \emph{Inverse problems based on physics-informed neural networks} (PINNs) \cite{raissi2019physics,chen2021physics}, which incorporate known physical laws—typically expressed as differential-equation constraints or residuals—into the learning objective while estimating unknown functions or parameters from data.
        \item \emph{Classical system identification} \cite{ljung1998systemind} in its structured or parametric formulations (often referred to as grey-box or semi-physical identification), where the model structure is derived from domain knowledge and only a subset of parameters is learned from data.
    \end{itemize}

    \item The black-box modeling includes 
    \begin{itemize}
        \item \emph{Classical system identification} \cite{ljung1998systemind} in its black-box form, which aims at predictions through flexible input–output models without enforcing mechanistic interpretability.
        \item \emph{Operator learning} \cite{kovachkiNO2023neural,li2020fourier,raonic2024convolutional}, which aims to learn mappings between function spaces (e.g., from input coefficients, forcing terms, or boundary conditions to solution fields) in PDE settings.
    \end{itemize}
\end{itemize}
Although white-, grey-, and black-box models differ in how much structure is prescribed a priori, they share a common objective: to identify system behavior from measurements.
In this work, we operate in a hybrid viewpoint: we assume an ODE form 
\begin{equation*}
    \dot{x}=f(x),
\end{equation*} 
which is a ``white-box'' at the level of state-space ODE structure, while the choice of hypothesis class for the unknown vector field $f$ could be either a white-, grey-, or black-box model. Moreover, rather than designing a particular white-, grey-, or black-box parametrization for $f$, our focus is on the underlying identification problem itself: we analyze the uniqueness and stability of the identification process, namely, under what conditions $f$ can be uniquely recovered from solution data, and how sensitive this recovery is to perturbations in the data.


We address these questions by developing a theoretical framework that characterizes identifiability and stability independently of any specific parameterization of 
$f$. Based on this framework, we further analyze the complexity in the identification problem and motivate a loss function that could be used in identifying governing differential equations based on our frameworks. Depending on the structural assumptions imposed on $f$, this philosophy naturally applies to white-box, grey-box, or black-box modeling approaches. 

We hasten to add that questions of uniqueness and stability in system identification have been studied extensively in several mature theoretical fields, most notably in control theory and inverse problems for ODEs and PDEs. 

In control theory, identifiability is typically defined relative to a prespecified model structure \cite{grewal2003identifiability,ljung1994globalidentifiability}: one assumes that the dynamics (often an input–output or a state-space model) belong to a known parametric family—often linear time-invariant systems, or nonlinear systems with fixed functional form—and asks whether the unknown parameters can be uniquely determined from input–output or state–trajectory data under suitable conditions. In this setting, identifiability
is not formulated as the recovery of an arbitrary vector field $f$ from solutions of $$\dot{x}=f(x),$$but rather as parameter estimation within an assumed class of dynamics.

Likewise, the inverse problems for ODEs and PDEs
aim to identify the unknown coefficients, source terms, boundary or initial conditions, or other finite- or infinite-dimensional parameters appearing in a known governing differential equation \cite{isakov2006inverse,stuart2010inverse,nickl2023bayesian} and rigorous identifiability results in inverse problems almost always assume that the form of the differential operator is known a priori. 

However, to the best of our knowledge, there is relatively little theory addressing the identifiability and stability of recovering a general vector field $f$ directly from solution trajectories without assuming a fixed parameterization or model class. Our work is intended to complement these existing theories by studying identifiability at this more intrinsic, representation-independent level.



\section{Problem Setup}\label{sec:setup}
In the paper, we consider the problem of identifying the structure $f$ of an ODE 
\begin{equation}\label{eq:intro_ode}
    x^{(\odeorder)} = f\left(x^{(\odeorder-1)},\dots,x^{(1)},x\right)
\end{equation}
given access to its solution data $x(t)$ or a set of solutions $\{x(t)\mid x \text{ satisfies \eqref{eq:intro_ode}}\}$. The question is, if the solution data is given, is it possible to uniquely and stably identify $f$? Heuristically, if the solution data satisfy different ODEs in nature, then we cannot find an algorithm that recovers the unique ground-truth equation without further restrictions.

To formalize the matter, \textcolor{cc}{we first introduce the classes of differential equations considered. Then, we define the distance between different differential equations based on the Hausdorff distance between their solution sets, a key notion that lays the foundation for the analysis of the uniqueness and complexity of equation learning for the rest of this paper.}

\subsection{Classes of Ordinary differential equations}
In this paper, we focus on autonomous ODEs of the form 
\begin{equation}\label{eq:general_ode}
    x^{(\odeorder)} = f\left(x^{(\odeorder-1)},\dots,x^{(1)},x\right).
\end{equation}
To this end, we consider the class of autonomous ODEs in the form of \eqref{eq:general_ode} characterized by the class of functions.

\begin{definition}\label{def:general_ode_class}
    Given $\odedim,\odeorder \in \Nplus$, let $\mathcal{F}$ be a class of functions $f:\R^{\odeorder \odedim}\rightarrow\R^\odedim$, and let $x^{(\odeorder)} = f\left(x^{(\odeorder-1)},\dots,x^{(1)},x\right)$ be a differential equation of order $\odeorder$ \textcolor{cc}{and dimension $\odedim$}. 
    We define the corresponding class of ODEs, denoted as $\odeClass[\odeorder]{\mathcal{F}}$, to be the set of all such equations where $f\in \mathcal{F}$, that is 
    \begin{equation}\label{eq:general_ode_F_m}
        \odeClass[\odeorder]{\mathcal{F}}\defeq \left\{x^{(\odeorder)} = f\left(x^{(\odeorder-1)},\dots,x^{(1)},x\right) \mid f\in\mathcal{F} \right\}.
    \end{equation}
    When $\odeorder=1$, we simply write $\odeClass{\mathcal{F}}$. For notational simplicity, we write $\diffeq[f]\in \odeClass{\mathcal{F}}$ as an ODE of the form \eqref{eq:general_ode} with structure $f\in \mathcal{F}$, unless otherwise specified.
\end{definition}
There are different types of ODEs depending on the structure $f$. In this paper, we focus on the following classes of ODEs.
\begin{definition}\label{def:ode_classes}
    Given $\odedim,\odeorder\in \Nplus$, we consider the class of $\odeorder$-th order, $\odedim$-dimensional 
    \begin{enumerate}
        \item linear ODEs: 
        $\odeClass[\odeorder ]{\lin[\R^{\odeorder \odedim},\R^\odedim ]{K}}$, for $K>0$, where 
        \begin{equation}
            \begin{aligned}
               \lin[\R^{\odeorder \odedim},\R^\odedim ]{K}\defeq \{f:\R^{\odeorder \odedim}\rightarrow\R^\odedim \mid &
               f(x)=Ax=
               \begin{pmatrix}
                   A_{\odeorder -1}&\dots&A_1&A_0
               \end{pmatrix}x\\
               &A_i\in \R^{\odedim\times \odedim}, \norm{A_i}_2\leq K, i=0,1,\dots,\odeorder-1\}
            \end{aligned}
        \end{equation}
        Specifically, we write $\diffeq[A]\in \odeClass[\odeorder ]{\lin[\R^{\odeorder \odedim},\R^\odedim ]{K}}$ as an ODE of the form \eqref{eq:general_ode} with structure $f\in \mathcal{F}$ being $f(x)=Ax$.
        \item Lipschitz ODEs: $\odeClass[\odeorder]{\lip[\R^{\odeorder \odedim},\R^\odedim ]{L,K}}$, for $L,K>0$, where 
        \begin{equation}
        \begin{aligned}
            \lip[\R^{\odeorder \odedim},\R^\odedim ]{L,K}\defeq \{f:\R^{\odeorder \odedim}\rightarrow\R^\odedim \mid &\operatorname{Lip}(f)\defeq \sup_{x\neq y\in \R^{\odeorder \odedim}}\frac{\norm{f(x)-f(y)}_2}{\norm{x-y}_2}\leq L, \\
            &\norm{f(0)}_2\leq K\}.
        \end{aligned}
        \end{equation}    
        \end{enumerate}
    Moreover, we consider the class of first order, $1$-dimensional 
    \begin{enumerate}
        \item Hölder ODEs: $\odeClass{\holder{k,\alpha}{L,K}}$, for $k\in \Nzero$, $\alpha\in (0,1]$, $L,K>0$, where 
        \begin{equation}
        \begin{aligned}
            \holder{k,\alpha}{L,K}\defeq \{f:\R\rightarrow\R \mid f\in C^k(\R), \sup_{x\neq y\in \R}\frac{|f^{(k)}(x)-f^{(k)}(y)|}{|x-y|^{\alpha}}\leq L, \sum_{i=0}^k|f^{(i)}(0)|\leq K\};
        \end{aligned}
        \end{equation} 
        \item polynomial ODEs: $\odeClass{\poly{q,K}}$ where 
        \begin{equation}
            \poly{q,K}\defeq \{p:\R\rightarrow\R\mid p(x) = x^q +a_{q-1}x^{q-1}+\dots+a_0,~|a_i|\leq K \text{ for }i=0,1,\dots,q-1 \}.
        \end{equation}
    \end{enumerate}
\end{definition}

\subsection{Hausdorff distance for differential equation identification}\label{sec:hd_def}
Now, given prior knowledge about the type of ordinary differential equations (ODEs) under consideration, and assuming access to solution data $\{(t_i,x_k(t_i))_{i=1}^{M}\}_{k=1}^{N}$, can we uniquely and stably identify the governing equation \eqref{eq:general_ode}? If identifiability is achievable, what is the sample complexity, i.e., how many data points $M,N$ are required to learn $f$ within a prescribed accuracy $\epsilon$?

The first question has been discussed in \cite{rudyDatadrivenDiscoveryPartial2017, schollUniquenessProblemPhysical2023} in the PDE setting, where non-uniqueness issues are found. Specifically, the KdV evolution and the one-way wave equation share a set of solutions.
\begin{example}\label{ex:kdv}
    Consider the Korteweg–De Vries (KdV) equation
    \begin{equation*}
\signal_t + 6\signal\signal_x + \signal_{xxx}=0.
    \end{equation*}
    It is solved by $\signal(t,x) = c/2\operatorname{sech}^2(\sqrt{c}/2(x-ct-a))$. However, it also solves the \textcolor{cc}{one-way wave equation} $\signal_t+c\signal_x=0$.
\end{example}
This means that inferring the structure of a PDE from one solution or a set of solutions is potentially an ill-posed problem. The same issue lies in identifying the structure equation of an ODE, as can be seen in the following examples.
\begin{example}[Linear ODE]
    Let 
    \begin{equation}
        \begin{aligned}
            A &= \begin{pmatrix}
                1&0\\
                1&1
            \end{pmatrix},\\
            B &= \begin{pmatrix}
                1&0\\
                5&1
            \end{pmatrix}.
        \end{aligned}
    \end{equation}
    For $x_0 = c\begin{pmatrix}
        0&1
    \end{pmatrix}$, where $c\in \R$, we have that the solutions of $\dot{x}=Ax$ and $\dot{x}=Bx$ with $x(0)=x_0$ satisfy
    \begin{equation}
        \sup_{t\in [0,\horizon]} \norm{e^{At}x_0-e^{Bt}x_0}_2 = \sup_{t\in [0,\horizon]} \norm{e^{t}x_0-e^{t}x_0}_2\equiv 0.
    \end{equation}
    On the other hand, we have
    \begin{equation}
        \norm{A-B}_2 = 4.
    \end{equation}
\end{example}

\begin{example}
    Consider Lipschitz functions $f,\tilde{f}$ depicted in Figure \ref{fig:lip_ex}. 
    \begin{figure}[H]
    \centering
    \includegraphics[width=0.6\linewidth]{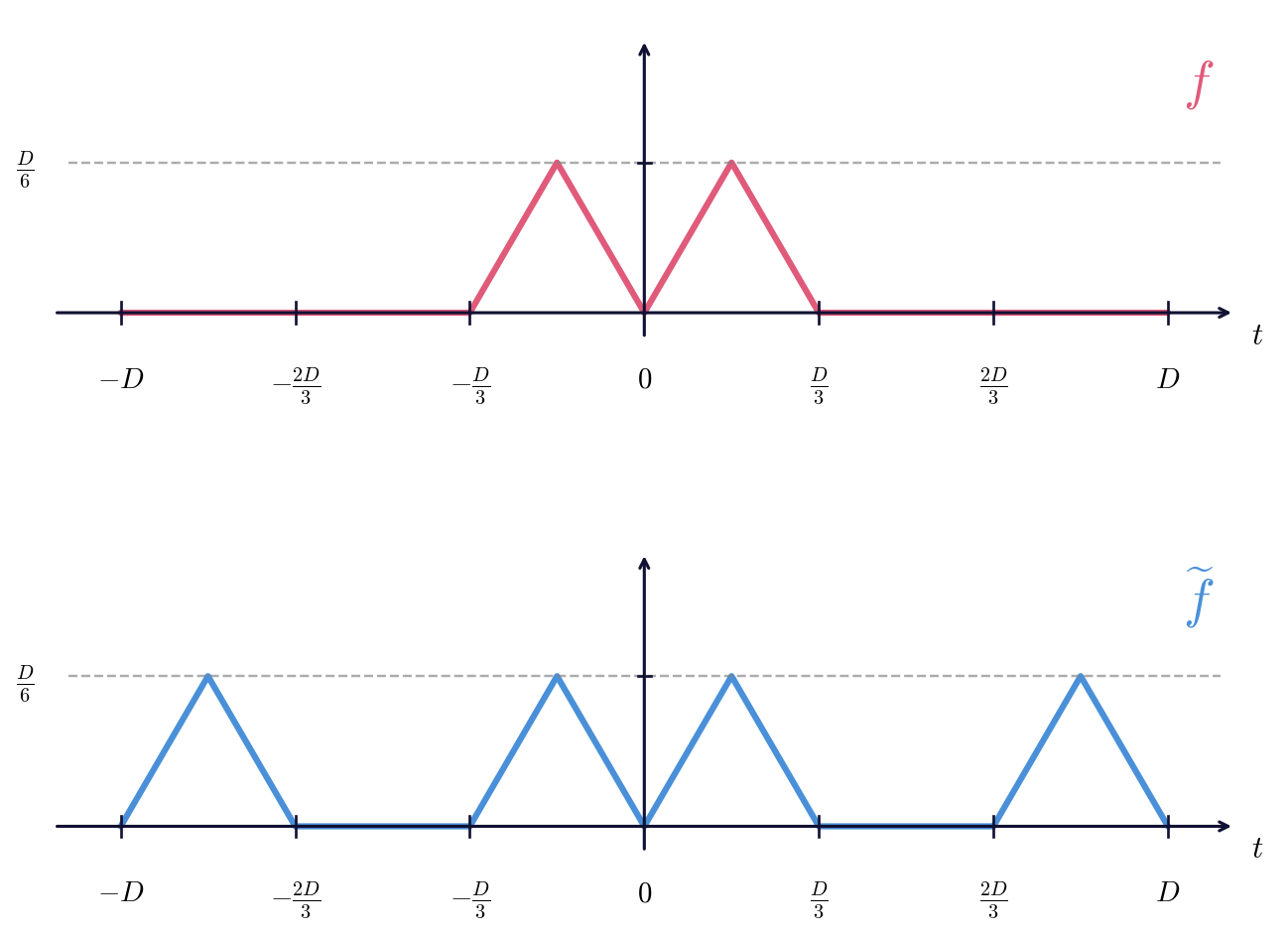}
    \caption{Lipschitz functions $f$ and $\tilde{f}$}\label{fig:lip_ex}
    \end{figure}
    We have $\normFun{f-\tilde{f}}{L^\infty([-D,D])} = D/6$. However, for all $|x_0|\leq D/3$, we have $x=\tilde{x}$, where $x$ and $\tilde{x}$ satisfy the ODEs
    \begin{equation}
        \begin{aligned}
            &\dot{x} = f(x),\quad x(0) = x_0,\\
            &\dot{\tilde{x}} = f(\tilde{x}),\quad \tilde{x}(0) = x_0.
        \end{aligned}
    \end{equation}
\end{example}

The above examples show that identifying the governing equation from solution data \textcolor{cc}{in a unique and stable manner} is generally ill-posed, even for linear ODEs. The essence of this issue lies in the fact that two differential equations could share a set of solutions.

This leads to the question that, if two differential equations share exactly the same solution sets, or their ``entire'' solution sets are close in terms of some specified distance measure, what is the implication for the structure equation? 

\textcolor{cc}{To address these questions, it is helpful to start with linear autonomous ODEs and ask when the structure matrix $A$ in $$\dot{x}=Ax$$ can be identified from solution trajectories. Assume, for simplicity, that $A$ is diagonalizable. Then $A$ is fully determined by its eigenvalues and eigenvectors. Concretely, suppose
\begin{equation*}
    A = V \Lambda V^{-1},~~V = \begin{pmatrix}
    v_1&v_2&\dots&v_n
\end{pmatrix},~~\Lambda = \operatorname{diag}(\lambda_1,\lambda_2,\dots,\lambda_n).
\end{equation*}
For an initial condition $x(0)=x_0$, the corresponding solution is
\begin{equation*}
    x(t) = Ve^{\Lambda t}V^{-1}x_0.
\end{equation*}
From the perspective of identification, recovering $A$ reduces to recovering $V$ and $\Lambda$. If $V$ were known, then each eigenvalue $\lambda_i$ could be revealed by initializing the system along the $i$-th eigenvector direction. Indeed, choosing $x_0=v_i=Ve_i$ yields
\begin{equation*}
    x(t) = Ve^{\Lambda t}e_i=e^{\lambda_i t}v_i.
\end{equation*}
So the trajectory is a single exponential mode, and the mode $\lambda_i$ can be read off directly. Therefore, if one can “excite” (i.e., initialize) the system at $\{v_1,v_2,\dots,v_n\}$, the matrix $A$ is identified. This observation also supports a natural uniqueness intuition: the solution set with initial conditions $\{v_1,v_2,\dots,v_n\}$ spans the full solution set of the linear ODE, and so it is reasonable to conjecture that the entire solution set strongly constrains the governing equation of the ODE. 
A concrete example is the spring (second-order linear) equation
\begin{equation*}
    \ddot{x}+2\zeta\omega_n\dot{x}+\omega_n^2x=0,
\end{equation*}
where $\omega_n$ is the undamped natural frequency and $\zeta$ is the damping ratio. Clearly, the solution to the spring equation can be written as superpositions of complex exponentials, i.e.,
$$x=Ae^{-\omega_nt\left(\zeta+\sqrt{\zeta^2-1}\right)}+Be^{-\omega_nt\left(\zeta-\sqrt{\zeta^2-1}\right)}.$$ Equivalently, by introducing the state $z = (x,\dot{x})^T$, one can rewrite the dynamics as a two-dimensional linear ODE $\dot{z}=Az$. The two exponential modes above are precisely the eigenvalues of that matrix $A$, so estimating those exponentials from trajectory data identifies $\omega_n$ and $\zeta$. In practice, one way to estimate such exponentials from sampled trajectories is to use ESPRIT \cite{roy2002esprit}.}

\textcolor{cc}{Nevertheless, this procedure is fundamentally structure-dependent: it relies on knowing a priori that the system is a linear, autonomous ODE of a fixed dimension, for which the solution decomposes into exponential modes. For general nonlinear ODEs, one typically lacks such a linear modal decomposition, so the questions are: how to characterize and compare the solution sets induced by different (nonlinear) ODEs, and what implication of this comparison has on the governing differential equations of the corresponding ODEs.
}

To this end, we first formally define the notion of solution set (behavior) for ODEs.



    \begin{definition}[Solution set (behavior)]
    Given $\odedim,\odeorder \in \Nplus$, let $\mathcal{F}$ be a class of functions $f:\R^{\odeorder \odedim}\rightarrow\R^\odedim$, and let 
    $\odeClass[\odeorder]{\mathcal{F}}$ be the corresponding class of ODEs. Fix $\horizon>0$. For 
    $f\in \mathcal{F}$ and the corresponding ODE $\diffeq[f] \in \odeClass[\odeorder ]{\mathcal{F}}$, we define its behavior as 
    \begin{equation*}
        \behavior{\diffeq[f]} = \{x\in (C^\odeorder ([0,\horizon]))^\odedim \mid x^{(\odeorder)} = f\left(x^{(\odeorder-1)},\dots,x^{(1)},x\right)\}.
    \end{equation*}
    Moreover, for the rest of this paper, we consider initial value problems (IVPs) for ODEs and define the behavior with respect to a set of initial conditions $\initset\subset \R^{\odeorder\odedim}$ according to 
    \begin{equation*}
    \begin{aligned}
        \behavior[\initset]{\diffeq[f]} = \{x\in (C^\odeorder ([0,\horizon]))^\odedim \mid &x^{(\odeorder)} = f\left(x^{(\odeorder-1)},\dots,x^{(1)},x\right),\\
        &\left((x(0))^T,\left(x^{(1)}(0)\right)^T,\dots, \left(x^{(\odeorder-1)}(0)\right)^T\right)^T\in \initset\}.
    \end{aligned}
    \end{equation*}
    \end{definition}

The following presents an example of behavior in terms of linear ODEs.
\begin{example}[Linear ODEs]
    Consider $A\in \R^{\odedim\times \odedim}$ and a first-order linear ODE $\diffeq[A]\in \odeClass{\lin[\R^d,\R^d]{K}}$. Then, its behavior with respect to initial conditions $\initset$ is 
    \begin{equation}
    \begin{aligned}
        \behavior[\initset]{\diffeq[A]} &= \{x\in (C([0,\horizon]))^\odedim \mid \dot{x} = Ax, x(0)\in \initset\}\\
        &=\{x\in (C([0,\horizon]))^\odedim \mid x(t)=e^{At}x_0, x_0\in \initset\}.
    \end{aligned}
    \end{equation}
\end{example}
Now, we define the Hausdorff distance for sets and the Hausdorff distance for ODEs based on their behavior as follows.
\begin{definition}[Hausdorff distance between sets]\label{def:hd_set}
    Given a metric space $(X,\normFun{\cdot}{X})$ and the sets $U,V \subset X$, we define the Hausdorff distance between $U$ and $V$ as 
    \begin{equation}
        \hdset[\normFun{\cdot}{X}]{U}{V} = \max\left\{\sup_{u\in U}\inf_{v\in V} \normFun{u-v}{X},\sup_{v\in V}\inf_{u\in U} \normFun{u-v}{X}\right\}.
    \end{equation}
\end{definition}

\begin{definition}[Hausdorff distance between ODEs]\label{def:hd}
    Given $\odedim,\odeorder \in \Nplus$, let $\mathcal{F}$ be a class of functions $f:\R^{\odeorder \odedim}\rightarrow\R^\odedim$, and let 
    $\odeClass[\odeorder]{\mathcal{F}}$ be the corresponding class of ODEs. Fix $\horizon>0$. For every $f_1,f_2 \in \mathcal{F}$ and their corresponding differential equations $\diffeq[f_1],\diffeq[f_2]\in \odeClass[\odeorder]{\mathcal{F}}$,
    we define the solution variation with respect to initial conditions $\initset$ as 
    \begin{equation}\label{eq:diff_eq_var}
    \sv[\initset][\normFun{\cdot}{C^{\odeorder-1}}]{\diffeq[f_i]}{\diffeq[f_j]} \defeq \sup_{\tilde{x}\in \behavior[\initset]{\diffeq[f_j]}}\inf_{x\in \behavior[\initset]{\diffeq[f_i]}}\normFun[2]{x-\tilde{x}}{C^{\odeorder-1}([0,\horizon])}\quad \text{for } \{i,j\}=\{1,2\}.
\end{equation}
The Hausdorff distance between these two equations with respect to initial conditions $\initset$ is
\begin{equation}\label{eq:diff_eq_hd}
    \hd[\initset][\normFun{\cdot}{C^{\odeorder-1}}]{\diffeq[f_1]}{\diffeq[f_2]}\defeq \max\{\sv[\initset][\normFun{\cdot}{C^{\odeorder-1}}]{\diffeq[f_1]}{\diffeq[f_2]},\sv[\initset][\normFun{\cdot}{C^{\odeorder-1}}]{\diffeq[f_2]}{\diffeq[f_1]}\}.
\end{equation}
One can think of this Hausdorff distance between ODEs as the Hausdorff distance between their behaviors (solution sets), i.e.,
\begin{equation}
    \hd[\initset][\normFun{\cdot}{C^{\odeorder-1}}]{\diffeq[f_1]}{\diffeq[f_2]}  = \hd[][\normFun{\cdot}{C^{\odeorder-1}}]{\behavior[\initset]{\diffeq[f_1]}}{\behavior[\initset]{\diffeq[f_2]}}.
\end{equation}
For ease of notation, we abuse the notation of Hausdorff distance for ODEs and Hausdorff distance for sets.
\end{definition}
In the rest of this paper, we set $\horizon>0$ to be a universal constant characterizing the horizon for ODEs. For simplicity, domain $[0,\horizon]$ and vector norm $\norm{\cdot}_2$
is omitted in the notation of Hausdorff distance and we simply write $\hd[][\normFun{\cdot}{L^{\infty}}]{\cdot}{\cdot}$ or $\hd[][\normFun{\cdot}{C^{k}}]{\cdot}{\cdot}$ instead of $\hd[][\normFun[2]{\cdot}{L^{\infty}([0,\horizon])}]{\cdot}{\cdot}$ or $\hd[][\normFun[2]{\cdot}{C^{k}([0,\horizon])}]{\cdot}{\cdot}$.


\section{Identification results for classes of ODEs}\label{sec:hd_bound_ode}
In this section, we examine the relationship between the Hausdorff distance in ODEs and distances in structure equations. Practically speaking, if small Hausdorff distances in ODEs imply small distances in structure equations, i.e.,
\begin{equation}
    \norm{f_1-f_2}\leq \delta(\hd{\diffeq[f_1]}{\diffeq[f_2]})\text{ where }\delta(0)=0\text{ and }\lim_{a\rightarrow 0^+}\delta(a)=0,
\end{equation}
we can develop an algorithm that minimizes the Hausdorff distances between the solution sets of the \textcolor{cc}{model ODE $\diffeq[f_2]$} and the ground-truth ODE $\diffeq[f_1]$, which then ensures the possibility of a unique and stable identification of the structure $f_1$.
\subsection{First-order linear ODEs}\label{sec:1st_lin_ode}
We first consider the class of $1$st order, $\odedim$-dimensional linear ODEs $\odeClass{\lin[\R^{ \odedim},\R^\odedim ]{K}}$ as per Definition \ref{def:ode_classes}. Note that \textcolor{cc}{the solutions of} linear ODEs are generally stable with respect to \textcolor{cc}{the perturbations in} structure matrices. This leads to the following upper bound on the Hausdorff distance between linear ODEs.

\begin{lemma}[Upper bound]\label{lm:upper_lin_ode}
    Fix an arbitrary $D>0$ and consider $\ball{D}\defeq \{x\in \R^{\odedim}\mid \norm{x}_2\leq D\}$. For every two linear ODEs $\diffeq[A],\diffeq[\tilde{A}]\in \odeClass{\lin[\R^{ \odedim},\R^\odedim ]{K}}$ with corresponding $A, \tilde{A}\in \R^{\odedim\times \odedim}$, we have that 
    \begin{equation}
        \hd[\ball{D}][\normFun{\cdot}{L^{\infty}}]{\diffeq[A]}{\diffeq[\tilde{A}]}\leq D\horizon e^{K\horizon}\norm{A-\tilde{A}}_2.
    \end{equation}
\end{lemma}
\begin{proof}
    By applying the matrix exponent inequality from Lemma \ref{lm:exp_diff_norm_upper},
    we have that 
    \begin{equation}
        \norm{e^{\tilde{A}t}\tilde{x}_0-e^{At}\tilde{x}_0}_2\leq \norm{A-\tilde{A}}_2D\horizon e^{K\horizon} 
    \end{equation}
    for all $t\in [0,\horizon]$ and all $\tilde{x}_0$ such that $\norm{\tilde{x}_0}\leq D$. This implies that 
    \begin{equation}
        \begin{aligned}
         \sv[\ball{D}][\normFun{\cdot}{L^{\infty}}]{\diffeq[A]}{\diffeq[\tilde{A}]}  &= \sup_{\tilde{x}\in \behavior[\ball{D}]{\diffeq[\tilde{A}]}}\inf_{x\in \behavior[\ball{D}]{\diffeq[A]}}\normFun[2]{x-\tilde{x}}{L^{\infty}([0,\horizon])}\\
        &= \sup_{\norm{\tilde{x}_0}_2\leq D}\inf_{\norm{x_0}_2\leq D}\normFun[2]{e^{A\cdot}x_0-e^{\tilde{A}\cdot}\tilde{x}_0}{L^{\infty}([0,\horizon])}\\
        &\leq \sup_{\norm{\tilde{x}_0}\leq D}\normFun[2]{e^{A\cdot}\tilde{x}_0-e^{\tilde{A}\cdot}\tilde{x}_0}{L^{\infty}([0,\horizon])}\\
        &\leq D\horizon e^{K\horizon}\norm{A-\tilde{A}}_2.
        \end{aligned}
    \end{equation}
    By symmetry, this gives the desired bound for $\hd[\ball{D}][\normFun{\cdot}{L^{\infty}}]{\diffeq[A]}{\diffeq[\tilde{A}]}$.
\end{proof}
Next, we obtain a lower bound on the Hausdorff distance between linear ODEs with respect to the distance between structure matrices, which, according to the beginning of Section \ref{sec:hd_bound_ode}, implies identification results.
\begin{lemma}[Lower bound]\label{lm:lower_lin_eq}
    Fix an arbitrary $D>0$ and consider $\ball{D}\defeq \{x\in \R^{\odedim}\mid \norm{x}_2\leq D\}$. For every two linear ODEs $\diffeq[A],\diffeq[\tilde{A}]\in \odeClass{\lin[\R^{ \odedim},\R^\odedim ]{K}}$ with corresponding $A, \tilde{A}\in \R^{d\times d}$, we have that 
    \begin{equation}\label{eq:lower_bound_hd}
        \hd[\ball{D}][\normFun{\cdot}{L^{\infty}}]{\diffeq[A]}{\diffeq[\tilde{A}]}\geq \frac{\left(2-e^{\frac{1}{2}}\right)D\min\left\{\horizon,\frac{1}{2K}\right\}}{\sqrt{2}} \left(\frac{e^{-5\min\left\{\horizon,\frac{1}{2K}\right\}K}}{1+e^{-5\min\left\{\horizon,\frac{1}{2K}\right\}K}}\right)^{\frac{1}{2}}\norm{A-\tilde{A}}_2
    \end{equation}
\end{lemma}
\begin{proof}
    For arbitrarily fixed $t_0\in (0,\horizon]$ and $x_0,\tilde{x}_0\in \ball{D}$, set $B\defeq At_0$, $\tilde{B}\defeq\tilde{A}t_0$, $y\defeq x_0-\tilde{x}_0$ and $b\defeq (e^{B}-e^{\tilde{B}})\tilde{x}_0$. Then, $\norm{B}_2=\norm{B^T}_2\leq t_0K$. Note that for arbitrarily fixed $\tilde{x}_0\in \ball{D}$, we have  
    \begin{equation}\label{eq:lower_lti_1}
        \begin{aligned}
        &\inf_{\norm{x_0}\leq D}\normFun[2]{e^{A\cdot}x_0-e^{\tilde{A}\cdot}\tilde{x}_0}{L^\infty}\geq \inf_{\norm{x_0}\leq D} \max\left\{\norm{x_0-\tilde{x}_0}_2, \norm{e^{At_0}x_0-e^{\tilde{A}t_0}\tilde{x}_0}_2\right\}\\
        &\geq \frac{1}{\sqrt{2}}\inf_{x_0\in \R^{\odedim}}\left(\norm{x_0-\tilde{x}_0}_2^2+\norm{e^{B}x_0-e^{\tilde{B}}\tilde{x}_0}_2^2\right)^{\frac{1}{2}}\\
        &= \frac{1}{\sqrt{2}}\inf_{y\in \R^{\odedim}}\left(\norm{y}_2^2+\norm{e^{B}y+b}_2^2\right)^{\frac{1}{2}}\\
        &= \frac{1}{\sqrt{2}} \inf_{y\in \R^{\odedim}}\left(y^T(I+e^{B^T}e^{B})y+2b^Te^B y + b^Tb\right)^{\frac{1}{2}}\\
        &= \frac{1}{\sqrt{2}}\left( b^T \left(I-e^B(I+e^{B^T}e^{B})^{-1}e^{B^T}\right) b\right)^{\frac{1}{2}},
        \end{aligned}
    \end{equation}
    where the last equality holds since we are minimizing a quadratic form and $I+e^{B^T}e^{B}$ is a symmetric positive definite matrix. Note that $$I-e^B(I+e^{B^T}e^{B})^{-1}e^{B^T} = I - \left(I+e^{-B^T}e^{-B}\right)^{-1}$$ is a symmetric matrix and thus only has real eigenvalues. Moreover, $e^{-B^T}e^{-B}$ is a symmetric positive definite matrix and by Lemma
    \ref{lm:bound_sigN_EATEA} (noting that $\norm{A^T}_2 = \norm{A}_2\leq K$ and replacing $A$ with $-B$), 
    \begin{equation}\label{eq:lower_sigN_EBTEB}
        \eig_n(e^{-B^T}e^{-B}) \geq e^{-5t_0K}.
    \end{equation}
    Now, note that 
    \begin{equation*}
        \eig_n\left(I -\left(I+e^{-B^T}e^{-B}\right)^{-1}\right) = 1-\left(1+\eig_i\left(e^{-B^T}e^{-B}\right)\right)^{-1}\quad \text{for some }i\in \{1,2,\dots,n\},
    \end{equation*}
    which means
    \begin{equation}\label{eq:lower_lti_2}
    \begin{aligned}
        \eig_n\left(I -\left(I+e^{-B^T}e^{-B}\right)^{-1}\right)&\geq 1-\left(1+\eig_n\left(e^{-B^T}e^{-B}\right)\right)^{-1}\\
        &\overset{\eqref{eq:lower_sigN_EBTEB}}{\geq} 1-\left(1+e^{-5t_0K}\right)^{-1} \\
        & = \frac{e^{-5t_0K}}{1+e^{-5t_0K}}.
    \end{aligned}
    \end{equation}
    Thus, we have 
    \begin{equation}\label{eq:lower_lti_3}
        \begin{aligned}
        \sv[\ball{D}][\normFun{\cdot}{L^{\infty}}]{\diffeq[A]}{\diffeq[\tilde{A}]} &= \sup_{\tilde{x}\in \behavior[\ball{D}]{\diffeq[\tilde{A}]}}\inf_{x\in \behavior[\ball{D}]{\diffeq[A]}}\normFun[2]{x-\tilde{x}}{L^{\infty}([0,\horizon])}\\
        &= \sup_{\norm{\tilde{x}_0}_2\leq D}\inf_{\norm{x_0}_2\leq D}\normFun[2]{e^{A\cdot}x_0-e^{\tilde{A}\cdot}\tilde{x}_0}{L^{\infty}([0,\horizon])}\\
        & \overset{\eqref{eq:lower_lti_1}}{\geq} \frac{1}{\sqrt{2}}\sup_{\norm{\tilde{x}_0}_2\leq D}\left( b^T \left(I-e^B(I+e^{B^T}e^{B})^{-1}e^{B^T}\right) b\right)^{\frac{1}{2}}\\
        &\overset{\eqref{eq:lower_lti_2}}{\geq} \frac{1}{\sqrt{2}} \left(\frac{e^{-5t_0K}}{1+e^{-5t_0K}}\right)^{\frac{1}{2}}\sup_{\norm{\tilde{x}_0}_2\leq D}\norm{(e^{At_0}-e^{\tilde{A}t_0})\tilde{x}_0}_2\\
        &= \frac{D}{\sqrt{2}} \left(\frac{e^{-5t_0K}}{1+e^{-5t_0K}}\right)^{\frac{1}{2}}\sup_{\norm{\tilde{x}_0}_2\leq 1}\norm{(e^{At_0}-e^{\tilde{A}t_0})\tilde{x}_0}_2\\
        & =  \frac{D}{\sqrt{2}} \left(\frac{e^{-5t_0K}}{1+e^{-5t_0K}}\right)^{\frac{1}{2}}\norm{e^{At_0}-e^{\tilde{A}t_0}}_2.
        \end{aligned}
    \end{equation}
    Next, we set out to find a $t_0$ such that $\norm{e^{At_0}-e^{\tilde{A}t_0}}_2$ can be lower-bounded by a positive scaling of $\norm{A-\tilde{A}}_2$. Note that 
    \begin{equation}
    \begin{aligned}
        \norm{e^{At_0}-e^{\tilde{A}t_0}}_2 &= \norm{\sum_{k=0}^{\infty} \frac{1}{k!}\left((At_0)^k -(\tilde{A}t_0)^k \right)}_2   \\
        &\geq \norm{A-\tilde{A}}_2t_0 - \sum_{k=2}^{\infty} \norm{A^k-\tilde{A}^k}_2 t_0^k\\
        \overset{\text{Lemma \ref{lm:power_diff_norm_upper}}}&{\geq} \norm{A-\tilde{A}}_2t_0 - \sum_{k=2}^{\infty} \frac{k}{k!}\norm{A-\tilde{A}}_2\max\left\{\norm{A}_2,\norm{\tilde{A}}_2\right\}^{k-1}t_0^k\\
        &\geq \norm{A-\tilde{A}}_2t_0 - \norm{A-\tilde{A}}_2t_0\sum_{k=1}^{\infty} \frac{1}{k!}K^{k}t_0^k\\
        & = \norm{A-\tilde{A}}_2t_0 \left(2-e^{Kt_0}\right)
    \end{aligned}
    \end{equation}
    If we choose $t_0 = \min\left\{\horizon,\frac{1}{2K}\right\}$, then 
    \begin{equation}\label{eq:lower_lti_4}
        \norm{e^{At_0}-e^{\tilde{A}t_0}}_2 \geq \min\left\{\horizon,\frac{1}{2K}\right\} \left(2-e^{\frac{1}{2}}\right)\norm{A-\tilde{A}}_2.
    \end{equation}
    Combining \eqref{eq:lower_lti_3} and \eqref{eq:lower_lti_4}, we have 
    \begin{equation*}
        \sv[\ball{D}][\normFun{\cdot}{L^{\infty}}]{\diffeq[A]}{\diffeq[\tilde{A}]} \geq \frac{\left(2-e^{\frac{1}{2}}\right)D\min\left\{\horizon,\frac{1}{2K}\right\}}{\sqrt{2}} \left(\frac{e^{-5\min\left\{\horizon,\frac{1}{2K}\right\}K}}{1+e^{-5\min\left\{\horizon,\frac{1}{2K}\right\}K}}\right)^{\frac{1}{2}}\norm{A-\tilde{A}}_2.
    \end{equation*}
    By symmetry, this gives the desired bound for $\hd[\ball{D}][\normFun{\cdot}{L^{\infty}}]{\diffeq[A]}{\diffeq[\tilde{A}]}$.
\end{proof}

In summary, we get the following bounds for Hausdorff distance between linear ODEs.
\begin{theorem}\label{th:dist_lu_linearODE_const}
    Fix an arbitrary $D>0$ and consider $\ball{D}\defeq \{x\in \R^{\odedim}\mid \norm{x}_2\leq D\}$. For every two linear ODEs $\diffeq[A],\diffeq[\tilde{A}]\in \odeClass{\lin[\R^{ \odedim},\R^\odedim ]{K}}$ with corresponding $A, \tilde{A}\in \R^{d\times d}$, we have that 
    \begin{equation}
        \underbrace{c\norm{A-\tilde{A}}_2\leq}_{\text{Identification stability}} \hd[\ball{D}][\normFun{\cdot}{L^{\infty}}]{\diffeq[A]}{\diffeq[\tilde{A}]}\underbrace{\leq C\norm{A-\tilde{A}}_2}_{\text{Structural stability}},
    \end{equation}
    where $c,C>0$ are constants that only depend on $\horizon,K,D$.
\end{theorem}
The lower bound in Theorem \ref{th:dist_lu_linearODE_const} \textcolor{cc}{implies that the closeness in solution sets in terms of Hausdorff distance guarantees closeness in the corresponding structure equation, which, in turn, means that we can identify the governing equation of an ODE uniquely and stably based on the full knowledge of its solution set.}

Note that there are generally no conditions on the size of the ball $\ball{D}$, this means even a small perturbation region in initial conditions is enough for us to identify the linear ODE based on the information of the solution set. Nevertheless, a small ball in $\R^{\odedim}$ is still a set with infinite cardinality. The question is, what is the least amount of \textcolor{cc}{solution data} we need to identify the structure?

\textcolor{cc}{Recall that in Section \ref{sec:hd_def} an example was given in which the structure matrix of a linear ODE is diagonalizable. In that setting, having initial conditions that span the state space is sufficient to identify the system. This motivates the assumption that, for linear ODEs with general structure matrices, “completeness” of the initial-condition set (in the sense of spanning the state space) should likewise be sufficient for identification. To allow additional flexibility, it is convenient to work not only with bases but also with spanning, possibly redundant generating sets —namely, a frame—and, without disrupting the flow of the present paper, refer the reader to a short overview of basic frame theory in \cite{bolcskei2020lecturemoi}. The following lemmata formalize and justify the above intuition.}
\begin{lemma}\label{lm:lower_lin_eq_local}
    Suppose $G=\{g_1,g_2,\dots,g_N\}$ is a frame for $\R^{\odedim}$ and $\{\tilde{g}_1,\tilde{g}_2,\dots,\tilde{g}_N\}$ is its corresponding canonical dual frame, such that there exist $F_1,F_2$ with $0<F_1\leq F_2$,
    \begin{equation*}
        \begin{aligned}
        &F_1\norm{x}_2^2\leq\sum_{k=1}^{N}\left|\langle x,g_k\rangle\right|^2\leq F_2\norm{x}_2^2,\\
        &\frac{1}{F_2}\norm{x}_2^2\leq\sum_{k=1}^{N}\left|\langle x,\tilde{g}_k\rangle\right|^2\leq\frac{1}{F_1}\norm{x}_2^2.
        \end{aligned}
    \end{equation*}
    For every two linear ODEs $\diffeq[A],\diffeq[\tilde{A}]\in \odeClass{\lin[\R^{ \odedim},\R^\odedim ]{K}}$ with corresponding $A, \tilde{A}\in \R^{d\times d}$, we have that
    \begin{equation}
        \hd[G][\normFun{\cdot}{L^{\infty}}]{\diffeq[A]}{\diffeq[\tilde{A}]}\geq \frac{\left(2-e^{\frac{1}{2}}\right)\sqrt{F_1}\min\left\{\horizon,\frac{1}{2K}\right\}}{\sqrt{2N}} \left(\frac{e^{-5\min\left\{\horizon,\frac{1}{2K}\right\}K}}{1+e^{-5\min\left\{\horizon,\frac{1}{2K}\right\}K}}\right)^{\frac{1}{2}}\norm{A-\tilde{A}}_2.
    \end{equation}
\end{lemma}
\begin{proof}
    The proof follows similar lines as that of Lemma \ref{lm:lower_lin_eq}. We defer it to Appendix \ref{app:proof_lin_local}.
\end{proof}

From Lemma \ref{lm:lower_lin_eq_local}, \textcolor{cc}{it's enough to identify the structure matrix of a linear ODE if the set of initial values is complete in the sense that it contains a basis of $\R^\odedim$}. The following lemma shows that a basis is, in turn, necessary in identifying structure matrices for linear ODEs. 

\begin{highlight}
\begin{lemma}
    Given $G=\{g_1,g_2\dots,g_N\}$ such that $\operatorname{span}\{g_1,g_2,\dots,g_N\}\neq \R^{\odedim}$, there exist $A,B\in \R^{\odedim\times \odedim}$, such that $\norm{A-B}_2>0$ and the corresponding linear ODEs $\diffeq[A],\diffeq[B]\in\odeClass{\lin[\R^{ \odedim},\R^\odedim ]{K}}$ satisfy $\hd[G][\normFun{\cdot}{L^{\infty}}]{\diffeq[A]}{\diffeq[B]}=0$.
\end{lemma}
\end{highlight}
\begin{proof}
    Since the solutions of linear ODEs with constant coefficient in linear with respect to initial values, we can assume, without loss of generality, that $N<\odedim$ and $g_1,g_2\dots,g_N$ is linearly independent. Now, we extend $g_1,g_2\dots,g_N$ to $g_1,g_2,\allowbreak \dots,g_N,g_{N+1},\dots, g_\odedim$ such that $\operatorname{span}\{g_1,g_2,\dots,g_\odedim\}= \R^{\odedim}$. Moreover, define 
    \begin{equation*}
        \begin{aligned}
            T &= \begin{pmatrix}g_1&g_2&\dots&g_\odedim \end{pmatrix},\\
            \Lambda &= \begin{pmatrix}\indmat{N}& \\  &2\indmat{\odedim-N} \end{pmatrix}.
        \end{aligned}
    \end{equation*}
    Now, consider $A = T\Lambda T^{-1}$ and $B=\indmat{\odedim}$. We have 
    \begin{equation*}
        \begin{aligned}
            A g_i = B g_i = g_i, \quad \text{for all }i\leq N,
        \end{aligned}
    \end{equation*}
    which implies that $\hd[G][\normFun{\cdot}{L^{\infty}}]{\diffeq[A]}{\diffeq[B]}=0$. In the meanwhile, we have 
    \begin{equation*}
        \begin{aligned}
            \norm{A-B}_2 & \overset{\text{Lemma \ref{lm:mat_F_2}}}{\geq} \frac{1}{\sqrt{\odedim}} \norm{A-B}_F\\
            &\geq \frac{1}{\odedim}\left|\operatorname{tr}(A-B)\right|\\
            & =\frac{1}{\odedim}\left|\operatorname{tr}(T(\Lambda-\indmat{\odedim})T^{-1})\right|\\
            &= \frac{1}{\odedim}\left|\operatorname{tr}(\Lambda-\indmat{\odedim})\right| \\
            & = \frac{\odedim-N}{\odedim}>0,
        \end{aligned}
    \end{equation*}
    which concludes the proof.
\end{proof}

\subsection{First-order nonlinear ODEs: Lipschitz ODEs}\label{sec:1st_lip_ode}
In this section, we consider the most regular nonlinear ODEs---Lipschitz ODEs. In particular, we first consider the class of 1st order, $\odedim$-dimensional Lipschitz ODEs $\odeClass{\lip[\R^{\odedim},\R^\odedim ]{L,K}}$. Since \textcolor{cc}{the solutions of} Lipschitz ODEs are generally stable with respect to \textcolor{cc}{the perturbations in} structure equations, we immediately have the following upper bound on the Hausdorff distance between Lipschitz ODEs.
\begin{lemma}[Upper bound]\label{lm:lip_ode_upper}
    Fix an arbitrary $D>0$ and a compact set $\compactInit\subset \ball{D} \defeq  \{x\in \R^{\odedim}\mid \norm{x}_2\leq D\}$. Let $\widehat{\compactInit} \defeq \{x\in \R^{\odedim}\mid \norm{x}_2\leq (K\horizon+D) e^{L\horizon}\}$. For every two Lipschitz ODEs $\diffeq[f],\diffeq[\tilde{f}]\in \odeClass{\lip[\R^{\odedim},\R^\odedim ]{L,K}}$ with corresponding $f, \tilde{f}\in \lip[\R^{\odedim},\R^\odedim ]{L,K}$, we have that 
    \begin{equation}
        \hd[\compactInit][\normFun{\cdot}{L^{\infty}}]{\diffeq[f]}{\diffeq[\tilde{f}]}\leq \horizon e^{L\horizon}\normFun[2]{f-\tilde{f}}{L^{\infty}(\widehat{\compactInit})}.
    \end{equation}
\end{lemma}
\begin{proof}
    Let $\tilde{x}\in \behavior[\compactInit]{\diffeq[\tilde{f}]}$ be a solution to $\diffeq[\tilde{f}]$. We then have 
    \begin{equation}
        \tilde{x}(t) = x_0+\int_{0}^{t} \tilde{f}(\tilde{x}(s))ds,
    \end{equation}
    where $x_0\in \compactInit$. Taking the norm, noting that $\operatorname{Lip}(f)\leq L,
    \norm{f(0)}_2\leq K$, $\norm{x_0}_2\leq D$ we have 
    \begin{equation}
        \norm{\tilde{x}(t)}_2\leq L\int_{0}^{t} \norm{\tilde{x}(s)}_2ds + Kt+D.
    \end{equation}
    Appling Grönwall's inequality, we obtain 
    \begin{equation}
        \norm{\tilde{x}(t)}_2\leq (Kt+D) e^{Lt}\leq (K\horizon+D) e^{L\horizon}
    \end{equation}
    Now, consider the solution $x$ to $\diffeq[f]$ with $x(0)=\tilde{x}(0)\in \compactInit$ and define $y \defeq x-\tilde{x}$. Then, $y$ satisfy the ODE
    \begin{equation}
        \dot{y}(t) = f(y(t)+\tilde{x}(t)) - \tilde{f}(\tilde{x}(t))
    \end{equation}
    with $y(0)=0$. Taking the integral and bounding the norm, we have 
    \begin{equation}
        \begin{aligned}
            \norm{y(t)}_2&\leq \int_{0}^{t} \norm{f(y(s)+\tilde{x}(s)) - \tilde{f}(\tilde{x}(s))}_2 ds\\
            & \leq \int_{0}^{t} \norm{f(y(s)+\tilde{x}(s)) - f(\tilde{x}(s))}_2 ds\\
            & \quad+\int_{0}^{t} \norm{f(\tilde{x}(s)) - \tilde{f}(\tilde{x}(s))}_2 ds\\
            & \leq L\int_{0}^{t} \norm{y(s)}_2ds + \normFun[2]{f-\tilde{f}}{L^{\infty}(\widehat{\compactInit})}t.
        \end{aligned}
    \end{equation}
    Applying Grönwall's lemma and take supreme over $t$, we obtain 
    \begin{equation}
        \normFun[2]{y}{L^\infty([0,\horizon])} \leq \horizon e^{L\horizon}\normFun[2]{f-\tilde{f}}{L^{\infty}(\widehat{\compactInit})}.
    \end{equation}
    This implies that 
    \begin{equation}
        \inf_{x\in \behavior[\compactInit]{\diffeq[f]}}\normFun[2]{x-\tilde{x}}{L^{\infty}([0,\horizon])}\leq \normFun[2]{y}{L^\infty([0,\horizon])}\leq \horizon e^{L\horizon}\normFun[2]{f-\tilde{f}}{L^{\infty}(\widehat{\compactInit})}.
    \end{equation}
    Since this holds for arbitrarily chosen $\tilde{x}\in\behavior[\compactInit]{\diffeq[\tilde{f}]}$, we have 
    \begin{equation}
        \sv[\compactInit][\normFun{\cdot}{L^{\infty}}]{\diffeq[f]}{\diffeq[\tilde{f}]}\leq \horizon e^{L\horizon}\normFun[2]{f-\tilde{f}}{L^{\infty}(\widehat{\compactInit})} ,
    \end{equation}
    and thus, by symmetry, the desired result.
\end{proof}

Next, we obtain a lower bound on the Hausdorff distance between Lipschitz ODEs, which will then be used to obtain identification results.
\begin{lemma}\label{lm:dist_lower_LipODE_local}
Fix an arbitrary $D>0$ and a compact set $\compactInit\subset \ball{D} \defeq  \{x\in \R^{\odedim}\mid \norm{x}_2\leq D\}$. For every two Lipschitz ODEs $\diffeq[f],\diffeq[\tilde{f}]\in \odeClass{\lip[\R^{\odedim},\R^\odedim ]{L,K}}$ with corresponding $f, \tilde{f}\in \lip[\R^{\odedim},\R^\odedim ]{L,K}$, we have 
    \begin{equation}
        \hd[\compactInit][\normFun{\cdot}{L^{\infty}}]{\diffeq[f]}{\diffeq[\tilde{f}]}\geq \min\left\{\frac{\normFun[2]{f-\tilde{f}}{L^{\infty}(\compactInit)}}{2\sqrt{\odedim}L\left(2D L\horizon e^{L\horizon}  + KL\horizon^2 e^{L\horizon} + 2K\right)}, \horizon\right\}\frac{\normFun[2]{f-\tilde{f}}{L^{\infty}(\compactInit)}}{2\sqrt{\odedim}(2+L\horizon)}.
    \end{equation} 
\end{lemma}
\begin{proof}
    First, we pick $x_0 \in \compactInit$ such that $$\norm{f(x_0)-\tilde{f}(x_0)}_2 = \normFun[2]{f-\tilde{f}}{L^{\infty}(\compactInit)}$$
    thanks to the continuity of $f,\tilde{f}$ and compactness of $\compactInit$. Now, pick $\tilde{x}\in \behavior[\compactInit]{\diffeq[\tilde{f}]}$ such that $\tilde{x}(0)=x_0$. We then have 
    \begin{equation}
        \tilde{x}(t) = x_0+\int_{0}^{t} \tilde{f}(\tilde{x}(s))ds.
    \end{equation}
    Taking the norm, noting that $\operatorname{Lip}(f)\leq L,
    \norm{f(0)}_2\leq K$, we have 
    \begin{equation}\label{eq:proof_x_x0}
        \norm{\tilde{x}(t)-x_0}_2\leq L\int_{0}^{t} \norm{\tilde{x}(s)}_2ds + Kt.
    \end{equation}
    Defining 
    \begin{equation}\label{eq:proof_y}
        y(t) = e^{-Lt}\int_{0}^{t} \norm{\tilde{x}(s)}_2ds,
    \end{equation}
    we have that $y(0)=0$ and 
    \begin{equation}
        \dot{y}(t)=e^{-Lt}\left(\norm{\tilde{x}(t)}_2-L\int_{0}^{t} \norm{\tilde{x}(s)}_2ds\right)\leq e^{-Lt}\left(\norm{x_0}_2+Kt\right).
    \end{equation}
    Integrating from $0$ to $t$, we obtain 
    \begin{equation}
        y(t) \leq \frac{1-e^{-Lt}}{L}\norm{x_0}_2  + \frac{K\left(1-e^{-Lt}(1+Lt)\right)}{L^2},
    \end{equation}
    which, together with \eqref{eq:proof_x_x0} and \eqref{eq:proof_y}, implies that 
    \begin{equation}
    \begin{aligned}
        \norm{\tilde{x}(t)-x_0}_2&\leq Le^{Lt}\left(\frac{1-e^{-Lt}}{L}\norm{x_0}_2  + \frac{K\left(1-e^{-Lt}(1+Lt)\right)}{L^2}\right) +Kt\\
        &\leq  (e^{Lt}-1)\norm{{x}_0}_2 + (e^{Lt}-1-Lt)\frac{K}{L} + Kt\\
        &\leq \left(D L\horizon e^{L\horizon}  + \frac{KL\horizon^2e^{L\horizon}}{2} + K\right)t.
    \end{aligned}
    \end{equation}
    Setting 
    \begin{equation}
        t_0\defeq \min\left\{\frac{\normFun[2]{f-\tilde{f}}{L^{\infty}(\compactInit)}}{2\sqrt{\odedim}L\left(2D L\horizon e^{L\horizon}  + KL\horizon^2 e^{L\horizon} + 2K\right)}, \horizon\right\},
    \end{equation}
    we obtain 
    \begin{equation}
        \norm{\tilde{x}(t)-x_0}_2\leq \frac{\normFun[2]{f-\tilde{f}}{L^{\infty}(\compactInit)}}{4\sqrt{\odedim}L}, \quad \text{for }t\in [0,t_0].
    \end{equation}
    Define $h\defeq f-\tilde{f}$. Noting that $x_0$ is picked such that $\norm{h(x_0)}_2 = \normFun[2]{f-\tilde{f}}{L^{\infty}(\compactInit)}$, there must exits $j\in \{1,2,\dots,\odedim\}$ satisfying 
    \begin{equation}
        \left| (h(x_0))_j \right| \geq \frac{\normFun[2]{f-\tilde{f}}{L^{\infty}(\compactInit)}}{\sqrt{\odedim}}.
    \end{equation}
    Since $\operatorname{Lip}(h) \leq 2L$, we have 
    \begin{equation*}
        \left| (h(\tilde{x}(t)))_j-(h(x_0))_j \right|\leq \norm{h(\tilde{x}(t))- h(x_0)}_2\leq 2L\norm{\tilde{x}(t)-x_0}_2,
    \end{equation*}
    and thus
        \begin{equation}
    \begin{aligned}\label{eq:hxt_lower_ext}
        \left| (h(\tilde{x}(t)))_j\right|&\geq \left| (h(x_0))_j \right| - 2L\norm{\tilde{x}(t)-x_0}_2\\
        &\geq \frac{\normFun[2]{f-\tilde{f}}{L^{\infty}(\compactInit)}}{\sqrt{\odedim}} -2L\frac{\normFun[2]{f-\tilde{f}}{L^{\infty}(\compactInit)}}{4\sqrt{\odedim}L}\\
        &= \frac{\normFun[2]{f-\tilde{f}}{L^{\infty}(\compactInit)}}{2\sqrt{\odedim}},\quad \text{for }t\in [0,t_0].
    \end{aligned}    
    \end{equation}
    Now, for every $x\in \behavior[\compactInit]{\diffeq[f]}$, define $y = x-\tilde{x}$. Then, $y$ satisfy the ODE
    \begin{equation}
        \dot{y}(t) = f(y(t)+\tilde{x}(t)) - \tilde{f}(\tilde{x}(t)).
    \end{equation}
    Integrating from $0$ to $t$, we have 
    \begin{equation}
        y(t) = y(0) + \int_{0}^{t}\left(f(y(s)+\tilde{x}(s)) - f(\tilde{x}(s))\right)ds + \int_{0}^{t}\left(f(\tilde{x}(s)) - \tilde{f}(\tilde{x}(s))\right)ds.
    \end{equation}
    Taking the absolute value, applying triangle inequality, and noticing $\operatorname{Lip}(f)\leq L$ gives 
    \begin{equation}
    \begin{aligned}
        \normFun[2]{y}{L^\infty}&\geq \norm{y(t)}_2\geq \norm{\int_{0}^{t}\left(f(\tilde{x}(s)) - \tilde{f}(\tilde{x}(s))\right)ds}_2 - \norm{y(0)}_2 - L\int_{0}^{t}\norm{y(s)}_2ds\\
        &\geq \norm{\int_{0}^{t}\left(f(\tilde{x}(s)) - \tilde{f}(\tilde{x}(s))\right)ds}_2- (1+LT)\normFun[2]{y}{L^\infty},
    \end{aligned}
    \end{equation}
    which in turn gives 
    \begin{equation}\label{eq:proof_y_inf_lower}
        \normFun[2]{y}{L^\infty} \geq  \frac{1}{2+LT}\norm{\int_{0}^{t}h(\tilde{x}(s))ds}_2,\quad \text{for all }t\in[0,\horizon].
    \end{equation}
    By \eqref{eq:hxt_lower_ext}, since $(h\circ \tilde{x})_j$ is continuous, $(h(\tilde{x}(t)))_j$ must remain positive or negative for $t\in [0,t_0]$. In either case, we have 
    \begin{equation*}
    \begin{aligned}
         \normFun[2]{x-\tilde{x}}{L^\infty}&=\normFun[2]{y}{L^\infty} \overset{\eqref{eq:proof_y_inf_lower},t=t_0}{\geq} \frac{1}{2+L\horizon}\norm{\int_{0}^{t_0}h(\tilde{x}(s))ds}_2\\
         &\geq  \frac{1}{2+L\horizon} \left|\left(\int_{0}^{t_0}h(\tilde{x}(s))ds\right)_j\right|\\
         \overset{\eqref{eq:hxt_lower_ext}}&{\geq} \frac{t_0}{2+L\horizon}\frac{\normFun[2]{f-\tilde{f}}{L^{\infty}(\compactInit)}}{2\sqrt{\odedim}} \\
         & \geq \min\left\{\frac{\normFun[2]{f-\tilde{f}}{L^{\infty}(\compactInit)}}{2\sqrt{\odedim}L\left(2D L\horizon e^{L\horizon}  + KL\horizon^2 e^{L\horizon} + 2K\right)}, \horizon\right\}\frac{\normFun[2]{f-\tilde{f}}{L^{\infty}(\compactInit)}}{2\sqrt{\odedim}(2+L\horizon)} .
    \end{aligned}
    \end{equation*}
    Since this holds for arbitrary $x\in \behavior[\compactInit]{\diffeq[f]}$, we obtain 
    \begin{equation*}
        \inf_{x\in \behavior[\compactInit]{\diffeq[f]}}\normFun[2]{x-\tilde{x}}{L^\infty}\geq \min\left\{\frac{\normFun[2]{f-\tilde{f}}{L^{\infty}(\compactInit)}}{2\sqrt{\odedim}L\left(2D L\horizon e^{L\horizon}  + KL\horizon^2 e^{L\horizon} + 2K\right)}, \horizon\right\}\frac{\normFun[2]{f-\tilde{f}}{L^{\infty}(\compactInit)}}{2\sqrt{\odedim}(2+L\horizon)}.
    \end{equation*}
    Taking the supremum over $\behavior[\compactInit]{\diffeq[\tilde{f}]}$ and by symmetry, we arrive at the desired result.
\end{proof}

In summary, we get the following bounds on the Hausdorff distance between Lipschitz ODEs.
\begin{theorem}\label{th:dist_lu_LipODE}
Fix an arbitrary $D>0$ and a compact set $\compactInit\subset \ball{D} \defeq  \{x\in \R^{\odedim}\mid \norm{x}_2\leq D\}$. Let $\widehat{\compactInit} \defeq \{x\in \R^{\odedim}\mid \norm{x}_2\leq (K\horizon+D) e^{L\horizon}\}$. For every two Lipschitz ODEs $\diffeq[f],\diffeq[\tilde{f}]\in \odeClass{\lip[\R^{\odedim},\R^\odedim ]{L,K}}$ with corresponding $f, \tilde{f}\in \lip[\R^{\odedim},\R^\odedim ]{L,K}$, we have that 
    \begin{equation}
    \begin{aligned}
        \underbrace{\min\left\{C_1\normFun[2]{f-\tilde{f}}{L^{\infty}(\compactInit)}^2, C_2\normFun[2]{f-\tilde{f}}{L^{\infty}(\compactInit)}\right\}
        \leq}_{\text{Identification stability}}& \hd[\compactInit][\normFun{\cdot}{L^{\infty}}]{\diffeq[f]}{\diffeq[\tilde{f}]}\\&\underbrace{\leq C_3\normFun[2]{f-\tilde{f}}{L^{\infty}(\widehat{\compactInit})}}_{\text{Structural stability}},
    \end{aligned}
    \end{equation}
    where $C_1, C_2,C_3>0$ are constants that only depend on $\horizon,L,D,K,\odedim$. 
\end{theorem}

The lower bound in Theorem \ref{th:dist_lu_LipODE} shows that if we are given information on the solution data with initial values in $\compactInit$, we are able to uniquely and stably identify the structure equation $f$ within $\compactInit$. Now, a natural question is, can we identify $f$ outside $\compactInit$ if we are only given solution information with initial values in $\compactInit$? The following lemma gives a negative answer when $\odedim=1$.

\begin{highlight}
\begin{lemma}
    Let $\odedim=1$. Given a compact set $\compactInit\subset \R$, for arbitrarily fixed $0<\delta<2K$, there exist $f,\tilde{f}\in\lip[\R,\R]{L,K}$, such that 
    $$|f(x)-\tilde{f}(x)|=\delta/2 ~\text{  for all  } ~x\in \left\{z\mid \inf_{w\in\compactInit}|z-w|> \delta\right\}\eqdef V(\compactInit,\delta),$$ 
    while the corresponding $\diffeq[f],\diffeq[\tilde{f}]\in \odeClass{\lip[\R,\R]{L,K}}$ satisfy $$\hd[\compactInit][\normFun{\cdot}{L^{\infty}}]{\diffeq[f]}{\diffeq[\tilde{f}]}=0.$$
\end{lemma}
\end{highlight}
\begin{proof}
Fix $f\equiv 0$. Define the hat function  
    \begin{equation}
        H(x)=\left\{\begin{array}{cc}
            1-Lx &  x\in [0,1/L]\\
            1+Lx &  x\in [-1/L,0]\\
            0 & \text{else}
        \end{array}\right. 
    \end{equation}
    and consider the translated and dilated version 
    \begin{equation}
        H_{y}(x) = \frac{\delta}{2}H\left(\frac{2}{\delta}(x-y)\right).
    \end{equation}
    Then, define 
    \begin{equation}
        \tilde{f}(x) = \sup_{y\in V(\compactInit,\delta)} H_y(x).
    \end{equation}
    It is straightforward that 
    \begin{equation}
        \begin{aligned}
            &|\tilde{f}(0)| = \left|\sup_{y\in V(\compactInit,\delta)} H_y(0)\right|\leq \frac{\delta}{2}\leq K,\\
            &\operatorname{Lip}(\tilde{f}) \leq \sup_{y\in V(\compactInit,\delta)} \operatorname{Lip}(H_y) \leq  L.
        \end{aligned}
    \end{equation}
    Moreover, we see that for all $x\in \{z\mid \inf_{w\in\compactInit}|z-w|\leq  \delta/2\}$, $\tilde{f}(x)=f(x)=0$, which implies for all $x\in \behavior[\compactInit]{\diffeq[f]}, \tilde{x}\in \behavior[\compactInit]{\diffeq[\tilde{f}]}$, we have $x=\tilde{x}\equiv 0$ (due to uniqueness of solutions to Lipschitz ODEs, \cite[][Theorem 6.I, Theorem 10.VII]{walterOrdinaryDifferentialEquations1998}), and thus
    \begin{equation*}
    \hd[\compactInit][\normFun{\cdot}{L^{\infty}}]{\diffeq[f]}{\diffeq[\tilde{f}]}=0.
    \end{equation*}
    In the meanwhile, for all $x\in V(\compactInit,\delta)$, we have
    \begin{equation}
    \begin{aligned}
        |\tilde{f}(x)-f(x)| &= \left|\sup_{y\in V(\compactInit,\delta)} H_y(x)\right|\\
        &= \left|H_x(x)\right| = \frac{\delta}{2}.    
    \end{aligned}
    \end{equation}
\end{proof}

\subsection{Higher-order linear and Lipschitz ODEs}\label{sec:higher_order}
In this section, we generalize the results from Section \ref{sec:1st_lin_ode} and \ref{sec:1st_lip_ode} to higher-order linear and Lipschitz ODEs. This generalization is straightforward via state-space reduction, i.e., if we are given a higher-order ODE of the form
\begin{equation}\label{eq:mth_ode}
    x^{(\odeorder)} = f\left(x^{(\odeorder-1)},\dots,x^{(1)},x\right),\quad x(t)\in \R^\odedim,
\end{equation}
we can apply state transformation $y = \begin{pmatrix}
    x^T&(x^{(1)})^T&\dots&(x^{(\odeorder-1)})^T\end{pmatrix}^T$ such that $y$ satisfy the first-order ODE
\begin{equation}\label{eq:1st_y_ode}
\begin{aligned}
    &\dot{y} = F\left(y\right), \quad y = \begin{pmatrix}
        y_1^T&y_2^T&\dots&y_\odeorder^T
    \end{pmatrix}^T\in \R^{\odeorder\odedim} \\
    \text{where }~&F(y) = \begin{pmatrix}
        y_2^T&y_3^T&\dots&y_\odeorder^T&\left(f(y_\odeorder,y_{\odeorder-1},\dots, y_1)\right)^T
    \end{pmatrix}^T.
\end{aligned}
\end{equation}
\textcolor{cc}{First, we generalize the results from Theorem \ref{th:dist_lu_linearODE_const} to higher-order linear ODEs.}
\begin{theorem}\label{th:dist_lu_linearODE_high}
    Fix an arbitrary $D>0$ and consider $\ball{D}\defeq \{y\in \R^{\odeorder\odedim}\mid \norm{y}_2\leq D\}$. For every two $\odeorder$-th order linear ODEs $\diffeq[A],\diffeq[\tilde{A}]\in \odeClass[\odeorder ]{\lin[\R^{\odeorder \odedim},\R^\odedim ]{K}}$ with corresponding
    \begin{equation}
    \begin{aligned}
        &A=\begin{pmatrix}
            A_{\odeorder -1}&\dots&A_1&A_0
        \end{pmatrix} \in \R^{\odeorder\odedim\times \odedim},\\
        &\tilde{A}=\begin{pmatrix}
            \tilde{A}_{\odeorder -1}&\dots&\tilde{A}_1&\tilde{A}_0
        \end{pmatrix} \in \R^{\odeorder\odedim\times \odedim},       
    \end{aligned}
    \end{equation}
    we have that 
    \begin{equation}
        \underbrace{c\sum_{i=0}^{\odeorder-1}\norm{A_i-\tilde{A}_i}_2\leq}_{\text{Identification stability}} \hd[\ball{D}][\normFun{\cdot}{C^{\odeorder-1}}]{\diffeq[A]}{\diffeq[\tilde{A}]}\underbrace{\leq C\sum_{i=0}^{\odeorder-1}\norm{A_i-\tilde{A}_i}_2}_{\text{Structural stability}},
    \end{equation}
    where $c,C>0$ are constants that only depend on $\horizon,K,D,\odeorder,\odedim$.
\end{theorem}
\begin{proof}
    See Appendix \ref{app:dist_lu_linearODE_high}.
\end{proof}
\textcolor{cc}{In a similar spirit, we can generalize the results from Theorem \ref{th:dist_lu_LipODE} to higher-order Lipschitz ODEs.}
\begin{theorem}\label{th:dist_lu_lipODE_high}
    Fix an arbitrary $D>0$ and a compact set $\compactInit\subset \ball{D} \defeq  \{x\in \R^{\odeorder\odedim}\mid \norm{y}_2\leq D\}$. Let $\widehat{\compactInit} \defeq \{x\in \R^{\odedim}\mid \norm{x}_2\leq (K\horizon+D) e^{\sqrt{L^2+1}\horizon}\}$. For every two $\odeorder$-th order Lipschitz ODEs $\diffeq[f],\diffeq[\tilde{f}]\in \odeClass[\odeorder]{\lip[\R^{\odeorder\odedim},\R^\odedim ]{L,K}}$ with corresponding $f, \tilde{f}\in \lip[\R^{\odeorder\odedim},\R^\odedim ]{L,K}$, we have that 
    \begin{equation}
    \begin{aligned}
        \underbrace{\min\left\{C_1\normFun[2]{f-\tilde{f}}{L^{\infty}(\compactInit)}^2, C_2\normFun[2]{f-\tilde{f}}{L^{\infty}(\compactInit)}\right\}
        \leq}_{\text{Identification stability}}& \hd[\compactInit][\normFun{\cdot}{C^{\odeorder-1}}]{\diffeq[f]}{\diffeq[\tilde{f}]}\\&\underbrace{\leq C_3\normFun[2]{f-\tilde{f}}{L^{\infty}(\widehat{\compactInit})}}_{\text{Structural stability}},
    \end{aligned}
    \end{equation}
    where $C_1, C_2,C_3>0$ are constants that only depend on $\horizon,L,D,K,\odeorder,\odedim$. 
\end{theorem}
\begin{proof}
    See Appendix \ref{app:dist_lu_lipODE_high}.
\end{proof}
One should note that $C^{\odeorder-1}$ norm instead of $L^{\infty}$ in Hausdorff distance for $m$-th order ODEs is necessary for the above results to hold, as shown by the following example.
\begin{example}
    Consider the ODE ($0<\epsilon<1$)
    \begin{equation*}
        \begin{aligned}
            \diffeq^\epsilon:  x^{(3)} + \frac{1}{\epsilon^2}x=0
        \end{aligned}
    \end{equation*}
    with intial conditions $x(0)=x_0$, $x^{(1)}(0)=x_1$, $x^{(2)}(0)=x_2$, $(x_0,x_1,x_2)^T\in [0,1]^3$. Solving this ODE, we obtain 
    \begin{equation*}
        x_\epsilon(t) = (x_0+\epsilon x_2) + \epsilon x_1 \sin\left(\frac{t}{\epsilon}\right) - \epsilon^2 x_2 \cos\left(\frac{t}{\epsilon}\right).
    \end{equation*}
    It is straight forward that 
    \begin{equation*}
        \hd[[0,1]^3][\normFun{\cdot}{L^{\infty}}]{\diffeq^\epsilon}{\diffeq^{2\epsilon}}\leq 5 \epsilon\rightarrow 0, \quad \epsilon\rightarrow0.
    \end{equation*}
    On the other hand, the distance in coefficients diverges to infinity
    \begin{equation*}
        \left| \frac{1}{\epsilon} - \frac{1}{2\epsilon}\right| = \frac{1}{2\epsilon} \rightarrow\infty, \quad \epsilon\rightarrow0.
    \end{equation*}
\end{example}

\subsection{First-order nonlinear ODEs: Hölder ODEs with applications to polynomial ODEs}
In this section, we extend the results to a wider class of ODEs--Hölder ODEs ($\odeClass{\holder{k,\alpha}{L,K}}$ as per Definition \ref{def:general_ode_class}) for dimension $\odedim=1$ and order $\odeorder=1$. Common examples of Hölder ODEs include ODEs with polynomial structures or sublinear structures. These ODEs exhibit behavior that is different from that of Lipschitz ODEs, with some not having a global solution, and some having multiple solutions for one initial value.
\begin{example}[Finite escaping time]\label{ex:FET}
Consider the ODE
\begin{equation*}
    \dot{x}=x^2
\end{equation*}
with $x(0)=x_0>0$. The solution is 
\begin{equation*}
    x(t) = \frac{x_0}{1-x_0t}.
\end{equation*}
and it is only defined for $t\in [0,1/x_0)$
\end{example}

\begin{example}[Multiple solutions]\label{ex:multi_sol}
    Consider the ODE
\begin{equation*}
    \dot{x}=2|x|^{\frac{1}{2}}
\end{equation*}
with $x(0)=0$. For every $t_0\geq 0$, the function 
\begin{equation*}
    x(t) = \left\{\begin{array}{cc}
        0 & t\leq t_0 \\
        (t-t_0)^2 & t>t_0
    \end{array}\right.
\end{equation*}
is a solution.
\end{example}
Despite these irregularities, the Hausdorff distance is still well-defined because it does not rely on the uniqueness or global existence of solutions with respect to initial value problems. Moreover, the identification results in the previous sections carry over to general Hölder ODEs. To this end, we first distinguish between three types of Hölder ODEs based on Definition \ref{def:general_ode_class}:
\begin{itemize}
    \item sublinear Hölder ODEs (Example \ref{ex:multi_sol}): $k=0$, $\alpha\in (0,1)$;
    \item Lipschitz ODEs: $k=0$, $\alpha=1$;
    \item superlinear Hölder ODEs (Example \ref{ex:FET}): $k\in \Nplus$, $\alpha\in (0,1]$.
\end{itemize}
Since the second type has been studied in the previous sections, we focus on sublinear (Section \ref{sec:sub_holder}) and superlinear Hölder ODEs in the following sections. Moreover, as we are more interested in identifying structures based on solution data, we exclusively study the lower bound for Hausdorff distances.

\subsubsection{Sublinear Hölder ODEs}\label{sec:sub_holder}
We first study sublinear Hölder ODEs. In this case, we have the following identification results by slightly adjusting the proof of Lemma \ref{lm:dist_lower_LipODE_local}.
\begin{lemma}\label{lm:dist_lower_sublinear_ODE_local}
    Assume $\alpha\in (0,1)$. Fix an arbitrary $D>0$ and a compact set $\compactInit\subset \ball{D} \defeq  \{x\in \R\mid |x|\leq D\}$. For every two sublinear Hölder ODEs $\diffeq[f],\diffeq[\tilde{f}]\in \odeClass{\holder{0,\alpha}{L,K}}$ with corresponding $f, \tilde{f}\in \holder{0,\alpha}{L,K} $, we have that 
    \begin{equation}
    \begin{aligned}
        &\hd[\compactInit][\normFun{\cdot}{L^{\infty}}]{\diffeq[f]}{\diffeq[\tilde{f}]}\\
        &\geq C\min\left\{\normFun{f-\tilde{f}}{L^{\infty}(\compactInit)},\normFun{f-\tilde{f}}{L^{\infty}(\compactInit)}^{\frac{\alpha+1}{\alpha}},\normFun{f-\tilde{f}}{L^{\infty}(\compactInit)}^{\frac{1}{\alpha}},\normFun{f-\tilde{f}}{L^{\infty}(\compactInit)}^{\frac{\alpha+1}{\alpha^2}}\right\},
    \end{aligned}
    \end{equation}
    where $C>0$ is a constant depending only on $\alpha, L,K,\horizon,D$.
\end{lemma}
\begin{proof}
    See Appendix \ref{app:dist_lower_sublinear_ODE_local}.
\end{proof}
One should note that for sublinear ODEs, solutions are generally non-unique for initial value problems (Example \ref{ex:multi_sol}). However, the Hausdorff distance is well-defined as it depends on the solution set instead of a single solution with respect to a fixed initial value. Moreover, one can see that the uniqueness of solutions is not required in the proof in Appendix \ref{app:dist_lower_sublinear_ODE_local}.

\subsubsection{Superlinear Hölder ODEs}
Next, we study superlinear Hölder ODEs. In this scenario, solutions could explode within a finite time (Example \ref{ex:FET}). Thus, to study the class of superlinear Hölder ODEs, the horizon $\horizon$ should be suitably chosen such that the solutions of each ODE in this class is at least well-defined on $[0,\horizon]$. 

In the following lemma, we provide conditions on $\horizon$ and prove the identification results by noting that the superlinear Hölder ODEs we consider are nothing but Lipschitz ODEs for $t\in[0,\horizon]$.
\begin{lemma}\label{lm:dist_lower_superHolderODE}
For $D,L,K>0$, $k\in \Nplus$, $\alpha\in (0,1]$, let $$\horizon< ((K+L)e)^{-1}\phi_{D,k,\alpha}^{-1}\left(\frac{1}{k+\alpha-1}\right),$$ where $\phi_{D,k,\alpha}(r) = r(r+D)^{k+\alpha-1}$, for $r\in \R_+$. Now, given a compact set $\compactInit\subset \ball{D} \defeq  \{x\in \R\mid |x|\leq D\}$, for every two superlinear Hölder ODEs $\diffeq[f],\diffeq[\tilde{f}]\in \odeClass{\holder{k,\alpha}{L,K}}$ with corresponding $f, \tilde{f}\in \holder{k,\alpha}{L,K} $, we have that 
    \begin{equation}
        \hd[\compactInit][\normFun{\cdot}{L^{\infty}}]{\diffeq[f]}{\diffeq[\tilde{f}]}\geq \min\left\{C_1\normFun[2]{f-\tilde{f}}{L^{\infty}(\compactInit)}^2, C_2\normFun[2]{f-\tilde{f}}{L^{\infty}(\compactInit)}\right\},
    \end{equation}  
    where $C_1,C_2>0$ are constants depending only on $k,\alpha,L,K,\horizon,D$.
\end{lemma}
\begin{proof}
    See Appendix \ref{app:dist_lower_superHolderODE}
\end{proof}

\subsubsection{Polynomial ODEs}
As per Definition \ref{def:general_ode_class}, it is straightforward that $\odeClass{\poly{q,K}}\subset \odeClass{\holder{q-1,1}{q!,K\sum_{i=0}^{q-1}  i!}}$. Thus results in Lemma \ref{lm:dist_lower_superHolderODE} can be applied to $\odeClass{\poly{q,K}}$. Moreover, due to the properties of polynomials, we obtain the following results.
\begin{corollary}\label{col:poly_ode_id}
    Let $I_1\subset \ball{D} \defeq  \{x\in \R\mid |x|\leq D\}$ be a compact set with non-empty interior and let $I_2=\{x_1,x_2,\dots,x_q\}\in \R$ be a set of $q$ distinct real numbers. Then, for every two polynomial ODEs $\diffeq[p],\diffeq[\tilde{p}]\in \odeClass{\poly{q,K}}$ with corresponding $p,\tilde{p}\in \poly{q,K}$, the following holds:
    \begin{enumerate}
        \item[(i).]  $\hd[I_1][\normFun{\cdot}{L^{\infty}}]{\diffeq[p]}{\diffeq[\tilde{p}]}=0$ implies $p=\tilde{p}$.
        \item[(ii).] $\hd[I_2][\normFun{\cdot}{L^{\infty}}]{\diffeq[p]}{\diffeq[\tilde{p}]}=0$ implies $p=\tilde{p}$.
    \end{enumerate}
\end{corollary}

\subsection{\textcolor{rev1}{Numerics showing bounds}}
We empirically examine the validity of lower and upper bounds on Hausdorff distance between ODEs, with a focus on first-order Linear ODEs and Lipschitz ODEs (Theorem \ref{th:dist_lu_linearODE_const} \& Theorem \ref{th:dist_lu_LipODE}). Specifically, we consider ODEs of dimension $\odedim=2$ on a fixed time horizon $t\in [0,0.1]$ and consider initial value set $\ball{1} = \{x\in \R^2\mid \norm{x}_2\leq 1\}$. For numerical purpose, we discretize the time horizon and initial value set into 
\begin{equation*}
    \begin{aligned}
        \mathcal{T} &= \{t_0,t_1,\dots, t_{20}\mid \text{where }t_i = 0.005i\}, \\
        \widehat{\ball{1}} &= \{x_{i,j}=\begin{pmatrix}
            r_i\cos (\theta_j) & r_i\sin (\theta_j)
        \end{pmatrix}^T \mid r_i = 0.25 i, \theta = \pi j/5, ~i,j\in\{0,1,\dots,4\}^2 \}.
    \end{aligned}
\end{equation*}
For each ODE and each initial value $x_0\in \widehat{\ball{1}}$, trajectories are computed on $\mathcal{T}$ using a fixed explicit integrator (RK4).
\subsubsection{Linear ODEs}
For numerical experiments on linear ODEs, we construct an ensemble of systems by first sampling 30 ``base'' matrices $\{\tilde{A}_1,\tilde{A}_2,\dots,\tilde{A}_{30}\}$ where $\tilde{A}_i$ are independent random matrices, and then for each ``base'' matrix $\tilde{A}_k,k\in \{1,2,\dots,30\}$, we generate 20 perturbed variants of the form $\tilde{A}_k + 0.02B_{k}^l$, where $B_{k}^l, l\in \{1,2,\dots,20\}$ independent random matrices. This yields a collection of 600 matrices, which we denote by $\{A_1,A_2,\dots, A_{600}\}$. 

Figure \ref{fig:hd_linear} reports the pairwise Hausdorff distance between linear ODEs $\diffeq[A_i]$ and $\diffeq[A_j]$: $\hd[\widehat{\ball{1}}][\normFun[2]{\cdot}{L^{\infty}(\mathcal{T})}]{\diffeq[A_i]}{\diffeq[A_j]}$ versus the pairwise structure distances between matrices: $\norm{A_i-A_j}_2$. Each point thus represents one pair of linear ODEs, with the x-axis measuring the size of the perturbation in matrix space and the y-axis the induced discrepancy between solution sets--Hausdorff distance between two linear ODEs. The yellow and blue lines are the predicted lower and upper bounds in line with Theorem \ref{th:dist_lu_linearODE_const}.
\begin{figure}[H]
    \centering
    \includegraphics[width=0.8\linewidth]{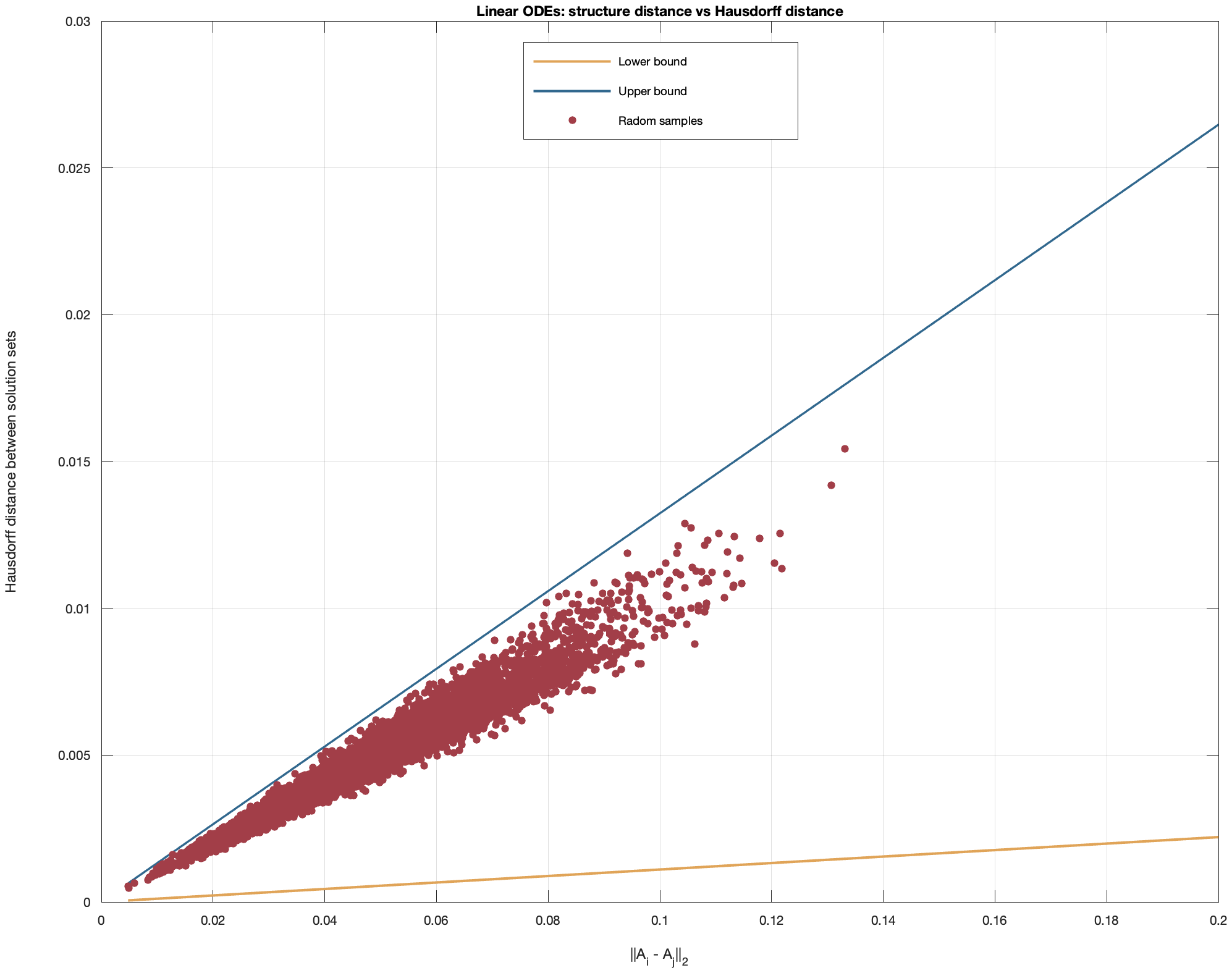}
    \caption{Red points represent the randomly generated samples. Yellow and blue lines represent the lower bound and upper bound on the Hausdorff distance between linear ODEs with respect to structure distances between matrices based on Theorem \ref{th:dist_lu_linearODE_const}.}
    \label{fig:hd_linear}
\end{figure}

\subsubsection{Lipschitz ODEs}
For the numerical experiments on Lipschitz ODEs, we generate random Lipschitz vector fields using a two-layer ReLU network parametrization
$$f(x)=A_2\operatorname{ReLU}(A_1x+b_1) +b_2.$$ The construction proceeds as follows. First, we sample 30 base parameter tuples
\(\{\tilde{A}_1^k,\tilde{A}_2^k,\tilde{b}_1^k,\tilde{b}_2^k\}_{k=1}^{30}\),
with \(\tilde{A}_1^k \in \mathbb{R}^{5\times 2}\), \(\tilde{A}_2^k \in \mathbb{R}^{2\times 5}\),
\(\tilde{b}_1^k \in \mathbb{R}^{5}\), and \(\tilde{b}_2^k \in \mathbb{R}^{2}\) drawn independently at random.
For each base index \(k\), we then generate 20 perturbed parameter sets indexed by
\(l \in \{1,\dots,20\}\) via
\[
\begin{aligned}
A_1^{k,l} &= \tilde{A}_1^k + B_1^{k,l}, \quad &&B_1^{k,l} \in \mathbb{R}^{5\times 2},\\
A_2^{k,l} &= \tilde{A}_2^k + B_2^{k,l}, \quad &&B_2^{k,l} \in \mathbb{R}^{2\times 5},\\
b_1^{k,l} &= \tilde{b}_1^k + c_1^{k,l}, \quad &&c_1^{k,l} \in \mathbb{R}^{5},\\
b_2^{k,l} &= \tilde{b}_2^k + c_2^{k,l}, \quad &&c_2^{k,l} \in \mathbb{R}^{2},
\end{aligned}
\]
where \(B_1^{k,l}, B_2^{k,l}, c_1^{k,l}, c_2^{k,l}\) are independent random perturbations.
This yields an ensemble of 600 Lipschitz vector fields \(\{f_1,\dots,f_{600}\}\), where each \(f_i\)
corresponds to some quadruple \((A_1^{i},A_2^{i},b_1^{i},b_2^{i})\defeq (A_1^{k,l},A_2^{k,l},b_1^{k,l},b_2^{k,l})\) and is written as
\[
f_i(x) = A_2^i \operatorname{ReLU}(A_1^i x + b_1^i) + b_2^i.
\]

The global Lipschitz constant \(L\) and the affine growth constant \(K\) in
Theorem~\ref{th:dist_lu_LipODE} are determined from this ensemble as
\[
L = \max_{i = 1,\dots,600} \bigl\{\norm{A_2^i}_2\norm{A_1^i}_2\bigr\}, 
\qquad
K = \max_{i = 1,\dots,600} \bigl\{\norm{A_2^i}_2\norm{b_1^i}_2 + \norm{b_2^i}_2\bigr\}.
\]

Figure~\ref{fig:hd_lip} reports, for all indices \(i,j \in \{1,\dots,600\}\), the pairwise Hausdorff distance between Lipschitz ODEs $\diffeq[f_i]$ and $\diffeq[f_j]$: $\hd[\widehat{\ball{1}}][\normFun[2]{\cdot}{L^{\infty}(\mathcal{T})}]{\diffeq[f_i]}{\diffeq[f_j]}$ (on a logrithmic scale) versus the pairwise structure distances between Lipschitz functions: $\normFun[2]{f_i-f_j}{L^{\infty}}$. Each point in the figure thus represents one
pair of Lipschitz ODEs, with the horizontal axis encoding the size of the perturbation in function
space and the vertical axis the induced discrepancy between their solution sets--Hausdorff distance between two Lipschitz ODEs. The yellow and blue lines are the predicted lower and upper bounds in line with Theorem \ref{th:dist_lu_LipODE}.

\begin{figure}[H]
    \centering
    \includegraphics[width=0.8\linewidth]{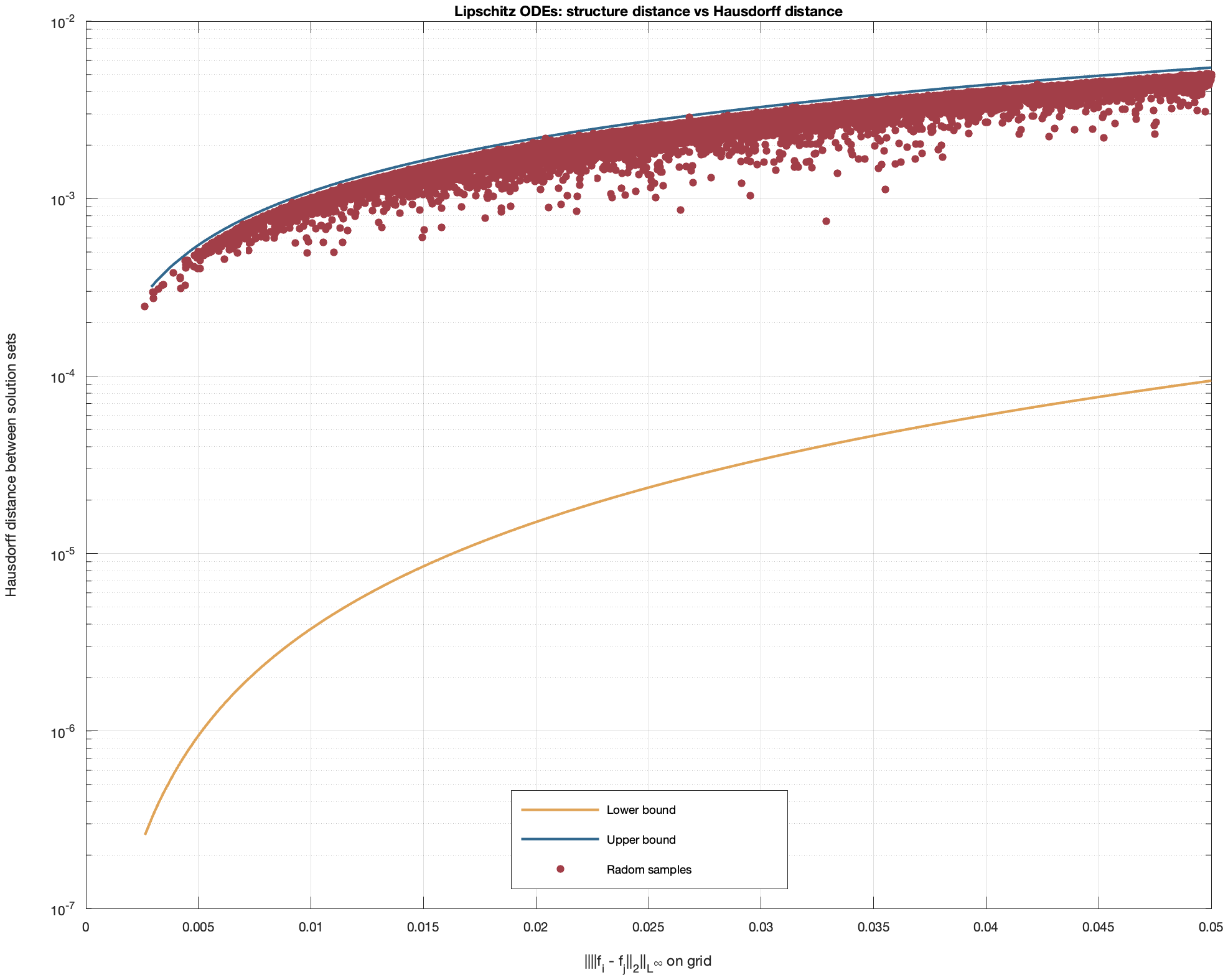}
    \caption{Red points represent the randomly generated samples. Yellow and blue lines represent the lower bound and upper bound on the Hausdorff distance between Lipschitz ODEs with respect to structure distances between Lipschitz functions based on Theorem \ref{th:dist_lu_LipODE}.}
    \label{fig:hd_lip}
\end{figure}
It is worth noting that our numerical experiments suggest that the identification lower bound derived for Lipschitz ODEs is not tight: the observed separation rates indicate room for improvement at the level of the exponents. Establishing sharp bounds for this class is left as an open problem for future work.

\section{Complexity in \textcolor{subj}{learning governing equations of ODEs}}\label{sec:complexity}
In this section, we analyze the complexity in learning of structure equations for linear and Lipschitz ODEs based on Hausdorff-distance-based identification results from previous sections. Specifically, we study
\begin{itemize}
    \item how ``rich'' the class of ODEs is in terms of Hausdorff distance based on metric entropy \cite{tikhomirovEEntropyECapacitySets1993} (Section \ref{sec:metric_ent_ode}),
    \item and how many solution samples $\{(t_i,x_k(t_i))_{i=1}^{M}\}_{k=1}^{N}$ we need to learn the structure equation (Section \ref{sec:sample_complexity}).
\end{itemize}
\textcolor{cc}{To the best of our knowledge, this work is the first to study the metric entropy and sample complexity in learning ODEs.} In what follows, we restrict the discussions to first-order ODEs, as we can see from Section \ref{sec:higher_order} that higher-order ODEs are merely a generalization of first-order ODEs.

\subsection{Information-theoretic complexity in \textcolor{subj}{learning governing equations of ODEs}}\label{sec:metric_ent_ode}
In this section, we study the ``richness of'' classes of first-order linear ODEs based on metric entropy. 

\textcolor{cc}{To motivate the analysis in this section, we briefly recall several notions related to metric entropy. Originating in the work of Kolmogorov, metric entropy was introduced as a way to quantify the complexity of compact subsets of metric spaces via covering or packing numbers \cite{tikhomirovEEntropyECapacitySets1993}.}

\textcolor{cc}{Closely related but conceptually distinct is Kolmogorov–Sinai (KS) entropy, developed in ergodic theory and dynamical systems to measure the rate at which complexity is generated along the temporal evolution of a single dynamical system \cite{kolgomorov1959entropyds,sinai1959notion}. Although dynamical systems are closely related to ODEs, the perspective in our paper differs fundamentally from that of KS entropy. Rather than quantifying the growth of complexity along trajectories of a fixed system, we are interested in using metric entropy to analyze the complexity of an entire class of ODEs by viewing it as a subset of an appropriate metric space.}

\textcolor{cc}{In this spirit, several works have studied metric entropy for families of dynamical systems by representing each system as an input–output operator and analyzing the resulting operator class \cite{hutter2022metric,pan2024metric,zames1977metric}.}

\textcolor{cc}{The most closely related work to the present setting is \cite{zhu2020odesolutioncovering}, which investigates the metric entropy of the solution class induced by a class of ODEs: solutions are aggregated into a set, whose covering complexity is then bounded in terms of the complexity of the underlying set of governing vector fields (or equations) \cite{zhu2020odesolutioncovering}.}

\textcolor{cc}{But unlike \cite{zhu2020odesolutioncovering}, we are not mixing the solutions of different ODEs in a set. Instead, we consider the set of solution sets induced by the class of ODEs. Each element in the set is the entire solution set of a specific ODE. And the metric used in analyzing the complexity is the Hausdorff distance. This unique and innovative way allows us to obtain identification results, as shown in Section \ref{sec:hd_bound_ode}. In this section, we are interested in how the prescribed perspective allows us to study the ``richness of'' classes of first-order linear ODEs based on metric entropy.}


To this end, we start by noting that the Hausdorff distance is indeed a well-defined metric for the classes of linear ODEs.
\begin{lemma}
    Fix $D>0$ and consider $\ball{D}\defeq \{x\in \R^{\odedim}\mid \norm{x}_2\leq D\}$. The Hausdorff distance $\hd[\ball{D}][\normFun{\cdot}{L^{\infty}}]{\cdot}{\cdot}$ is a well-defined metric for the class of linear ODEs $\odeClass{\lin[\R^{\odedim},\R^\odedim ]{K}}$.
\end{lemma}
\begin{proof}
    Basically, we need to prove that $\hd[\ball{D}][\normFun{\cdot}{L^{\infty}}]{\cdot}{\cdot}$ satisfies the metric axioms on $\odeClass{\lin[\R^{\odedim},\R^\odedim ]{K}}$. The symmetry, triangle inequality and $\hd[\ball{D}][\normFun{\cdot}{L^{\infty}}]{\diffeq}{\diffeq}=0$ immediately follow from Definition \ref{def:hd}. For linear ODEs, we know from Theorem \ref{th:dist_lu_linearODE_const} that for every $A_1,A_2\in \lin[\R^{\odedim},\R^\odedim ]{K}$, $$\hd[\ball{D}][\normFun{\cdot}{L^{\infty}}]{\diffeq[A_1]}{\diffeq[A_2]}=0 ~\Leftrightarrow 
    ~A_1=A_2~\Leftrightarrow ~\diffeq[A_1]=\diffeq[A_2],$$
    which proves positivity.
\end{proof}
Then, we have the following results regarding the metric entropy of the class of linear ODEs.
\begin{theorem}\label{th:ent_lin_ode}
    Fix $D>0$ and consider $\ball{D}\defeq \{x\in \R^{\odedim}\mid \norm{x}_2\leq D\}$. For all $\epsilon>0$, the metric entropy of the class of linear ODEs $\lin[\R^{\odedim},\R^\odedim ]{K}$ satisfies 
    \begin{equation}
        \log\covering (\epsilon;\odeClass{\lin[\R^{\odedim},\R^\odedim ]{K}},\hd[\ball{D}][\normFun{\cdot}{L^{\infty}}]{\cdot}{\cdot}))\asymp \odedim^2\log (\epsilon^{-1}).
    \end{equation}
\end{theorem}
\begin{proof}
    See Appendix \ref{app:ent_lin_ode}.
\end{proof}
This means that the metric entropy of linear ODEs is \textcolor{cc}{essentially determined by the metric entropy of structure matrices characterizing the ODEs}

Things become tricky for Lipschitz ODEs. According to Theorem \ref{th:dist_lu_LipODE}, the Hausdorff distance of two Lipschitz ODEs with respect to a compact initial set $\compactInit$ being zero only implies structural equivalence restricted to $\compactInit$. Conversely, structural equivalence restricted to $\compactInit$ is not sufficient for the Hausdorff distance to be zero. \textcolor{cc}{Therefore, for the purposes of analysis, we restrict the scope of this section to a special subclass of Lipschitz ODEs—namely, ODEs with compactly supported Lipschitz vector fields. This assumption is mainly technical, but it does not undermine the reasonableness of the perspective taken here.}

\textcolor{cc}{First, metric-entropy arguments for Lipschitz functions are typically stated for functions defined on a bounded domain. For ODEs, however, working with a vector field that is only locally defined is inconvenient: even if the dynamics are well-defined in a neighborhood, the trajectory may exit that neighborhood, and the solution map is then no longer globally meaningful over the time horizon of interest.}

\textcolor{cc}{Second, while true governing vector fields are rarely compactly supported, compact support can be imposed without changing the dynamics on the trajectories under consideration. Indeed, if the initial conditions are restricted to a bounded set, then solutions of a Lipschitz ODE remain bounded on any finite time interval. Consequently, all relevant trajectories lie in some (possibly large) bounded region of state space. One may therefore modify the vector field outside this region—e.g., by multiplying it with a smooth cutoff that equals one on the trajectory-containing set—thereby obtaining a compactly supported vector field that agrees with the original dynamics along all trajectories considered in this paper.}

\textcolor{cc}{In the following, we introduce the class of Lipschitz ODEs with compactly supported vector fields.} For simplicity, we restrict our scope to ODEs with state dimension of $1$.
\begin{definition}
    Fix $D>0$ and consider $\ball{D}\defeq \{x\in \R\mid |x|\leq D\}$. Consider the class of Lipschitz ODEs with compact support 
    \begin{equation*}
        \mathcal{F}_{L,K,D} \defeq \{f:\R\rightarrow \R\mid \operatorname{Lip}(f)\leq L,|f(0)|\leq K,\operatorname{supp}(f)\in \ball{D}\}\subset \lip[\R,\R]{L,K}.
    \end{equation*}
\end{definition}
Next, we show that $\hd[\ball{D}][\normFun{\cdot}{L^{\infty}}]{\cdot}{\cdot}$ is a well-define metric for $\odeClass{\mathcal{F}_{L,K,D}}$.
\begin{lemma}
    Fix $D>0$ and consider $\ball{D}\defeq \{x\in \R^{\odedim}\mid \norm{x}_2\leq D\}$. The Hausdorff distance $\hd[\ball{D}][\normFun{\cdot}{L^{\infty}}]{\cdot}{\cdot}$ is a well-defined metric for the class of Lipschitz ODEs $\odeClass{\mathcal{F}_{L,K,D}}$.    
\end{lemma}
\begin{proof}
    We need to prove that $\hd[\ball{D}][\normFun{\cdot}{L^{\infty}}]{\cdot}{\cdot}$ satisfies the metric axioms on $\odeClass{\mathcal{F}_{L,K,D}}$. The symmetry, triangle inequality and $\hd[\ball{D}][\normFun{\cdot}{L^{\infty}}]{\diffeq}{\diffeq}=0$ immediately follow from Definition \ref{def:hd}. Applying Theorem \ref{th:dist_lu_LipODE} for $\odedim=1$ and noting that the supporting sets of functions in $\mathcal{F}_{L,K,D}$ are susbets of $\ball{D}$, we obtain 
    \begin{equation*}
    \begin{aligned}
        \min\left\{C_1\normFun{f-\tilde{f}}{L^{\infty}(\R)}^2, C_2\normFun{f-\tilde{f}}{L^{\infty}(\R)}\right\}
        \leq& \hd[\ball{D}][\normFun{\cdot}{C^{\odeorder-1}}]{\diffeq[f]}{\diffeq[\tilde{f}]}\\&\leq C_3\normFun{f-\tilde{f}}{L^{\infty}(\R)}.
    \end{aligned}
    \end{equation*}
    This implies for every $f,\tilde{f}\in \mathcal{F}_{L,K,D}$, $$\hd[\ball{D}][\normFun{\cdot}{L^{\infty}}]{\diffeq[f]}{\diffeq[\tilde{f}]}=0 ~\Leftrightarrow 
    ~f=\tilde{f}~\Leftrightarrow ~\diffeq[f]=\diffeq[\tilde{f}],$$
    which proves positivity.
\end{proof}
Finally, we have the following results regarding the metric entropy of the class of Lipschitz ODEs $\odeClass{\mathcal{F}_{L,K,D}}$.
\begin{theorem}\label{th:ent_lip_ode}
    For all $\epsilon>0$, the metric entropy of the class of Lipschitz ODEs $\odeClass{\mathcal{F}_{L,K,D}}$ satisfies 
    \begin{equation}
       \log^{(2)}\covering (\epsilon;\odeClass{\mathcal{F}_{L,K,D}},\hd[\ball{D}][\normFun{\cdot}{L^{\infty}}]{\cdot}{\cdot})\asymp \log (\epsilon^{-1}).
    \end{equation}
\end{theorem}
\begin{proof}
    See Appendix \ref{app:ent_lip_ode}.
\end{proof}
\textcolor{cc}{This implies that the metric entropy of Lipschitz ODEs has asymptotically the same growth rate as that of Lipschitz functions on the log scale.}

\subsection{Sample complexity in \textcolor{subj}{learning governing equations of ODEs}}\label{sec:sample_complexity}
In this section, we analyze the sample complexity in learning linear and Lipschitz ODEs. \textcolor{cc}{We would like to mention that the study of sample complexity for learning dynamical systems—whether for learning input–output maps or for identifying (typically linear) state-space models—is a widely researched topic \cite{oymak2019sclti,campi2002finitesamplesi}. By contrast, the corresponding sample-complexity analysis in identifying the structure equations of general (including nonlinear) ODEs remains comparatively less explored—which we shall discuss in the remainder of this section.} 

Specifically, we consider samples $\{(t_i,x_k(t_i))_{i=1}^{M}\}_{k=1}^{N}$ generated through experiments 
or numerical simulations. Furthermore, the measurement errors in sample generation—including random experimental noise, numerical errors, and truncation errors arising from simulations—are assumed to admit a deterministic uniform upper bound $\error[meas]$. We summarize the above ideas in Lemma \ref{lm:general_train_hd} and show how to learn the structure equation within a prescribed hypothesis class in a general principle.
\textcolor{cc}{
\begin{lemma}\label{lm:general_train_hd}
    Given a function class $\mathcal{F}$ with norm (or semi-norm) $\normFun{\cdot}{\mathcal{F}}$, the corresponding class of ODEs $\odeClass{\mathcal{F}}$ as per Definition \ref{def:ode_classes} and an initial value set $\initset$ for ODEs in $\odeClass{\mathcal{F}}$, we assume the following:
    \begin{enumerate}[label=(\arabic*)]
        \item \label{item:sc-id}\textbf{Identification guarantee}: There exists a function $\phi:\R_+\rightarrow \R_+$ with $\lim_{a\rightarrow0^+}\phi(a)=0$ and $\phi(0)=0$, such that for every two ODEs $\diffeq[f_1], \diffeq[f_2]\in \odeClass{\mathcal{F}}$, with corresponding $f_1,f_2\in \mathcal{F}$, the following holds 
    \begin{equation}\label{eq:lm_f1f2}
         \normFun{f_1-f_2}{\mathcal{F}}\leq \phi\left(\hd[\initset][\normFun{\cdot}{L^{\infty}}]{\diffeq[f_1]}{\diffeq[f_2]}\right).
    \end{equation}
    \item\label{item:sc-fs-meas} \textbf{Finite-sampling and bounded measurement error:} Given a function $f\in \mathcal{F}$, we can sample $N$ solutions $X\defeq \{x_k(t)\}_{k=1}^N\subset \mathcal{F}$ at time $0= t_1\leq t_2\leq \dots\leq t_M=\horizon$, leading to a sampling set $\dis{X}\defeq \{(t_i,\hat{x}_{k,i})_{i=1}^{M}\}_{k=1}^{N}$ such that    \begin{equation}
        \hd{X}{\behavior[\initset]{\diffeq[f]}}\leq \error[sampling](N),
    \end{equation}
    where $\error[sampling](N)$ is the sampling error depending on the number of solution samples $N$, and
    \begin{equation}
       \norm{\hat{x}_{k,i}-x_k(t_i)}_2\leq \error[meas],\quad \text{for all }k=1,2,\dots,N, \text{ and }i=1,2,\dots,M,
    \end{equation}
    where $\error[meas]$ is the measurement error. Here, $\dis{X}$ can be interpreted as the set of sampled solutions with respect to ODE $\diffeq[f]$ on finite times resolution.
    \item \label{item:sc-num}\textbf{Numerical error:} We can numerically construct (e.g., by interpolation) a function set $\cont{\dis{X}}\defeq\{\hat{x}_k\}_{k=1}^N\subset \mathcal{F}$, where $\hat{x}_k(t)$ depends on $(t_i,\hat{x}_{k,i})_{i=1}^{M}$, such that 
    \begin{equation}
        \normFun[2]{\hat{x}_k-x_k}{L^{\infty}([0,\horizon])}\leq \error[num](\error[meas],M),
    \end{equation}
    where $\error[num]$ is the numerical error depending on both the measurement errors $\error[meas]$ and number of time steps $M$ used in sampling. Here, $\cont{\dis{X}}$ is the approximation of the set of sampled solutions with respect to ODE $\diffeq[f]$ based on reconstructions from $\dis{X}$.
    \item \label{item:sc-train}\textbf{Training error:} We can find a function $\tilde{f}$ (typically by training a model) such that 
    \begin{equation}
        \hdset[\normFun{\cdot}{L^{\infty}}]{\behavior[\initset]{\diffeq[\tilde{f}]}}{\cont{\dis{X}}}\leq \error[train].
    \end{equation}
    \end{enumerate}
    Then, we have 
    \begin{equation}
        \normFun{f-\tilde{f}}{\mathcal{F}}\leq \phi\left(\error[total]\right),
    \end{equation}
    where 
    \begin{equation}
        \error[total] = \error[sampling](N)+\error[num](\error[meas],M)+\error[train].
    \end{equation}
\end{lemma}}
\begin{proof}
    Note that
    \begin{equation*}
        \begin{aligned}
            \hd[\initset][\normFun{\cdot}{L^{\infty}}]{\diffeq[f]}{\diffeq[\tilde{f}]} \overset{\text{Definition \ref{def:hd}}}&{=} \hdset[\normFun{\cdot}{L^{\infty}}]{\behavior[\initset]{\diffeq[f]}}{\behavior[\initset]{\diffeq[\tilde{f}]}}\\
            \overset{\text{triangle inequality}}&{\leq} \hdset[\normFun{\cdot}{L^{\infty}}]{\behavior[\initset]{\diffeq[f]}}{X}+\hdset[\normFun{\cdot}{L^{\infty}}]{X}{\cont{\dis{X}}}\\
            & \qquad\quad +\hdset[\normFun{\cdot}{L^{\infty}}]{\cont{\dis{X}}}{\behavior[\initset]{\diffeq[\tilde{f}]}}\\
            &\leq  \max\left\{\sup_{x\in X}\inf_{\hat{x}\in \cont{\dis{X}}} \normFun[2]{x-\hat{x}}{L^{\infty}},\sup_{\hat{x}\in \cont{\dis{X}}}\inf_{x\in X} \normFun[2]{x-\hat{x}}{L^{\infty}}\right\}\\
            & \quad+\error[sampling](N)+\error[train]\\
            &\leq \sup_{k\in \{1,2,\dots, N\}}\normFun[2]{\hat{x}_k-x_k}{L^{\infty}}+ \error[sampling](N)+\error[train]\\
            &\leq \error[sampling](N)+\error[num](\error[meas],M)+\error[train].
        \end{aligned}
    \end{equation*}
    Applying \eqref{eq:lm_f1f2} concludes the proof.
\end{proof}
\textcolor{cc}{The lemma has some practical meanings. Most importantly, it provides a principled basis for the design of a numerical learning objective for recovering the governing differential equation. Specifically, if 
\begin{itemize}
    \item the ODE class $\odeClass{\mathcal{F}}$ is identifiable according to \ref{item:sc-id} in Lemma \ref{lm:general_train_hd},
    \item we can sample the solution set of the ODE $\diffeq[f]\in \odeClass{\mathcal{F}}$ up to some sampling, numerical, and measurement error and collect them in $\cont{\dis{X}}$
\end{itemize}
then it is natural to train a model $\tilde{f}$—potentially a neural-network model, or a sparse superposition over a prescribed library of candidate functions as in SINDy—by minimizing the loss function $$ \operatorname{Loss}\defeq \hdset[\normFun{\cdot}{L^{\infty}}]{\behavior[\compactInit]{\diffeq[\tilde{f}]}}{\cont{\dis{X}}},$$
which is the Hausdorff distance between the solution set of our model and the sampled solution set. Lemma \ref{lm:general_train_hd} guarantees that any model that achieves a small training loss necessarily yields a small error in the recovered dynamics—the learning error between $f$ and $\tilde{f}$ can be bounded above by the sampling, numerical, measurement, and training error. Note that this idea works for any ODEs class that satisfies the identification guarantee in Lemma \ref{lm:general_train_hd}}

\textcolor{cc}{Now, we apply Lemma \ref{lm:general_train_hd} to ODE classes in Section \ref{sec:hd_bound_ode}, for which the identification guarantee actually holds—namely, linear and Lipschitz ODEs.} Since linear ODEs are merely a special case of Lipschitz ODEs, we first present the result on the sample complexity in learning Lipschitz ODEs. 

\begin{lemma}[Sampling results for Lipschitz ODEs]\label{lm:sample_learn_lip} 
Arbitrarily fix $N_{x},N_{t}\in \Nplus$, $D>0$ and consider $\compactInit \defeq [-D,D]^\odedim$. Given a target ODE $\diffeq[f]\in \odeClass{\lip[\R^{\odedim},\R^\odedim ]{L,K}}$ with corresponding $f\in \lip[\R^{\odedim},\R^\odedim ]{L,K}$, assume we can generate solutions samples from $\behavior[\compactInit]{\diffeq[f]}$ according to 
\begin{equation}
    X = \left\{x_{\vecindex}\in \behavior[\compactInit]{\diffeq[f]}\mid x_{\vecindex}(0) = \frac{D}{N_x}\vecindex, \text{ where }\vecindex \in \{-N_x,-N_x+1,\dots, N_x\}^\odedim\right\},
\end{equation}
and for each of the solution $x_{\vecindex} \in X$, we can measure the solution uniformly in time according to 
\begin{equation}
    \dis{X}\defeq \left\{(t_{\tau},\hat{x}_{\vecindex,\tau})\mid t_{\tau} =  \frac{\horizon}{N_t}\tau, \hat{x}_{\vecindex,\tau} = x_{\vecindex}\left(t_{\tau}\right) + w_{\vecindex,\tau}, x_{\vecindex}\in X,\tau\in\{0,1,\dots,N_t\}\right\},
\end{equation}
where $w_{\vecindex,\tau}$ is the measurement error admitting uniform bound $\norm{w_{\vecindex,\tau}}_2\leq \error[meas]$. Then, we construct 
\begin{equation}
\begin{aligned}
    \cont{\dis{X}}\defeq \left\{ \hat{x}_{\vecindex}:[0,\horizon]\rightarrow \R^{\odedim}\mid \right.&\hat{x}_{\vecindex} \text{ is a piecewise linear interpolation such that }\\
    &\left.\hat{x}_{\vecindex}\left(\frac{T}{N_t}\tau\right)=\hat{x}_{\vecindex,\tau}\text{ for } x_\vecindex\in X,\tau\in\{0,1,\dots,N_t\}\right\}
\end{aligned}
\end{equation}
as the (training) solution data. Now, suppose we can find a function $\tilde{f}\in  \lip[\R^{\odedim},\R^\odedim ]{L,K}$ (typically by training a model) such that 
    \begin{equation}
        \hdset[\normFun{\cdot}{L^{\infty}}]{\behavior[\compactInit]{\diffeq[\tilde{f}]}}{\cont{\dis{X}}}\leq \error[train].
    \end{equation}
Then, we have the learning bound
\begin{equation}
\begin{aligned}
   &\normFun[2]{f-\tilde{f}}{L^{\infty}(\compactInit)}\\
   &\leq 2\sqrt{\odedim}(2+L\horizon)\max\left\{\frac{\error[total]}{\horizon},\sqrt{\left(2\sqrt{\odedim}L\left(2\sqrt{\odedim}D L\horizon e^{L\horizon}  + KL\horizon^2 e^{L\horizon} + 2K\right)\right)\error[total]}\right\},    
\end{aligned}
\end{equation}
where 
\begin{equation}
    \error[total]=\error[meas]+\error[train]+\frac{\sqrt{\odedim}De^{L\horizon}}{N_x}+\frac{\left( K+L(D+K\horizon)e^{L\horizon}\right)\horizon}{4N_t}.
\end{equation}
\end{lemma}
\begin{proof}
    We follow the procedure outlined in Lemma \ref{lm:general_train_hd}. 
    \begin{enumerate}
        \item[(i).] Determine the form of $\phi$. From Lemma \ref{lm:dist_lower_LipODE_local}, we know that $$\compactInit \subset \ball{\sqrt{\odedim}D} \defeq  \{x\in \R^{\odedim}\mid \norm{x}_2\leq \sqrt{\odedim}D\}$$ and thus for every two ODEs $\diffeq[f_1], \diffeq[f_2]\in \odeClass{\lip[\R^{\odedim},\R^\odedim ]{L,K}}$, with corresponding $f_1,f_2\in \lip[\R^{\odedim},\R^\odedim ]{L,K}$, the following holds 
    \begin{equation*}
    \begin{aligned}
        \hd[\compactInit][\normFun{\cdot}{L^{\infty}}]{\diffeq[f_1]}{\diffeq[f_2]}\geq &\min\left\{\frac{\normFun[2]{f_1-f_2}{L^{\infty}(\compactInit)}}{2\sqrt{\odedim}L\left(2\sqrt{\odedim}D L\horizon e^{L\horizon}  + KL\horizon^2 e^{L\horizon} + 2K\right)}, \horizon\right\}\\
        &\cdot\frac{\normFun[2]{f_1-f_2}{L^{\infty}(\compactInit)}}{2\sqrt{\odedim}(2+L\horizon)},   
    \end{aligned}
    \end{equation*}
    which implies that 
    \begin{equation}\label{eq:lm_f1f2_proof}
         \normFun[2]{f_1-f_2}{L^{\infty}(\compactInit)}\leq \phi\left(\hd[\initset][\normFun{\cdot}{L^{\infty}}]{\diffeq[f_1]}{\diffeq[f_2]}\right),
    \end{equation}
    with 
    \begin{equation}\label{eq:phi_form}
        \phi(\delta)\defeq 2\sqrt{\odedim}(2+L\horizon)\max\left\{\frac{\delta}{\horizon},\sqrt{\left(2\sqrt{\odedim}L\left(2\sqrt{\odedim}D L\horizon e^{L\horizon}  + KL\horizon^2 e^{L\horizon} + 2K\right)\right)\delta}\right\}.
    \end{equation}
    \item[(ii).] Determine the sampling error $\error[sampling]$. For arbitrarily $x\in \behavior[\compactInit]{\diffeq[f]}$, we can find $\vecindex\in \{-N_x,-N_x+1,\dots, N_x\}^\odedim $, such that $x_{\vecindex}\in X$, $\norm{x_{\vecindex}(0)-x(0)}_2\leq \frac{\sqrt{\odedim}D}{N_x}$. Since $x_{\vecindex}-x$ satisfy 
    \begin{equation*}
        \begin{aligned}
            \norm{x_{\vecindex}(t)-x(t)}_2 &= \norm{x_{\vecindex}(0)-x(0) + \int_{0}^t\left(f(x(s))-f(y(s))ds\right)}_2\\
            &\leq \norm{x_{\vecindex}(0)-x(0)}_2 + \int_{0}^t\norm{f(x(s))-f(y(s))}_2ds\\
            &\leq \norm{x_{\vecindex}(0)-x(0)}_2 + L\int_{0}^t\norm{x(s)-y(s)}_2ds,
        \end{aligned}
    \end{equation*}
    by Grönwall's inequality, we have 
    \begin{equation*}
        \norm{x_{\vecindex}(t)-x(t)}_2\leq \norm{x_{\vecindex}(0)-x(0)}_2  e^{Lt}\leq \frac{\sqrt{\odedim}De^{Lt}}{N_x}.
    \end{equation*}
    Taking supreme over $t$, we have 
    \begin{equation}
        \inf_{y\in X}\normFun[2]{y-x}{L^{\infty}([0,T])}\leq \normFun[2]{x_{\vecindex}-x}{L^{\infty}([0,T])}\leq \frac{\sqrt{\odedim}De^{L\horizon}}{N_x}.
    \end{equation}
    Noting that the above inequality holds for arbitrary $x\in \behavior[\compactInit]{\diffeq[f]}$, we obtain
    \begin{equation}\label{eq:Bf_X}
        \sup_{x\in \behavior[\compactInit]{\diffeq[f]}}\inf_{y\in X}\normFun[2]{y-x}{L^{\infty}([0,T])}\leq \frac{\sqrt{\odedim}De^{L\horizon}}{N_x}.
    \end{equation}
    Conversely, for arbitrary $x_\vecindex \in X$, thanks to the fact that $x_\vecindex \in \behavior[\compactInit]{\diffeq[f]}$, we obtain 
    \begin{equation}\label{eq:X_Bf}
        \sup_{y\in X}\inf_{x\in \behavior[\compactInit]{\diffeq[f]}}\normFun[2]{y-x}{L^{\infty}([0,T])}=0.
    \end{equation}
    Combing \eqref{eq:Bf_X} and \eqref{eq:X_Bf} gives 
    \begin{equation}\label{eq:eps_samp_form}
        \hd{X}{\behavior[\initset]{\diffeq[f]}}\leq \frac{\sqrt{\odedim}De^{L\horizon}}{N_x}\eqdef \error[sampling].
    \end{equation}
    \item[(iii).] Determine the numerical error $\error[num]$. For every $x_\vecindex \in X$, we have 
    \begin{equation*}
    \begin{aligned}
        \norm{x_\vecindex(t)}_2&= \norm{x_\vecindex(0) + \int_{s=0}^t f(x_\vecindex(s)ds)}_2\\       
        &\leq D+ \int_{0}^t (K+L\norm{x_\vecindex(s)}_2)ds\\
        &\leq D+K\horizon + L\int_{0}^t \norm{x_\vecindex(s)}_2ds
    \end{aligned}
    \end{equation*}
    which, by Grönwall's inequality, implies that 
    \begin{equation*}
        \norm{x_\vecindex(t)}_2\leq (D+K\horizon)e^{Lt},
    \end{equation*}
    and thus
    \begin{equation*}
        \normFun[2]{x_\vecindex}{L^{\infty}([0,\horizon])}\leq (D+K\horizon)e^{L\horizon}.
    \end{equation*}
    Then we can bound the Lipschitz constants of $x_\vecindex$ according to 
    \begin{equation}\label{eq:lip_xi}
    \begin{aligned}
        \operatorname{Lip}(x_\vecindex)&\leq \normFun[2]{\dot{x_\vecindex}}{L^{\infty}([0,\horizon])} = \normFun[2]{f(x_\vecindex)}{L^{\infty}([0,\horizon])}  \\
        &\leq K+L\normFun[2]{x_\vecindex}{L^{\infty}([0,\horizon])}\leq K+L(D+K\horizon)e^{L\horizon}.
    \end{aligned}
    \end{equation}
    This gives that for $\hat{x}_i\in \cont{\dis{X}}$, every $\tau\in \{0,1,\dots,N_t\}$, and every $$t\in \left[\frac{\horizon}{N_t}(\tau-1),\frac{\horizon}{N_t}\tau\right]\eqdef \left[t_{\tau-1},t_{\tau}\right],$$ 
    \begin{equation}\label{eq:eps_num_form}
        \begin{aligned}
            \norm{\hat{x}_\vecindex(t) - x_\vecindex(t)}_2 \overset{h\defeq\frac{\horizon}{N_t}}&{=} \norm{\frac{t-t_{\tau-1}}{h}(\hat{x}_i(t_{\tau})-x(t)) + \frac{t_{\tau}-t}{h}(\hat{x}_i(t_{\tau-1})-x(t))}_2\\
            &\leq \frac{t-t_{\tau-1}}{h} \left(\error[meas]+\operatorname{Lip}(x_\vecindex)(t_{\tau}-t)\right) + \frac{t_{\tau}-t}{h} \left(\error[meas]+\operatorname{Lip}(x_\vecindex)(t-t_{\tau-1})\right)\\
            \overset{\eqref{eq:lip_xi}}&{\leq} \error[meas] +\left( K+L(D+K\horizon)e^{L\horizon}\right)\frac{(t_{\tau}-t)(t-t_{\tau-1})}{h}\\
            &\leq \error[meas] +\left( K+L(D+K\horizon)e^{L\horizon}\right)\frac{h}{4}\\
            & = \error[meas] + \frac{\left( K+L(D+K\horizon)e^{L\horizon}\right)\horizon}{4N_t} \eqdef \error[num].
        \end{aligned}
    \end{equation}
    \end{enumerate}
    Finally, combining \eqref{eq:lm_f1f2_proof},\eqref{eq:phi_form},\eqref{eq:eps_samp_form},\eqref{eq:eps_num_form} and applying Lemma \ref{lm:general_train_hd} conclude the proof.
\end{proof}

\begin{remark}
    According to Lemma \ref{lm:sample_learn_lip}, when measurement error and training error are sufficiently small, in order to achieve the learning error $\bigo{\epsilon}$, we need at least $N=\bigo{N_x^d}=\bigo{-\epsilon^{2d}}$ number of samples and \textcolor{cc}{the sampling rate when measuring solutions should be $\bigo{N_t^{-1}}=\bigo{\epsilon^{2}}$.}
\end{remark}
Since linear ODEs are a special case of Lipschitz ODEs, we immediately obtain the following sampling results for linear ODEs by adjusting the proof of Lemma \ref{lm:sample_learn_lip}.
\begin{lemma}[Sampling results for linear ODEs]\label{lm:sample_learn_lin} 
Arbitrarily fix $\odedim, N_{x},N_{t}\in \Nplus$, $D>0$. Consider $\initset \defeq \{g_1,g_2,\dots,g_\odedim\}\subset\R^\odedim$ such that $\operatorname{span}\{g_1,g_2,\dots,g_\odedim\}=\R^\odedim$, $\{\tilde{g}_1,\tilde{g}_2,\dots,\tilde{g}_\odedim\}$ is its corresponding canonical dual frame and there exists $0<F_1\leq F_2$,
    \begin{equation*}
        \begin{aligned}
        &F_1\norm{x}_2^2\leq\sum_{k=1}^{N}\left|\langle x,g_k\rangle\right|^2\leq F_2\norm{x}_2^2,\\
        &\frac{1}{F_2}\norm{x}_2^2\leq\sum_{k=1}^{N}\left|\langle x,\tilde{g}_k\rangle\right|^2\leq\frac{1}{F_1}\norm{x}_2^2.
        \end{aligned}
    \end{equation*}
Further assume $\sup_{j\in \{1,2,\dots,\odedim\}}\norm{g_j}_2\leq D$. \\
Now, given a target ODE $\diffeq[A]\in \odeClass{\lin[\R^{ \odedim},\R^\odedim ]{K}}$ with corresponding $A\in \R^{ \odedim\times \odedim} $, assume we can generate solutions samples from $\behavior[\initset]{\diffeq[A]}$ according to 
\begin{equation}
    X = \left\{x_{j}\in \behavior[\initset]{\diffeq[A]}\mid x_{j}(0) = g_j, \text{ where }j \in \{1,2,\dots,\odedim\}\right\},
\end{equation}
and for each of the solution $x_{j} \in X$, we can measure the solution uniformly in time according to 
\begin{equation}
    \dis{X}\defeq \left\{(t_{\tau},\hat{x}_{j,\tau})\mid t_{\tau} =  \frac{\horizon}{N_t}\tau, \hat{x}_{j,\tau} = x_{j}\left(t_{\tau}\right) + w_{j,\tau}, x_{j}\in X,\tau\in\{0,1,\dots,N_t\}\right\},
\end{equation}
where $w_{j,\tau}$ is the measurement error admitting uniform bound $\norm{w_{j,\tau}}_2\leq \error[meas]$. Then, we construct 
\begin{equation}
\begin{aligned}
    \cont{\dis{X}}\defeq \left\{ \hat{x}_{j}:[0,\horizon]\rightarrow \R^{\odedim}\mid \right.&\hat{x}_{j} \text{ is a piecewise linear interpolation such that }\\
    &\left.\hat{x}_{j}\left(\frac{T}{N_t}\tau\right)=\hat{x}_{j,\tau}\text{ for } x_j\in X,\tau\in\{0,1,\dots,N_t\}\right\}
\end{aligned}
\end{equation}
as the (training) solution data. Now, suppose we can find matrix $\tilde{A}\in \R^{\odedim\times\odedim}$ such that $\norm{\tilde{A}}_2\leq K$ and 
    \begin{equation}
        \hdset[\normFun{\cdot}{L^{\infty}}]{\behavior[\initset]{\diffeq[\tilde{A}]}}{\cont{\dis{X}}}\leq \error[train].
    \end{equation}
Then, we have the learning bound
\begin{equation}
    \norm{A-\tilde{A}}_2\leq \frac{\sqrt{2\odedim}}{\left(2-e^{\frac{1}{2}}\right)\sqrt{F_1}\min\left\{\horizon,\frac{1}{2K}\right\}} \left(\frac{e^{-5\min\left\{\horizon,\frac{1}{2K}\right\}K}}{1+e^{-5\min\left\{\horizon,\frac{1}{2K}\right\}K}}\right)^{-\frac{1}{2}}\error[total],
\end{equation}
where 
\begin{equation}
    \error[total]=\error[meas]+\error[train]+\frac{KDe^{K\horizon}\horizon}{4N_t}.
\end{equation}
\end{lemma}
\begin{proof}
Note that the function $f(x)=Ax$ is an Lipschitz function with $\operatorname{Lip}(f)\leq K$ and $f(0)=0$, i.e., $f\in \lip[\R^{\odedim},\R^\odedim ]{K,0}$. We thus follow the proof of Lemma \ref{lm:sample_learn_lip} to determine the numerical error $$\error[num]=\error[meas]+\frac{KDe^{K\horizon}\horizon}{4N_t}.$$
Since the sampling procedure is precise, i.e., $X = \behavior[\initset]{\diffeq[A]}$, except for the measurement error already considered in the numerical error, we have $\error[sampling]=0$. Finally, we apply Lemma \ref{lm:lower_lin_eq_local} (replacing the the function class $\mathcal{F}$ with the class of matrices does not affect the result) to determine 
$$\phi(a) = \frac{\sqrt{2\odedim}}{\left(2-e^{\frac{1}{2}}\right)\sqrt{F_1}\min\left\{\horizon,\frac{1}{2K}\right\}} \left(\frac{e^{-5\min\left\{\horizon,\frac{1}{2K}\right\}K}}{1+e^{-5\min\left\{\horizon,\frac{1}{2K}\right\}K}}\right)^{-\frac{1}{2}} a.$$
\end{proof}
\begin{remark}
    According to Lemma \ref{lm:sample_learn_lin}, when measurement error and training error are sufficiently small, in order to achieve the learning error $\bigo{\epsilon}$, we need at least $d$ number of samples and \textcolor{cc}{the sampling rate when measuring solutions should be $\bigo{N_t^{-1}}=\bigo{\epsilon}$.}
\end{remark}

\clearpage

\begin{appendices}
\section{Metric entropy}
\begin{definition}[Covering number\cite{tikhomirovEEntropyECapacitySets1993}]\label{def:both_covering_numbers}
    Let $(\setX,\genMetric)$ be a metric space and $\entSet \subset \setX$ compact. The set $\{x_1,x_2,\dots,x_N\}\subset \entSet$ \emph{({\strut}respectively $\{x_1,x_2,\dots,x_N\}\subset \setX$){\strut}} is an $\epsilon$-covering \emph{(respectively $\epsilon$-net)} for $(\entSet, \genMetric)$ if, for each $x\in \entSet$, there exists an $i \in \{1,2,\dots, N\}$ so that $\genMetric(x,x_i)\leq \epsilon$.
    The $\epsilon$-covering number $\covering (\epsilon;\entSet,\genMetric)$ \emph{({\strut}respectively the exterior $\epsilon$-covering number $\covering^{\text{ext}} (\epsilon;\entSet,\genMetric)$){\strut}} is the cardinality of a smallest $\epsilon$-covering \emph{(respectively smallest $\epsilon$-net)} for $(\entSet, \genMetric)$. 
\end{definition}

\begin{definition}[Packing number\cite{tikhomirovEEntropyECapacitySets1993}]
\label{def:packing}
Let $(\setX,\genMetric)$ be a metric space and $\entSet \subset \setX$ compact. An $\epsilon$-packing for $(\entSet, \genMetric)$ is a set $\{x_1,x_2,\dots,x_N\}\subset \entSet$ such that $\genMetric(x_i,x_j) > \epsilon$, for all distinct $i,j$. The $\epsilon$-packing number $\packing (\epsilon;\entSet,\genMetric)$ is the cardinality of a largest $\epsilon$-packing for $(\entSet, \genMetric)$.
\end{definition}

\begin{lemma}
[\hspace{1sp}\cite{tikhomirovEEntropyECapacitySets1993}, Theorem IV] 
\label{lm:number_relation}
Let $(\setX, \genMetric)$ be a metric space and $\entSet \subset \setX$ compact. For all $\epsilon >0$, we have
\begin{equation}\label{eq:number_relation}
    \packing (2\epsilon;\entSet,\genMetric) \leq \covering^{\text{ext}} (\epsilon;\entSet,\genMetric) \leq \covering (\epsilon;\entSet,\genMetric) \leq \packing (\epsilon;\entSet,\genMetric).
\end{equation}
\end{lemma}

\begin{lemma}[\hspace{1sp}\cite{tikhomirovEEntropyECapacitySets1993}, p. 93]\label{lm:ent_isomorphism} 
Let $(\setX, \genMetric_{\setX})$ and $(\mathcal{Y}, \genMetric_{\mathcal{Y}})$ be metric spaces and consider the compact sets $\entSet_{\setX}\subset \setX$ and $\entSet_{\mathcal{Y}}\subset \mathcal{Y}$. Assume that there exists an isometric isomorphism $f : \entSet_{\setX} \rightarrow \entSet_{\mathcal{Y}}$, i.e., $f$ is bijective and for every pair $a,b\in \entSet_{\setX}$, one has $\genMetric_{\mathcal{Y}}(f(a),f(b))=\genMetric_{\setX}(a,b)$. Then, 
\begin{align}
    \covering(\epsilon;\entSet_{\setX},\genMetric_{\setX}) =\covering(\epsilon;\entSet_{\mathcal{Y}},\genMetric_{\mathcal{Y}}) \quad \textrm{and} \quad
    \packing(\epsilon;\entSet_{\setX},\genMetric_{\setX}) =\packing(\epsilon;\entSet_{\mathcal{Y}},\genMetric_{\mathcal{Y}}).
\end{align}
\end{lemma}

\begin{lemma}\label{lm:isometry_packing_relation}
    Let $(\setX, \genMetric_{\setX})$ and $(\mathcal{Y}, \genMetric_{\mathcal{Y}})$ be metric spaces and consider the compact set $\entSet_{\setX}\subset \setX$ and $\entSet_{\mathcal{Y}}\subset \mathcal{Y}$. Assume that there exists an isometry $f : \entSet_{\setX} \rightarrow \entSet_{\mathcal{Y}}$, i.e., for every pair $a,b\in \entSet_{\setX}$, one has $\genMetric_{\mathcal{Y}}(f(a),f(b))=\genMetric_{\setX}(a,b)$. Then, 
    \begin{equation}
    \packing(\epsilon;\entSet_{\setX},\genMetric_{\setX}) \leq \packing(\epsilon;\entSet_{\mathcal{Y}},\genMetric_{\mathcal{Y}}).
    \end{equation}
\end{lemma}

An immediate result from Lemma \ref{lm:isometry_packing_relation} is the following.
\begin{corollary}\label{col:subset_packing_relation}
    Let $(\setX, \genMetric_{\setX})$ be a metric space and consider the compact sets $\entSet_{1}\subset\entSet_{2} \subset\setX$. Then, 
    \begin{equation}
    \packing(\epsilon;\entSet_{1},\genMetric_{\setX}) \leq \packing(\epsilon;\entSet_{2},\genMetric_{\setX}).
    \end{equation}
\end{corollary}
\begin{proof}
    This follows from setting $(\setX, \genMetric_{\setX})=(\mathcal{Y}, \genMetric_{\mathcal{Y}})$, $\entSet_{\setX}=\entSet_{1}$, $\entSet_{\mathcal{Y}}= \entSet_{2}$, $f$ to be an identity mapping, and finally applying Lemma \ref{lm:isometry_packing_relation}.
\end{proof}

\begin{lemma}\label{lm:covering_relation_equiv_norm}
Let $(\setX, \genMetric_{\setX})$ and $(\mathcal{Y}, \genMetric_{\mathcal{Y}})$ be metric spaces and consider the compact sets $\entSet_{\setX}\subset \setX$ and $\entSet_{\mathcal{Y}}\subset \mathcal{Y}$. Assume that there exist a surjective $f : \entSet_{\setX} \rightarrow \entSet_{\mathcal{Y}}$ and constants $C_1\geq C_2>0$ such that for all $a,b\in \entSet_{\setX}$,
\begin{equation}
    C_2 \genMetric_{\setX}(a,b)\leq \genMetric_{\mathcal{Y}}(f(a),f(b))\leq C_1 \genMetric_{\setX}(a,b),
\end{equation}
then 
\begin{equation}
\covering(\epsilon/C_2;\entSet_{\setX},\genMetric_{\setX}) \leq \covering(\epsilon;\entSet_{\mathcal{Y}},\genMetric_{\mathcal{Y}}) \leq \covering(\epsilon/C_1;\entSet_{\setX},\genMetric_{\setX}) .
\end{equation}
\end{lemma}
An immediate result from Lemma \ref{lm:covering_relation_equiv_norm} is the following.
\begin{corollary}\label{col:covering_relation_equiv_norm_same_set}
    Let $(\setX, \genMetric_{1})$ be a metric space and consider a compact set $\entSet_{\setX}\subset \setX$. Assume that for another metric $\genMetric_{2}$ on $\setX$, there exist constants $C_1\geq C_2>0$ such that for all $a,b\in \entSet_{\setX}$,
\begin{equation}
    C_2 \genMetric_{1}(a,b)\leq \genMetric_{2}(a,b)\leq C_1 \genMetric_{1}(a,b),
\end{equation}
then $\entSet_{\setX}$ is also compact under $\genMetric_{2}$. Moreover, we have
\begin{equation}
\covering(\epsilon/C_2;\entSet_{\setX},\genMetric_{1}) \leq \covering(\epsilon;\entSet_{\setX},\genMetric_{2}) \leq \covering(\epsilon/C_1;\entSet_{\setX},\genMetric_{2}) .
\end{equation}
\end{corollary}
\begin{proof}
    This follows by setting $(\setX, \genMetric_{\setX}) =  (\setX, \genMetric_{1})$, $(\mathcal{Y}, \genMetric_{\mathcal{Y}}) =  (\setX, \genMetric_{2})$, $\entSet_{\mathcal{Y}}=\entSet_{\setX}$, $f$ to be an identity mapping, and finally applying Lemma \ref{lm:covering_relation_equiv_norm}
\end{proof}

\begin{lemma}\label{lm:ent_ball}
    Let $\odedim\in \Nplus$ and $C>0$. Consider a ball in $\R^\odedim$ 
    $$\ball{C}^d \defeq  \{x\in \R^{\odedim}\mid \norm{x}_2\leq C\}.$$
    Then, we have 
    \begin{equation}
        \log\covering(\epsilon;\ball{C}^d, \norm{\cdot}_2)\asymp d\log\left(\epsilon^{-1}\right).
    \end{equation}
\end{lemma}
\begin{proof}
    It follows by adjusting the proof of 
    \cite[Lemma 5.7, Example 5.8]{wainwright2019high} from unit balls to balls with radius $C$.
\end{proof}

\section{Matrix analysis}

\begin{lemma}\label{lm:mat_F_2}
    Let $A\in \R^{n\times n}$. We have 
    \begin{equation}
        \norm{A}_2\leq \norm{A}_F\leq \sqrt{n}\norm{A}_2.
    \end{equation}
\end{lemma}

\begin{lemma}\label{lm:con_num}
    Let $A\in \R^{n\times n}$. We have 
    \begin{equation}
        \conNum{A}{\norm{\cdot}_2} = \frac{\sig_1(A)}{\sig_n(A)}.
    \end{equation}
    Moreover, if $A$ is a normal matrix, then
    \begin{equation}
        \conNum{A}{\norm{\cdot}_2} = \frac{|\eig_1(A)|}{|\eig_n(A)|}.
    \end{equation}
\end{lemma}

\begin{lemma}\label{lm:spec_radius_lower}
    Let $A\in \R^{n\times n}$. For all consistent matrix norms $\norm{\cdot}$, we have 
    \begin{equation}
        \norm{A}\geq \specRad{A}.
    \end{equation}
    Moreover, if $A$ is a symmetric matrix, then 
    \begin{equation}
        \norm{A}_2= \specRad{A}.
    \end{equation}    
\end{lemma}

\begin{lemma}[Weyl's inequality]\label{lm:weyl_eig}
    Let $A\in \R^{n\times n}$, then we have 
    \begin{equation}
        |\eig_1(A)\eig_2(A) \dots\eig_k(A)|\leq \sig_1(A)\sig_2(A)\dots \sig_k(A) 
    \end{equation}
    for all $k=1,2,\dots,n$.
\end{lemma}

\begin{lemma}\label{lm:exp_norm}
    For all consistent matrix norm $\norm{\cdot}$ and all $A\in \R^{n\times n}$, we have 
    \begin{equation}
        \norm{e^A}\leq e^{\norm{A}}.
    \end{equation}
\end{lemma}

\begin{lemma}\label{lm:power_diff_norm_upper}
    For all consistent matrix norm $\norm{\cdot}$ and all $A,B\in \R^{n\times n}$, we have
    \begin{equation}
        \norm{B^k-A^k} \leq k\norm{B-A}\max\{\norm{A},\norm{B}\}^{k-1}.
    \end{equation}
\end{lemma}
\begin{proof}
Note that 
    \begin{equation*}
        \begin{aligned}
            \norm{B^k-A^k}&= \norm{\sum_{l=0}^{k-1}B^l(B-A)A^{k-1-l}} \leq \sum_{l=0}^{k-1} \norm{B^l(B-A)A^{k-1-l}}\\
            &\leq \sum_{l=0}^{k-1} \norm{B}^l\norm{B-A}\norm{B}^{k-1-l}\leq k\norm{B-A}\max\{\norm{A},\norm{B}\}^{k-1}.
        \end{aligned}
    \end{equation*}    
\end{proof}

\begin{lemma}\label{lm:exp_diff_norm_upper}
    For all consistent matrix norm $\norm{\cdot}$ and all $A,B\in \R^{n\times n}$, we have
    \begin{equation}
        \norm{e^A-e^B} \leq \norm{A-B}e^{\max\{\norm{A},\norm{B}\}}.
    \end{equation}
\end{lemma}
\begin{proof}
    Applying Lemma \ref{lm:power_diff_norm_upper}, we obtain
    \begin{equation*}
    \begin{aligned}
        \norm{e^A-e^B} &\leq \sum_{k=0}^{\infty} \frac{1}{k!}\norm{B^k-A^k}\leq \sum_{k=0}^{\infty} k\frac{1}{k!}\norm{B-A}\max\{\norm{A},\norm{B}\}^{k-1} \\
        &= \norm{A-B}e^{\max\{\norm{A},\norm{B}\}} .      
    \end{aligned}
    \end{equation*}
\end{proof}

\begin{lemma}\label{lm:bound_sigN_EATEA}
    Let $A\in \R^{n\times n}$. Assume $\norm{\cdot}$ is a consistent matrix norm and $\max\{\norm{A},\norm{A^T}\}\leq K$. Then we have $e^{A^T}e^A$ is a symmetric positive definite matrix with 
    \begin{equation}
        \eig_n(e^{A^T}e^A) \geq e^{-5K}.
    \end{equation}
\end{lemma}
\begin{proof}
    Since $e^{A^T}e^A$ is a symmetric positive definite matrix, we have $\eig_i(e^{A^T}e^A)\in \R$, $\eig_i(e^{A^T}e^A)> 0$ for all $i\in \{1,2,\dots,n\}$, and 
    \begin{equation}\label{eq:bound_con_EATEA}
    \begin{aligned}
        \frac{\eig_1(e^{A^T}e^A)}{\eig_n(e^{A^T}e^A)} &\overset{\text{Lemma \ref{lm:con_num}}}{=} \conNum{e^{A^T}e^A}{\norm{\cdot}_2}   \\
        & = \norm{e^{A^T}e^A}_2\norm{e^{-A}e^{-A^T}}_2\\
        &\leq \norm{e^{A^T}}_2\norm{e^{A}}_2\norm{e^{-A}}_2\norm{e^{-A^T}}_2\\
        &\overset{\text{Lemma \ref{lm:spec_radius_lower}}}{\leq} \norm{e^{A^T}}\norm{e^{A}}\norm{e^{-A}}\norm{e^{-A^T}}\\
        & \overset{\text{Lemma \ref{lm:exp_norm}}}{\leq} e^{\norm{A^T}}e^{\norm{A}}e^{\norm{-A}}e^{\norm{-A^T}}\\
        &\leq e^{4K}.
    \end{aligned}
    \end{equation}
    Note that  
    \begin{equation*}
        |\re{\eig_i(A)}|\leq |\eig_i(A)| \leq \rho(A)\leq \norm{A}\leq K\quad \text{for all }i\in \{1,2,\dots,n\}.
    \end{equation*}
    Since $\eig_1(e^A) = e^{\eig_i(A)}$ for some $i\in \{1,2,\dots,n\}$, we have 
    \begin{equation*}
        |\eig_1(e^A)| = |e^{\eig_i(A)}| =e^{\re{\eig_i(A)}} \geq e^{-K}.
    \end{equation*}
    Applying Lemma \ref{lm:weyl_eig}, we obtain
    \begin{equation}\label{eq:bound_sig1_EATEA}
        \eig_1(e^{A^T}e^A) = \sig_1(e^A)\geq |\eig_1(e^A)|\geq e^{-K}.
    \end{equation}
    Finally, combining \eqref{eq:bound_con_EATEA} and \eqref{eq:bound_sig1_EATEA} gives 
    \begin{equation}
        \eig_n(e^{A^T}e^A)  = \eig_1(e^{A^T}e^A) \left(\frac{\eig_1(e^{A^T}e^A)}{\eig_n(e^{A^T}e^A)}\right)^{-1}\geq e^{-5K}.
    \end{equation}
\end{proof}

\begin{lemma}\label{lm:frame_lin}
    Suppose $\{g_1,g_2,\dots,g_N\}$ is a frame for $\R^n$ and $\{\tilde{g}_1,\tilde{g}_2,\dots,\tilde{g}_N\}$ is its corresponding canonical dual frame, such that there exists $0<F_1\leq F_2$,
    \begin{equation*}
        \begin{aligned}
        &F_1\norm{x}_2^2\leq\sum_{k=1}^{N}\left|\langle x,g_k\rangle\right|^2\leq F_2\norm{x}_2^2,\\
        &\frac{1}{F_2}\norm{x}_2^2\leq\sum_{k=1}^{N}\left|\langle x,\tilde{g}_k\rangle\right|^2\leq\frac{1}{F_1}\norm{x}_2^2.
        \end{aligned}
    \end{equation*}
    Then, for every matrix $A\in \R^{n\times n}$, we have 
    \begin{equation}
        \sup_i \norm{Ag_i}_2 \geq \sqrt{\frac{F_1}{N}} \norm{A}_2.
    \end{equation}
\end{lemma}
\begin{proof}
    For every $x\in \R^n$, we have the following frame decomposition 
    \begin{equation}
        x = \sum_{k=1}^N \inp{x}{\tilde{g}_k}g_k,
    \end{equation}
    and thus we have 
    \begin{equation}
        \begin{aligned}
            \norm{Ax}_2&= \norm{A\sum_{k=1}^N \inp{x}{\tilde{g}_k}g_k}_2\\
            &\leq \sup_i\norm{Ag_i}_2 \sum_{k=1}^N |\inp{x}{\tilde{g}_k}|\\
            &\leq \sup_i\norm{Ag_i}_2 \sqrt{N} \sqrt{\sum_{k=1}^N |\inp{x}{\tilde{g}_k}|^2}\\
            & \leq \sqrt{\frac{N}{F_1}} \norm{x}_2 \sup_i\norm{Ag_i}_2.
        \end{aligned}
    \end{equation}
    This implies that 
    \begin{equation}
        \sup_i\norm{Ag_i}_2 \geq \sqrt{\frac{F_1}{N}} \frac{\norm{Ax}_2}{\norm{x}_2},\quad \text{for all }x\in \R^n\setminus\{0\}
    \end{equation}
    and thus 
    \begin{equation}
        \sup_i\norm{Ag_i}_2 \geq \sqrt{\frac{F_1}{N}} \norm{A}_2
    \end{equation}
\end{proof}

\section{Ordinary differential equations}
\begin{lemma}[Gronwall–Bellman–Pachpatte inequality \cite{pachpatte2006integral}]\label{lm:gronwall_pach}
    Let $u(t),a(t),a'(t)\in C(\R_+,\R_+)$,$k(t,\sigma),\frac{\partial }{\partial t}k(t,\sigma) \in C(D,\R_+)$ where $D = \{(t,\sigma)\in \R_+^2:0\leq \sigma\leq t<\infty\}$. Let $g\in C(\R_+,\R_+)$ be a non-decreasing function, $g(u)>0$ on $(0,\infty)$. If 
    \begin{equation}
        u(t)\leq a(t)+\int_{0}^{t}k(t,\sigma)g(u(\sigma))d\sigma,
    \end{equation}
    for $t\in \R_+$, then for $0\leq t\leq t_1$,
    \begin{equation}
        u(t)\leq a(t)+\int_{0}^{t}k(t,\sigma)g(u(\sigma))d\sigma \leq G^{-1}\left[G(a(t))+\int_{0}^{t}A(s)ds\right],
    \end{equation}
    where 
    \begin{equation*}
        \begin{aligned}
            A(t) &=k(t,t) + \int_{0}^{t}\frac{\partial }{\partial t}k(t,\sigma) d\sigma,\\
            G(r)&=\int_{r_0}^{r}\frac{ds}{g(s)}, \quad r>0,
        \end{aligned}
    \end{equation*}
    $r_0$ is arbitrary and $G^{-1}$ is the inverse of $G$ and $t_1\in \R_+$ is chosen so that 
    \begin{equation*}
        G(a(t))+\int_{0}^{t}A(s)ds\in \operatorname{Dom}(G^{-1}),
    \end{equation*}
    for all $0\leq t\leq t_1$.
\end{lemma}

\section{Proofs}
\subsection{Proof of Lemma \ref{lm:lower_lin_eq_local}}\label{app:proof_lin_local}
    For arbitrarily fixed $t_0\in (0,\horizon]$ and $x_0,\tilde{x}_0\in B_D$, set $B\defeq At_0$, $\tilde{B}\defeq\tilde{A}t_0$, $y\defeq x_0-\tilde{x}_0$ and $b\defeq (e^{B}-e^{\tilde{B}})\tilde{x}_0$. Then, $\norm{B}_2=\norm{B^T}_2\leq t_0K$. Note that for arbitrarily fixed $\tilde{x}_0\in G$, we have  
    \begin{equation}\label{eq:lower_lti_1_ext}
        \begin{aligned}
        &\inf_{x_0\in G }\normFun[2]{e^{A\cdot}x_0-e^{\tilde{A}\cdot}\tilde{x}_0}{L^\infty}\geq \inf_{x_0\in G} \max\left\{\norm{x_0-\tilde{x}_0}_2, \norm{e^{At_0}x_0-e^{\tilde{A}t_0}\tilde{x}_0}_2\right\}\\
        &\geq \frac{1}{\sqrt{2}}\inf_{x_0\in \R^{\odedim}}\left(\norm{x_0-\tilde{x}_0}_2^2+\norm{e^{B}x_0-e^{\tilde{B}}\tilde{x}_0}_2^2\right)^{\frac{1}{2}}\\
        &= \frac{1}{\sqrt{2}}\inf_{y\in \R^{\odedim}}\left(\norm{y}_2^2+\norm{e^{B}y+b}_2^2\right)^{\frac{1}{2}}\\
        &= \frac{1}{\sqrt{2}} \inf_{y\in \R^{\odedim}}\left(y^T(I+e^{B^T}e^{B})y+2b^Te^B y + b^Tb\right)^{\frac{1}{2}}\\
        &= \frac{1}{\sqrt{2}}\left( b^T \left(I-e^B(I+e^{B^T}e^{B})^{-1}e^{B^T}\right) b\right)^{\frac{1}{2}},
        \end{aligned}
    \end{equation}
    where the last equality holds since we are minimizing a quadratic form and $I+e^{B^T}e^{B}$ is a symmetric positive definite matrix. Note that $$I-e^B(I+e^{B^T}e^{B})^{-1}e^{B^T} = I - \left(I+e^{-B^T}e^{-B}\right)^{-1}$$ is a symmetric matrix and thus only has real eigenvalues. Moreover, $e^{-B^T}e^{-B}$ is a symmetric positive definite matrix and by Lemma
    \ref{lm:bound_sigN_EATEA} (replacing $A$ with $-B$), 
    \begin{equation}\label{eq:lower_sigN_EBTEB_ext}
        \eig_n(e^{-B^T}e^{-B}) \geq e^{-5t_0K}.
    \end{equation}
    Now, note that 
    \begin{equation*}
        \eig_n\left(I -\left(I+e^{-B^T}e^{-B}\right)^{-1}\right) = 1-\left(1+\eig_i\left(e^{-B^T}e^{-B}\right)\right)^{-1}\quad \text{for some }i\in \{1,2,\dots,n\},
    \end{equation*}
    which means
    \begin{equation}\label{eq:lower_lti_2_ext}
    \begin{aligned}
        \eig_n\left(I -\left(I+e^{-B^T}e^{-B}\right)^{-1}\right)&\geq 1-\left(1+\eig_n\left(e^{-B^T}e^{-B}\right)\right)^{-1}\\
        &\overset{\eqref{eq:lower_sigN_EBTEB_ext}}{\geq} 1-\left(1+e^{-5t_0K}\right)^{-1} \\
        & = \frac{e^{-5t_0K}}{1+e^{-5t_0K}}.
    \end{aligned}
    \end{equation}
    Thus, we have 
    \begin{equation}\label{eq:lower_lti_3_ext}
        \begin{aligned}
        \sv[G][\normFun{\cdot}{L^{\infty}}]{\diffeq[A]}{\diffeq[\tilde{A}]} &= \sup_{\tilde{x}\in \behavior[G]{\diffeq[\tilde{A}]}}\inf_{x\in \behavior[G]{\diffeq[A]}}\normFun[2]{x-\tilde{x}}{L^{\infty}([0,\horizon])}\\
        &= \sup_{\tilde{x}_0 \in G}\inf_{x_0 \in G}\normFun[2]{e^{A\cdot}x_0-e^{\tilde{A}\cdot}\tilde{x}_0}{L^{\infty}([0,\horizon])}\\
         \overset{\eqref{eq:lower_lti_1_ext}}&{\geq} \frac{1}{\sqrt{2}}\sup_{\tilde{x}_0 \in G}\left( b^T \left(I-e^B(I+e^{B^T}e^{B})^{-1}e^{B^T}\right) b\right)^{\frac{1}{2}}\\
        \overset{\eqref{eq:lower_lti_2_ext}}&{\geq} \frac{1}{\sqrt{2}} \left(\frac{e^{-5t_0K}}{1+e^{-5t_0K}}\right)^{\frac{1}{2}}\sup_{i}\norm{(e^{At_0}-e^{\tilde{A}t_0})g_i}_2\\
        \overset{\textbf{Lemma \ref{lm:frame_lin}}}&{\geq}\sqrt{\frac{F_1}{2N}} \left(\frac{e^{-5t_0K}}{1+e^{-5t_0K}}\right)^{\frac{1}{2}}\norm{e^{At_0}-e^{\tilde{A}t_0}}_2.\\
        \end{aligned}
    \end{equation}
    Next, we set out to find a $t_0$ such that $\norm{e^{At_0}-e^{\tilde{A}t_0}}_2$ can be lower-bounded by a scaling of $\norm{A-\tilde{A}}_2$. Note that 
    \begin{equation}
    \begin{aligned}
        \norm{e^{At_0}-e^{\tilde{A}t_0}}_2 &= \norm{\sum_{k=0}^{\infty} \frac{1}{k!}\left((At_0)^k -(\tilde{A}t_0)^k \right)}_2   \\
        &\geq \norm{A-\tilde{A}}_2t_0 - \sum_{k=2}^{\infty} \norm{A^k-\tilde{A}^k}_2 t_0^k\\
        \overset{\text{Lemma \ref{lm:power_diff_norm_upper}}}&{\geq} \norm{A-\tilde{A}}_2t_0 - \sum_{k=2}^{\infty} \frac{k}{k!}\norm{A-\tilde{A}}_2\max\left\{\norm{A}_2,\norm{\tilde{A}}_2\right\}^{k-1}t_0^k\\
        &\geq \norm{A-\tilde{A}}_2 t_0 - \sum_{k=2}^{\infty} \frac{k}{k!}\norm{A-\tilde{A}}_2 K^{k-1}t_0^k\\
        & = \norm{A-\tilde{A}}_2 t_0 \left(2-e^{Kt_0}\right)
    \end{aligned}
    \end{equation}
    If we choose $t_0 = \min\left\{\horizon,\frac{1}{2K}\right\}$, then 
    \begin{equation}\label{eq:lower_lti_4_ext}
        \norm{e^{At_0}-e^{\tilde{A}t_0}}_2 \geq \min\left\{\horizon,\frac{1}{2K}\right\} \left(2-e^{\frac{1}{2}}\right)\norm{A-\tilde{A}}_2.
    \end{equation}
    Combining \eqref{eq:lower_lti_3_ext} and \eqref{eq:lower_lti_4_ext}, we have 
    \begin{equation*}
        \sv[G][\normFun{\cdot}{L^{\infty}}]{\diffeq[A]}{\diffeq[\tilde{A}]} \geq \frac{\left(2-e^{\frac{1}{2}}\right)\sqrt{F_1}\min\left\{\horizon,\frac{1}{2K}\right\}}{\sqrt{2N}} \left(\frac{e^{-5\min\left\{\horizon,\frac{1}{2K}\right\}K}}{1+e^{-5\min\left\{\horizon,\frac{1}{2K}\right\}K}}\right)^{\frac{1}{2}}\norm{A-\tilde{A}}_2.
    \end{equation*}
    By symmetry, this gives the desired bound.

\subsection{Proof of Theorem \ref{th:dist_lu_linearODE_high}}\label{app:dist_lu_linearODE_high}
    We apply state transformation $y = \begin{pmatrix}
    x^T&(x^{(1)})^T&\dots&(x^{(\odeorder-1)})^T\end{pmatrix}^T$. Then, $x\in \behavior[\ball{D}]{\diffeq[A]}$ with $\diffeq[A]\in \odeClass[\odeorder ]{\lin[\R^{\odeorder \odedim},\R^\odedim ]{K}}$ of the form \eqref{eq:mth_ode} if and only if $y \in \behavior[\ball{D}]{\diffeq[F(A)]}$ with $\diffeq[F(A)]\in \odeClass{\lin[\R^{\odeorder \odedim},\R^{\odeorder\odedim} ]{K}}$ of the form \eqref{eq:1st_y_ode}, where 
    \begin{equation}
        F(A) \defeq \begin{pmatrix}
             0& \indmat{\odedim} &0 &\dots&0\\
             0& 0 &\indmat{\odedim} &\dots&0\\
             \vdots &\vdots &\vdots&&\vdots\\
             0 &0 &0&\dots&\indmat{\odedim}\\
             -A_0& -A_1 & -A_2&\dots &-A_{\odeorder-1}
             \end{pmatrix}.
    \end{equation}
    We can upper bound the $2$-norm of $F(A)$ according to
    \begin{equation}\label{eq:upper_norm_FA}
    \begin{aligned}
        \norm{F(A)}_2 &= \norm{\sum_{i=1}^{\odeorder-1}E_{i,i+1}\otimes\indmat{\odedim} - \sum_{j=0}^{\odeorder-1}E_{\odedim,j+1}\otimes A_j}_2\\
        &\leq \sum_{i=1}^{\odeorder-1}\norm{E_{i,i+1}\otimes\indmat{\odedim}}_2 + \sum_{j=0}^{\odeorder-1}\norm{E_{\odedim,j+1} \otimes A_j}_2\\
        &= \sum_{i=1}^{\odeorder-1}\norm{E_{i,i+1}}_2\norm{\indmat{\odedim}}_2 + \sum_{j=0}^{\odeorder-1}\norm{E_{\odedim,j+1}}_2 \norm{A_j}_2\\
        &\leq \odeorder(K+1).
    \end{aligned}
    \end{equation}
    Note that 
    \begin{equation}
    \begin{aligned}\label{eq:y_x_bounds}
        \normFun[2]{y}{L^{\infty}[0,\horizon]}&= \norm{\left(\sum_{i=0}^{\odeorder-1}\norm{x^{(i)}}_2^2\right)^{\frac{1}{2}}}_{L^{\infty}[0,\horizon]}\leq \norm{\sum_{i=0}^{\odeorder-1}\norm{x^{(i)}}_2}_{L^{\infty}[0,\horizon]}\\
        &\leq \sum_{i=0}^{\odeorder-1}\norm{\norm{x^{(i)}}_2}_{L^{\infty}[0,\horizon]} = \normFun[2]{x}{C^{\odeorder-1}[0,\horizon]}\\
        &\leq \odeorder\norm{\left(\sum_{i=0}^{\odeorder-1}\norm{x^{(i)}}_2^2\right)^{\frac{1}{2}}}_{L^{\infty}[0,\horizon]} = \odeorder \normFun[2]{y}{L^{\infty}[0,\horizon]},
    \end{aligned}
    \end{equation}
    which implies that 
    \begin{equation}
        \hd[\ball{D}][\normFun{\cdot}{L^{\infty}}]{\diffeq[F(A)]}{\diffeq[F(\tilde{A})]}\leq \hd[\ball{D}][\normFun{\cdot}{C^{\odeorder-1}}]{\diffeq[A]}{\diffeq[\tilde{A}]} \leq \odeorder\hd[\ball{D}][\normFun{\cdot}{L^{\infty}}]{\diffeq[F(A)]}{\diffeq[F(\tilde{A})]}.
    \end{equation}
    Thus, we can upper bound the Hausdorff distance according to 
    \begin{equation}\label{eq:upper_lin_higher}
        \begin{aligned}
        \hd[\ball{D}][\normFun{\cdot}{C^{\odeorder-1}}]{\diffeq[A]}{\diffeq[\tilde{A}]} &\leq \odeorder\hd[\ball{D}][\normFun{\cdot}{L^{\infty}}]{\diffeq[F(A)]}{\diffeq[F(\tilde{A})]} \\
        \overset{\text{Lemma \ref{lm:upper_lin_ode}},\eqref{eq:upper_norm_FA}}&{\leq} \odeorder D\horizon e^{\odeorder(K+1)\horizon}\norm{F(A)-F(\tilde{A})}_2\\
        &= \odeorder D\horizon e^{\odeorder(K+1)\horizon}\norm{\sum_{j=0}^{\odeorder-1}E_{\odedim,j+1}\otimes (A_j-\tilde{A}_j)}_2\\
        &\leq \odeorder D\horizon e^{\odeorder(K+1)\horizon}\sum_{j=0}^{\odeorder-1} \norm{E_{\odedim,j+1}}_2 \norm{(A_j-\tilde{A}_j)}_2\\
        &= \odeorder D\horizon e^{\odeorder(K+1)\horizon}\sum_{j=0}^{\odeorder-1} \norm{(A_j-\tilde{A}_j)}_2.
        \end{aligned}
    \end{equation}
Similarly, we can lower bound the Hausdorff distance according to 
\begin{equation}\label{eq:lower_lin_higher}
    \begin{aligned}
        &\hd[\ball{D}][\normFun{\cdot}{C^{\odeorder-1}}]{\diffeq[A]}{\diffeq[\tilde{A}]} \geq \hd[\ball{D}][\normFun{\cdot}{L^{\infty}}]{\diffeq[F(A)]}{\diffeq[F(\tilde{A})]} \\
        \overset{\text{Lemma \ref{lm:lower_lin_eq}},\eqref{eq:upper_norm_FA}}&{\geq} \frac{\left(2-e^{\frac{1}{2}}\right)D\min\left\{\horizon,\frac{1}{2\odeorder(K+1)}\right\}}{\sqrt{2}} \left(\frac{e^{-5\min\left\{\horizon,\frac{1}{2\odeorder(K+1)}\right\}\odeorder(K+1)}}{1+e^{-5\min\left\{\horizon,\frac{1}{2\odeorder(K+1)}\right\}\odeorder(K+1)}}\right)^{\frac{1}{2}}\norm{F(A)-F(\tilde{A})}_2\\
        \overset{\text{Lemma \ref{lm:mat_F_2}}}&{\geq}\frac{\left(2-e^{\frac{1}{2}}\right)D\min\left\{\horizon,\frac{1}{2\odeorder(K+1)}\right\}}{\sqrt{2\odeorder\odedim}} \left(\frac{e^{-5\min\left\{\horizon,\frac{1}{2\odeorder(K+1)}\right\}\odeorder(K+1)}}{1+e^{-5\min\left\{\horizon,\frac{1}{2\odeorder(K+1)}\right\}\odeorder(K+1)}}\right)^{\frac{1}{2}}\norm{F(A)-F(\tilde{A})}_F\\
        &=\frac{\left(2-e^{\frac{1}{2}}\right)D\min\left\{\horizon,\frac{1}{2\odeorder(K+1)}\right\}}{\sqrt{2\odeorder\odedim}} \left(\frac{e^{-5\min\left\{\horizon,\frac{1}{2\odeorder(K+1)}\right\}\odeorder(K+1)}}{1+e^{-5\min\left\{\horizon,\frac{1}{2\odeorder(K+1)}\right\}\odeorder(K+1)}}\right)^{\frac{1}{2}}\sum_{i=0}^{\odeorder-1}\norm{A_i-\tilde{A}_i}_F\\
        \overset{\text{Lemma \ref{lm:mat_F_2}}}&{\geq} \frac{\left(2-e^{\frac{1}{2}}\right)D\min\left\{\horizon,\frac{1}{2\odeorder(K+1)}\right\}}{\sqrt{2\odeorder\odedim}} \left(\frac{e^{-5\min\left\{\horizon,\frac{1}{2\odeorder(K+1)}\right\}\odeorder(K+1)}}{1+e^{-5\min\left\{\horizon,\frac{1}{2\odeorder(K+1)}\right\}\odeorder(K+1)}}\right)^{\frac {1}{2}}\sum_{i=0}^{\odeorder-1}\norm{A_i-\tilde{A}_i}_2.
    \end{aligned}
\end{equation}
Combining \eqref{eq:upper_lin_higher} and \eqref{eq:lower_lin_higher} concludes the proof.

\subsection{Proof of Theorem \ref{th:dist_lu_lipODE_high}}\label{app:dist_lu_lipODE_high}
We apply state transformation $y = \begin{pmatrix}
    x^T&(x^{(1)})^T&\dots&(x^{(\odeorder-1)})^T\end{pmatrix}^T$. 
Then, $x\in \behavior[\compactInit]{\diffeq[f]}$ with $\diffeq[f]$ of the form \eqref{eq:mth_ode} if and only if $y\in \behavior[\compactInit]{\diffeq[F(f)]}$ with $\diffeq[F(f)]$ of the form \eqref{eq:1st_y_ode}, where
\begin{equation}
    F(f)(y) \defeq \begin{pmatrix}
        y_2\\y_3\\\vdots\\
        y_\odeorder\\
        f(y_1,y_2,\dots,y_\odeorder)
    \end{pmatrix}.
\end{equation}
Note that 
\begin{equation}\label{eq:F_f_relation}
    \begin{aligned}
        & \operatorname{Lip}(F(f)) \leq \sqrt{\left(\operatorname{Lip}(f)\right)^2+1}\leq  \sqrt{L^2+1}   \\
        &\norm{F(f)(0)}_2 = \norm{f(0)}_2\leq K,\\
        & \normFun[2]{F(f)-F(\tilde{f})}{L^{\infty}(\compactInit)} =  \normFun[2]{f-\tilde{f}}{L^{\infty}(\compactInit)},\\
        & \normFun[2]{F(f)-F(\tilde{f})}{L^{\infty}(\widehat{\compactInit})} =  \normFun[2]{f-\tilde{f}}{L^{\infty}(\widehat{\compactInit})},
        \end{aligned}
\end{equation}
Moreover, by the same spirit of \eqref{eq:y_x_bounds}, we can prove that 
    \begin{equation}
        \hd[\ball{D}][\normFun{\cdot}{L^{\infty}}]{\diffeq[F(f)]}{\diffeq[F(\tilde{f})]}\leq \hd[\ball{D}][\normFun{\cdot}{C^{\odeorder-1}}]{\diffeq[f]}{\diffeq[\tilde{f}]} \leq \odeorder\hd[\ball{D}][\normFun{\cdot}{L^{\infty}}]{\diffeq[F(f)]}{\diffeq[F(\tilde{f})]}.
\end{equation}
    Thus, we can upper bound the Hausdorff distance according to 
    \begin{equation}\label{eq:upper_lip_higher}
        \begin{aligned}
        \hd[\ball{D}][\normFun{\cdot}{C^{\odeorder-1}}]{\diffeq[f]}{\diffeq[\tilde{f}]} &\leq \odeorder\hd[\ball{D}][\normFun{\cdot}{L^{\infty}}]{\diffeq[F(f)]}{\diffeq[F(\tilde{f})]} \\
        \overset{\text{Lemma \ref{lm:lip_ode_upper}},\eqref{eq:F_f_relation}}&{\leq} \odeorder \horizon e^{\sqrt{L^2+1}\horizon}\normFun[2]{F(f)-F(\tilde{f})}{L^{\infty}(\widehat{\compactInit})}\\
        \overset{\eqref{eq:F_f_relation}}&{=} \odeorder \horizon e^{\sqrt{L^2+1}\horizon}\normFun[2]{f-\tilde{f}}{L^{\infty}(\widehat{\compactInit})}.
        \end{aligned}
    \end{equation}
Similarly, we can lower bound the Hausdorff distance according to 
\begin{equation}\label{eq:lower_lip_higher}
    \begin{aligned}
        &\hd[\ball{D}][\normFun{\cdot}{C^{\odeorder-1}}]{\diffeq[f]}{\diffeq[\tilde{f}]} \geq \hd[\ball{D}][\normFun{\cdot}{L^{\infty}}]{\diffeq[F(f)]}{\diffeq[F(\tilde{f})]} \\
        \overset{\text{Lemma \ref{lm:dist_lower_LipODE_local}},\eqref{eq:F_f_relation}}&{\geq} \min\left\{\frac{\normFun[2]{F(f)-F(\tilde{f})}{L^{\infty}(\compactInit)}}{2\sqrt{\odedim}\hat{L}\left(2D \hat{L}\horizon e^{\hat{L}\horizon}  + K\hat{L}\horizon^2 e^{\hat{L}\horizon} + 2K\right)}, \horizon\right\}\frac{\normFun[2]{F(f)-F(\tilde{f})}{L^{\infty}(\compactInit)}}{2\sqrt{\odedim}(2+\hat{L}\horizon)}\\
        \overset{\eqref{eq:F_f_relation}}&{=}\min\left\{\frac{\normFun[2]{f-\tilde{f}}{L^{\infty}(\compactInit)}}{2\sqrt{\odedim}\hat{L}\left(2D \hat{L}\horizon e^{\hat{L}\horizon}  + K\hat{L}\horizon^2 e^{\hat{L}\horizon} + 2K\right)}, \horizon\right\}\frac{\normFun[2]{f-\tilde{f}}{L^{\infty}(\compactInit)}}{2\sqrt{\odedim}(2+\hat{L}\horizon)},
    \end{aligned}
\end{equation}
where $\hat{L}=\sqrt{L^2+1}$. Combining \eqref{eq:upper_lip_higher} and \eqref{eq:lower_lip_higher} concludes the proof.

\subsection{Proof of Lemma \ref{lm:dist_lower_sublinear_ODE_local}}\label{app:dist_lower_sublinear_ODE_local}
    First, we pick $x_0 \in \compactInit$ such that $$|f(x_0)-\tilde{f}(x_0)| = \normFun{f-\tilde{f}}{L^{\infty}(\compactInit)}$$
    thanks to the continuity of $f,\tilde{f}$ and compactness of $\compactInit$. Now, pick $\tilde{x}\in \behavior[\compactInit]{\diffeq[\tilde{f}]}$ such that $\tilde{x}(0)=x_0$. We then have 
    \begin{equation}
        \tilde{x}(t) = x_0+\int_{0}^{t} \tilde{f}(\tilde{x}(s))ds.
    \end{equation}
    Taking the absolute value and noting that $\tilde{f}\in \holder[\R,\R]{0,\alpha}{L,K}$, we have 
    \begin{equation}\label{eq:holder_int_bound_1}
    |\tilde{x}(t)-x_0| \leq \int_{0}^t |f(\tilde{x}(s))|ds \leq Kt+\int_{0}^t L|\tilde{x}(s)|^{\alpha}ds,
\end{equation}
    which implies that 
    \begin{equation*}
        |\tilde{x}(t)|\leq |x_0|+ Kt+\int_{0}^t L|\tilde{x}(s)|^{\alpha}ds.
    \end{equation*}
    By Lemma \ref{lm:gronwall_pach}, we have
\begin{equation}\label{eq:holder_int_bound_2}
    |\tilde{x}(t)|\leq|x_0|+ Kt+\int_{0}^t L|\tilde{x}(s)|^{\alpha} \leq \left((|x_0|+Kt)^{1-\alpha}+(1-\alpha)Lt\right)^{\frac{1}{1-\alpha}}.
\end{equation}
Using \eqref{eq:holder_int_bound_2} in \eqref{eq:holder_int_bound_1}, we obtain
\begin{equation}
    |\tilde{x}(t)-x_0| \leq \left((|x_0|+Kt)^{1-\alpha}+(1-\alpha)Lt\right)^{\frac{1}{1-\alpha}} - |x_0|.
\end{equation}
Now, set 
\begin{equation}
    \phi(t)= \left((|x_0|+Kt)^{1-\alpha}+(1-\alpha)Lt\right)^{\frac{1}{1-\alpha}} - |x_0|.
\end{equation}
Then $\phi(0)=0$ and for all $\xi\in (0,\horizon)$,
\begin{equation}
    \begin{aligned}
        \dot{\phi}(\xi) &= \frac{1}{1-\alpha}\left((|x_0|+K\xi)^{1-\alpha}+(1-\alpha)L\xi\right)^{\frac{\alpha}{1-\alpha}} \left(\frac{(1-\alpha)K}{(|x_0|+K\xi)^{\alpha}}+(1-\alpha)L\right)\\
        &=\frac{1}{1-\alpha}\left(1+\frac{(1-\alpha)L\xi}{(|x_0|+K\xi)^{1-\alpha}}\right)^{\frac{\alpha}{1-\alpha}} \left((1-\alpha)K+(1-\alpha)L(|x_0|+K\xi)^{\alpha}\right)\\
        &\leq \frac{1}{1-\alpha}\left(1+\frac{(1-\alpha)L\xi^\alpha}{K^{1-\alpha}}\right)^{\frac{\alpha}{1-\alpha}} \left((1-\alpha)K+(1-\alpha)L(|x_0|+K\xi)^{\alpha}\right)\\
        &\leq \frac{1}{1-\alpha}\left(1+\frac{(1-\alpha)L\horizon^\alpha}{K^{1-\alpha}}\right)^{\frac{\alpha}{1-\alpha}} \left((1-\alpha)K+(1-\alpha)L(D+K\horizon)^{\alpha}\right)\\
        &\eqdef c(\alpha, L,K,\horizon,D).
    \end{aligned}
\end{equation}
By Lagrange's mean value theorem, we therefore have for all $t\in (0,\horizon]$,
\begin{equation*}
    0\leq \phi(t) \overset{\exists\xi\in (0,t)}{=} \dot{\phi}(\xi)t\leq c(\alpha, L,K,\horizon,D)t,
\end{equation*}
which implies 
\begin{equation}
    |\tilde{x}(t)-x_0|\leq c(\alpha, L,K,\horizon,D)t.
\end{equation}
Now, set 
    \begin{equation}\label{eq:t_0_holder}
        t_0\defeq \min\left\{\frac{\normFun{f-\tilde{f}}{L^{\infty}(\compactInit)}^{\frac{1}{\alpha}}}{(4L)^{\frac{1}{\alpha}}c(\alpha, L,K,\horizon,D)} ,T\right\},
    \end{equation}
we obtain 
    \begin{equation}
        |\tilde{x}(t)-x_0|\leq \frac{\normFun{f-\tilde{f}}{L^{\infty}(\compactInit)}^{\frac{1}{\alpha}}}{(4L)^{\frac{1}{\alpha}}}, \quad \text{for }t\in [0,t_0].
    \end{equation}
    Define $h\defeq f-\tilde{f}$, then $h\in \holder{0,\alpha}{2L,2K}$. We have 
    \begin{equation}\label{eq:hxt_lower_ext_holder}
        |h(\tilde{x}(t))| \geq |h(x_0)| - 2L|\tilde{x}(t)-x_0|^{\alpha}\geq \frac{\normFun{f-\tilde{f}}{L^{\infty}(\compactInit)}}{2}, \quad \text{for }t\in [0,t_0].
    \end{equation}
    Now, for arbitrarily fixed $x\in \behavior[\compactInit]{\diffeq[f]}$, define $y = x-\tilde{x}$. Then, $y$ satisfy the ODE
    \begin{equation}
        \dot{y}(t) = f(y(t)+\tilde{x}(t)) - \tilde{f}(\tilde{x}(t)).
    \end{equation}
    Integrating from $0$ to $t$, we have 
    \begin{equation}
        y(t) = y(0) + \int_{0}^{t}\left(f(y(s)+\tilde{x}(s)) - f(\tilde{x}(s))\right)ds + \int_{0}^{t}\left(f(\tilde{x}(s)) - \tilde{f}(\tilde{x}(s))\right)ds.
    \end{equation}
    Taking the absolute value, applying triangle inequality, and noticing $f\in \mathcal{H}(0,\alpha,L,K)$ gives 
    \begin{equation}
    \begin{aligned}
        \normFun{y}{L^{\infty}([0,\horizon])}&\geq |y(t)|\geq \left|\int_{0}^{t}\left(f(\tilde{x}(s)) - \tilde{f}(\tilde{x}(s))\right)ds\right| - |y(0)| - L\int_{0}^{t}|y(s)|^{\alpha}ds\\
        &\geq \left|\int_{0}^{t}\left(f(\tilde{x}(s)) - \tilde{f}(\tilde{x}(s))\right)ds\right| - \normFun{y}{L^{\infty}([0,\horizon])} - L\horizon\normFun{y}{L^{\infty}([0,\horizon])}^{\alpha},
    \end{aligned}
    \end{equation}
    which in turn gives 
    \begin{equation}
        \normFun{y}{L^{\infty}([0,\horizon])} \geq  \min\left\{\frac{1}{4}\left|\int_{0}^{t}h(\tilde{x}(s))ds\right|,\left(\frac{1}{2L\horizon}\left|\int_{0}^{t}h(\tilde{x}(s))ds\right|\right)^{\frac{1}{\alpha}}\right\},~ \text{for all }t\in[0,\horizon].
    \end{equation}
    Since $h\circ \tilde{x}$ is continuous, by \eqref{eq:hxt_lower_ext_holder}, $h(\tilde{x}(t))$ must remain positive or negative for $t\in [0,t_0]$. In either case, we have 
\begin{equation*}
    \begin{aligned}
        \normFun{y}{L^{\infty}([0,\horizon])} &\geq  \min\left\{\frac{1}{4}\left|\int_{0}^{t_0}h(\tilde{x}(s))ds\right|,\left(\frac{1}{2L\horizon}\left|\int_{0}^{t_0}h(\tilde{x}(s))ds\right|\right)^{\frac{1}{\alpha}}\right\}\\
        \overset{\eqref{eq:hxt_lower_ext_holder}}&{\geq} \min\left\{\frac{t_0}{8}\normFun{f-\tilde{f}}{L^{\infty}(\compactInit)},\left(\frac{t_0}{4L\horizon}\normFun{f-\tilde{f}}{L^{\infty}(\compactInit)}\right)^{\frac{1}{\alpha}}\right\}\\
        \overset{\eqref{eq:t_0_holder}}&{\geq} C\min\left\{\normFun{f-\tilde{f}}{L^{\infty}(\compactInit)},\normFun{f-\tilde{f}}{L^{\infty}(\compactInit)}^{\frac{\alpha+1}{\alpha}},\normFun{f-\tilde{f}}{L^{\infty}(\compactInit)}^{\frac{1}{\alpha}},\normFun{f-\tilde{f}}{L^{\infty}(\compactInit)}^{\frac{\alpha+1}{\alpha^2}}\right\},
    \end{aligned}
\end{equation*}    
    where $C>0$ is a constant depending only on $\alpha, L,K,\horizon,D$. Since this holds for arbitrary $x\in \behavior[\compactInit]{\diffeq[f]}$, we obtain
\begin{equation*}
\begin{aligned}
    &\inf_{x\in \behavior[\compactInit]{\diffeq[f]}}\normFun{x-\tilde{x}}{L^{\infty}([0,\horizon])} \\
    &\geq C\min\left\{\normFun{f-\tilde{f}}{L^{\infty}(\compactInit)},\normFun{f-\tilde{f}}{L^{\infty}(\compactInit)}^{\frac{\alpha+1}{\alpha}},\normFun{f-\tilde{f}}{L^{\infty}(\compactInit)}^{\frac{1}{\alpha}},\normFun{f-\tilde{f}}{L^{\infty}(\compactInit)}^{\frac{\alpha+1}{\alpha^2}}\right\}.    
\end{aligned}
\end{equation*}
Taking the supremum over $\behavior[\compactInit]{\diffeq[\tilde{f}]}$ and by symmetry, we arrive at the desired result.

\subsection{Proof of Lemma \ref{lm:dist_lower_superHolderODE}}\label{app:dist_lower_superHolderODE}
Before proving Lemma \ref{lm:dist_lower_superHolderODE}, we need two auxiliary lemmata. The first lemma shows that under the condition for $\horizon$ in Lemma \ref{lm:dist_lower_superHolderODE}, solutions of each superlinear Hölder ODE in $\odeClass{\holder{k,\alpha}{L,K}}$ is indeed well-defined on $[0,\horizon]$.
\begin{lemma}\label{lm:holder_explode_time}
    For $D,L,K>0$, $k\in \Nplus$, $\alpha\in (0,1]$, let $$\horizon< ((K+L)e)^{-1}\phi_{D,k,\alpha}^{-1}\left(\frac{1}{k+\alpha-1}\right),$$ where $\phi_{D,k,\alpha}(r) = r(r+D)^{k+\alpha-1}$, for $r\in \R_+$. Consider a compact set $\compactInit\subset \ball{D} \defeq  \{x\in \R\mid |x|\leq D\}$. Then, for each $\diffeq[f]\in \odeClass{\holder{k,\alpha}{L,K}}$ with $f\in \holder{k,\alpha}{L,K}$, the solutions to $\diffeq[f]$ with respect to initial values in $\compactInit$ are well defined on $[0,\horizon]$. Moreover, for all $x\in \behavior[\compactInit]{\diffeq[f]}$, we have 
   \begin{equation}
       |x(t)| \leq (D+(K+L)e\horizon)\left(1-(k+\alpha-1)(K+L)e(D+(K+L)e\horizon)^{k+\alpha-1}\right)^{-\frac{1}{k+\alpha-1}},
   \end{equation}
   for all $t\in [0,\horizon]$.
\end{lemma}
\begin{proof}
    For each $\diffeq[f]\in \odeClass{\holder{k,\alpha}{L,K}}$ and $x\in \behavior[\compactInit]{\diffeq[f]}$ with $x(0)=x_0 \in \compactInit$, we have
    \begin{equation*}
        x(t) = x_0 + \int_{0}^t f(x(s))ds,
    \end{equation*}
where $f \in \holder{k,\alpha}{L,K}$. Taking the absolute value, we obtain
\begin{equation*}
\begin{aligned}
    |x(t)| &\leq  |x_0| + \int_{0}^t |f(x(s))|ds\\
    \overset{\exists \xi(s)\in (0,x(s))}&{\leq} |x_0| + \int_{0}^t \left(\sum_{i=0}^{k-1}\frac{|f^{(i)}(0)|}{i!} |x(s)|^i + \frac{|f^{(k)}(\xi(s))|}{k!}|x(s)|^k\right)ds\\
    \overset{f \in \holder{k,\alpha}{L,K}}&{\leq} |x_0| + \int_{0}^t \left(\sum_{i=0}^{k-1}\frac{|f^{(i)}(0)|}{i!} |x(s)|^i + \frac{|f^{(k)}(0)|+L|\xi(s)|^{\alpha}}{k!}|x(s)|^k\right)ds\\
    &\leq |x_0| + \int_{0}^t \left(\sum_{i=0}^{k}\frac{|f^{(i)}(0)|}{i!} |x(s)|^i + \frac{L}{k!}|x(s)|^{k+\alpha}\right)ds\\
    &\leq |x_0| + \int_{0}^t (K+L)\left(\sum_{i=0}^{k}\frac{1}{i!} \right)(1+|x(s)|^{k+\alpha})ds\\
    &\overset{\tilde{C}\defeq (K+L)e}{\leq} (|x_0|+\tilde{C}t) + \tilde{C}\int_{0}^t|x(s)|^{k+\alpha}ds.
\end{aligned}
\end{equation*}
By Lemma \ref{lm:gronwall_pach}, we have 
\begin{equation*}
\begin{aligned}
    |x(t)| &\leq (|x_0|+\tilde{C}t)\left(1-(k+\alpha-1)\tilde{C}t(|x_0|+\tilde{C}t)^{k+\alpha-1}\right)^{-\frac{1}{k+\alpha-1}}\\
    &\leq (D+\tilde{C}\horizon)\left(1-(k+\alpha-1)\tilde{C}\horizon(D+\tilde{C}\horizon)^{k+\alpha-1}\right)^{-\frac{1}{k+\alpha-1}}<\infty,
\end{aligned}
\end{equation*}
for $0\leq t\leq \horizon < \tilde{C}^{-1}\phi_{D,k,\alpha}^{-1}\left(\frac{1}{k+\alpha-1}\right)$, where $\phi_{D,k,\alpha}(r) = r(r+D)^{k+\alpha-1}$.
\end{proof}
The next lemma shows that if the solutions to superlinear Hölder ODEs in $\odeClass{\holder{k,\alpha}{L,K}}$ are uniformly bounded, then the superlinear Hölder ODEs are nothing but Lipschitz ODEs.
\begin{lemma}\label{lm:holder2lip}
    Assume $D,L,K>0$, $k\in \Nplus$, $\alpha\in (0,1]$, and $B>0$. For all $f\in \mathcal{H}(k,\alpha,L,K)$, we have 
    \begin{equation}
        \operatorname{Lip}(f|_{[-B,B]}) \leq \max\{L,1\}(B+K)(B+1)^{k-1}.
    \end{equation}
\end{lemma}
\begin{proof}
    For all $|x|\leq B$, we have 
    \begin{equation*}
    \begin{aligned}
        |f^{(k-1)}(x)|&\leq |f^{(k-1)}(0)|+|f^{(k)}(\xi_x)| B\\
        &\leq K+ B(|f^{(k)}(\xi_x)-f^{(k)}(0)| + |f^{(k)}(0)|)\\
        &\leq K+BK + B^{1+\alpha}L\\
        &\leq K+BK + B(1+B)L\\
        &\leq \max\{L,1\}(B+K)(B+1).
    \end{aligned}
    \end{equation*}
    This further implies 
    \begin{equation*}
        \begin{aligned}
        |f^{(k-2)}(x)|&\leq |f^{(k-1)}(0)|+|f^{(k-1)}(\xi_x)| B\\
        &\leq K + \max\{L,1\}(B+K)(B+1) B\\
        &\leq \max\{L,1\}(B+K)(B+1)^2, \quad \text{for all }|x|\leq B.
        \end{aligned}
    \end{equation*}
    Vice versa, we finally obtain
    \begin{equation*}
        |f^{(1)}(x)|\leq \max\{L,1\}(B+K)(B+1)^{k-1}, \quad \text{for all }|x|\leq B,
    \end{equation*}
    which implies the desired result.
\end{proof}

\textit{Proof of Lemma \ref{lm:dist_lower_superHolderODE}:}     Lemma \ref{lm:holder_explode_time} implies that the absolute values of solutions to superlinear Hölder ODEs in $\odeClass{\holder{k,\alpha}{L,K}}$ are universally upper-bounded by 
    \begin{equation*}
        B\defeq (D+(K+L)eT)\left(1-(k+\alpha-1)(K+L)e(D+(K+L)eT)^{k+\alpha-1}\right)^{-\frac{1}{k+\alpha-1}}.
    \end{equation*}
    Then, Lemma \ref{lm:holder2lip} implies that $f$ acting on $x(t)$ is a Lipschitz function with 
    \begin{equation*}
        \operatorname{Lip}(f|_{[-B,B]}) \leq \max\{L,1\}(B+K)(B+1)^{k-1}.
    \end{equation*}
    Thus, we can apply Lemma \ref{lm:dist_lower_LipODE_local} with $\odedim=1$, $\tilde{L}=\max\{L,1\}(B+K)(B+1)^{k-1}$ to obtain the desired results.

\subsection{Proof Theorem \ref{th:ent_lin_ode}}\label{app:ent_lin_ode}
To prove Theorem \ref{th:ent_lin_ode}, we first need an auxiliary result on the metric entropy of bounded balls in the space of matrices. 
\begin{lemma}\label{lm:mat_ent}
    Define 
    \begin{equation}
        M^{\odedim,\odedim}_K\defeq \{A\in \R^{\odedim\times \odedim}\mid \norm{A}_2\leq K\}.
    \end{equation}
    Then, for all $\epsilon>0$, 
    \begin{equation}
        \log\covering (\epsilon;M^{\odedim,\odedim}_K,\norm{\cdot}_2)\asymp \odedim^2\log (\epsilon^{-1}).
    \end{equation}
\end{lemma}
\begin{proof}
    We define a mapping $V:M^{\odedim,\odedim}_K\rightarrow \R^{\odedim^2}$ such that 
    \begin{equation}
        (V(A))_{(i-1)\odedim+j} = A_{i,j} \text{ for }i,j\in\{1,2,\dots,\odedim\}.
    \end{equation}
    Then, we have 
    \begin{equation}\label{eq:bounds_VA_A}
        \norm{A}_2 \overset{\text{Lemma} \ref{lm:mat_F_2}}{\leq} \norm{V(A)}_2 = \norm{A}_F \overset{\text{Lemma} \ref{lm:mat_F_2}}{\leq} \sqrt{\odedim} \norm{A}_2
    \end{equation}
    and thus $V$ is an isometric isomorphism between $\left(M^{\odedim,\odedim}_K,\norm{\cdot}_F\right)$ and $\left(V(M^{\odedim,\odedim}_K),\norm{\cdot}_2\right)$. Now, for arbitrary $C>0$, define 
    $$\ball{C}^{\odedim^2} \defeq \{x\in \R^{\odedim^2}\mid \norm{x}_2\leq  C\}.$$
    By \eqref{eq:bounds_VA_A}, we obtain 
    \begin{equation}\label{eq:subset_A_VA}
        \ball{K}^{\odedim^2} \subset V(M^{\odedim,\odedim}_K) \subset \ball{\sqrt{\odedim}K}^{\odedim^2},
    \end{equation}
    which implies 
    \begin{equation}\label{eq:ent_M_F}
        \begin{aligned}
            \odedim^2\log\left(\epsilon^{-1}\right) \overset{\text{Lemma \ref{lm:ent_ball}}}&{\asymp}\log\packing(2\epsilon;\ball{K}^{\odedim^2},\norm{\cdot}_2)\\
            \overset{\text{\eqref{eq:subset_A_VA}, Corollary \ref{col:subset_packing_relation}}}&{\leq}\log\packing(2\epsilon;V(M^{\odedim,\odedim}_K),\norm{\cdot}_2)\\
            \overset{\text{Lemma \ref{lm:number_relation}}}&{\leq} \log\covering(\epsilon;V(M^{\odedim,\odedim}_K),\norm{\cdot}_2)  \\
            \overset{\text{Lemma \ref{lm:ent_isomorphism}}}&{=} \log\covering(\epsilon;M^{\odedim,\odedim}_K,\norm{\cdot}_F)\\
            \overset{\text{Lemma \ref{lm:number_relation}}}&{\leq}  \log\packing(\epsilon;V(M^{\odedim,\odedim}_K),\norm{\cdot}_2)\\
            \overset{\text{\eqref{eq:subset_A_VA}, Corollary \ref{col:subset_packing_relation}}}&{\leq}\log\packing(\epsilon;\ball{\sqrt{\odedim}K}^{\odedim^2},\norm{\cdot}_2)\\
            \overset{\text{Lemma \ref{lm:ent_ball}}}&{\asymp}\odedim^2\log\left(\epsilon^{-1}\right).
        \end{aligned}
    \end{equation}
    Finally, combining by \eqref{eq:bounds_VA_A}, \eqref{eq:ent_M_F}, and applying Corollary \ref{col:covering_relation_equiv_norm_same_set}, we have 
    \begin{equation*}
        \log\covering(\epsilon;M^{\odedim,\odedim}_K,\norm{\cdot}_F)\asymp\odedim^2\log\left(\epsilon^{-1}\right) \Rightarrow \log\covering(\epsilon;M^{\odedim,\odedim}_K,\norm{\cdot}_2)\asymp\odedim^2\log\left(\epsilon^{-1}\right).
    \end{equation*}
\end{proof}

\textit{Proof of Theorem \ref{th:ent_lin_ode}:} Consider the mapping 
\begin{equation*}
    \begin{aligned}
        \mathcal{G}: M^{\odedim,\odedim}_K &\rightarrow \odeClass{\lin[\R^{\odedim},\R^\odedim ]{K}}\\
         A &\rightarrow \diffeq[A]
    \end{aligned}.
\end{equation*}
By Theorem \ref{th:dist_lu_linearODE_const}, there exist constants $0<c\leq C$ depending only on $\horizon,K,D$, such that for every $A_1,A_2\in  M^{\odedim,\odedim}_K$,we have 
\begin{equation*}
    c\norm{A_1-A_2}_2\leq \hd[\ball{D}][\normFun{\cdot}{L^{\infty}}]{\mathcal{G}(A_1)}{\mathcal{G}(A_2)}\leq C \norm{A_1-A_2}_2.
\end{equation*}
Then, applying Lemma \ref{lm:mat_ent} and \ref{lm:covering_relation_equiv_norm} gives the desired result.

\subsection{Proof of Theorem \ref{th:ent_lip_ode}}\label{app:ent_lip_ode}
To prove Theorem \ref{th:ent_lip_ode}, we need a few auxiliary results on the metric entropy of the class of Lipschitz functions. First, we introduce the following class of Lipschitz functions
\begin{equation}
    \mathcal{F}_{L,C}[a,b]\defeq \{f:[a,b]\rightarrow \R\mid \operatorname{Lip}(f)\leq L;|f(x)|\leq C, \forall x\in [a,b]\}.
\end{equation}
\begin{lemma}[{\cite[Equation (11), Chapter 7]{tikhomirovEEntropyECapacitySets1993}}]\label{lm:ent_lip_LC}
Arbitrarily fix $L,C>0$, $b>a$. We have 
\begin{equation}
    \log^{(2)}\covering(\epsilon;\mathcal{F}_{L,C}[a,b],\normFun{\cdot}{L^{\infty}([a,b])})\asymp \log\left(\epsilon^{-1}\right).
\end{equation}
\end{lemma}
Based on Lemma \ref{lm:ent_lip_LC}, we obtain the following results for the metric entropy of $\mathcal{F}_{L,K,D}$.
\begin{lemma}\label{lm:ent_lip_FLKD}
    Arbitrarily fix $L,K,D>0$. We have 
    \begin{equation}
        \log^{(2)}\covering(\epsilon;\mathcal{F}_{L,K,D},\normFun{\cdot}{L^{\infty}(\R)})\asymp \log\left(\epsilon^{-1}\right).
    \end{equation}
\end{lemma}
\begin{proof}
    Consider the linear mapping 
    \begin{equation*}
    \begin{aligned}
        E: \mathcal{F}_{L,\min\left\{\frac{LD}{2},K\right\}}\left[0,\frac{D}{2}\right]&\rightarrow \mathcal{F}_{L,K,D}\\
        f &\rightarrow E(f),
    \end{aligned}   
    \end{equation*}
    where 
    \begin{equation*}
        E(f)(x) = \left\{ \begin{array}{cc}
            \frac{f(0)}{D}(x+D) & x\in [-D,0]\\
             f(x)& x\in \left[0,\frac{D}{2}\right]  \\
            -\frac{2f(\frac{D}{2})}{D}(x-D) & x\in \left[\frac{D}{2},D\right]\\
            0 & \text{else}
        \end{array}.\right.
    \end{equation*}
    We have 
    \begin{equation*}
        \begin{aligned}
            &|E(f)(0)| = |f(0)|\leq \min\left\{\frac{LD}{2},K\right\}\leq K,\\
            &\operatorname{Lip(E(f))}\leq \max\left\{\operatorname{Lip}(f), \left|\frac{f(0)}{D}\right|,\left|\frac{2f(\frac{D}{2})}{D}\right| \right\}\leq \max\left\{L, \frac{\min\left\{2\frac{LD}{2},K\right\}}{D} \right\}\leq L,\\
            &\operatorname{supp}(E(f))\subset[-D,D],\\
            & \normFun{E(f)}{L^{\infty}(\R)} = \normFun{f}{L^{\infty}\left(\left[0,\frac{D}{2}\right]\right)},
        \end{aligned}
    \end{equation*}
    and thus $E$ is a well-defined isometry from $\left(\mathcal{F}_{L,\min\left\{\frac{LD}{2},K\right\}}\left[0,\frac{D}{2}\right],\normFun{\cdot}{L^{\infty}\left(\left[0,\frac{D}{2}\right]\right)}\right)$ to $(\mathcal{F}_{L,K,D},\normFun{\cdot}{L^{\infty}(\R)})$. Thus, we obtain
    \begin{equation}\label{eq:lower_ent_FDL}
    \begin{aligned}
    \log^{(2)}\covering\left(\epsilon;\mathcal{F}_{L,K,D},\normFun{\cdot}{L^{\infty}(\R)}\right) \overset{\text{Lemma \ref{lm:number_relation}}}&{\geq}  \log^{(2)}\packing\left(2\epsilon;\mathcal{F}_{L,K,D},\normFun{\cdot}{L^{\infty}(\R)}\right) \\
    \overset{\text{Lemma \ref{lm:isometry_packing_relation}}}&{\geq}  \log^{(2)}\packing\left(2\epsilon;\mathcal{F}_{L,\min\left\{\frac{LD}{2},K\right\}}\left[0,\frac{D}{2}\right],\normFun{\cdot}{L^{\infty}\left(\left[0,\frac{D}{2}\right]\right)}\right)\\
    \overset{\text{Lemma \ref{lm:number_relation}}}&{\geq}\log^{(2)}\covering\left(2\epsilon;\mathcal{F}_{L,\min\left\{\frac{LD}{2},K\right\}}\left[0,\frac{D}{2}\right],\normFun{\cdot}{L^{\infty}\left(\left[0,\frac{D}{2}\right]\right)}\right)\\
    \overset{\text{Lemma \ref{lm:ent_lip_LC}}}&{\asymp} \log\left(\epsilon^{-1}\right).
    \end{aligned}
    \end{equation}
    On the other hand, consider the mapping 
    \begin{equation*}
    \begin{aligned}
        I: \mathcal{F}_{L,K,D}&\rightarrow \mathcal{F}_{L,K+LD}[-D,D]\\
        f   &\rightarrow I(f),
    \end{aligned}
    \end{equation*}
    where 
    \begin{equation*}
        I(f)= f(x) \quad \text{ for }x\in [-D,D].
    \end{equation*}
    Then $I$ is a well-defined isometry from $(\mathcal{F}_{L,K,D},\normFun{\cdot}{L^{\infty}(\R)})$ to $\left(\mathcal{F}_{L,K+LD}[-D,D], \normFun{\cdot}{L^{\infty}([-D,D])}\right)$. Thus, 
    \begin{equation}\label{eq:upper_ent_FDL}
    \begin{aligned}
    \log^{(2)}\covering\left(\epsilon;\mathcal{F}_{L,K,D},\normFun{\cdot}{L^{\infty}(\R)}\right) \overset{\text{Lemma \ref{lm:number_relation}}}&{\leq}  \log^{(2)}\packing\left(\epsilon;\mathcal{F}_{L,K,D},\normFun{\cdot}{L^{\infty}(\R)}\right) \\
    \overset{\text{Lemma \ref{lm:isometry_packing_relation}}}&{\leq} \log^{(2)}\packing\left(\epsilon;\mathcal{F}_{L,K+LD}[-D,D], \normFun{\cdot}{L^{\infty}([-D,D])}\right)\\
    \overset{\text{Lemma \ref{lm:number_relation}}}&{\leq}  \log^{(2)}\covering\left(\frac{\epsilon}{2};\mathcal{F}_{L,K+LD}[-D,D], \normFun{\cdot}{L^{\infty}([-D,D])}\right)\\
    \overset{\text{Lemma \ref{lm:ent_lip_LC}}}&{\asymp} \log\left(\epsilon^{-1}\right).
    \end{aligned}
    \end{equation}
    Combining \eqref{eq:lower_ent_FDL} and \eqref{eq:upper_ent_FDL} concludes the proof.
\end{proof}

\textit{Proof of Theorem \ref{th:ent_lip_ode}:} 
Appling Theorem \ref{th:dist_lu_LipODE} for $\odedim=1$ and setting $\compactInit=\ball{D}$, we have that for every two Lipschitz
ODEs $\diffeq[f],\diffeq[\tilde{f}]\in \odeClass{\mathcal{F}_{L,K,D}}$, there exists $C_1,C_2,C_3>0$ depending only on $\horizon,L,D,K$, such that  
\begin{equation}\label{eq:l_use_ent_LipODE}
\begin{aligned}
    \hd[\ball{D}][\normFun{\cdot}{L^{\infty}}]{\diffeq[f]}{\diffeq[\tilde{f}]}
    &\geq \min\left\{C_1\normFun{f-\tilde{f}}{L^{\infty}(\ball{D})}^2, C_2\normFun{f-\tilde{f}}{L^{\infty}(\ball{D})}\right\}\\
    \overset{\operatorname{supp}(f),\operatorname{supp}(\tilde{f})\subset \ball{D}}&{=}\min\left\{C_1\normFun{f-\tilde{f}}{L^{\infty}(\R)}^2, C_2\normFun{f-\tilde{f}}{L^{\infty}(\R)}\right\}
\end{aligned}
\end{equation}
and
    \begin{equation}\label{eq:u_use_ent_LipODE}
    \begin{aligned}
       \hd[\ball{D}][\normFun{\cdot}{L^{\infty}}]{\diffeq[f]}{\diffeq[\tilde{f}]} \overset{\widehat{\compactInit} \defeq \{x\in \R^{\odedim}\mid \norm{x}_2\leq (K\horizon+D) e^{L\horizon}\}}&{\leq} C_3\normFun{f-\tilde{f}}{L^{\infty}(\widehat{\compactInit})}\\
    \overset{\operatorname{supp}(f),\operatorname{supp}(\tilde{f})\subset \ball{D}\subset \widehat{\compactInit}}&{=}C_3\normFun{f-\tilde{f}}{L^{\infty}(\R)}.
    \end{aligned}
    \end{equation}
\eqref{eq:l_use_ent_LipODE} implies that 
\begin{equation}\label{eq:l_use_ent_LipODE_2}
    \normFun{f-\tilde{f}}{L^{\infty}(\R)} \leq C_4 \max\left\{\hd[\ball{D}][\normFun{\cdot}{L^{\infty}}]{\diffeq[f]}{\diffeq[\tilde{f}]}, \sqrt{\hd[\ball{D}][\normFun{\cdot}{L^{\infty}}]{\diffeq[f]}{\diffeq[\tilde{f}]}}\right\},
\end{equation}
for a constant $C_4>0$. Now, consider an $\epsilon$-covering $\{\diffeq[f_1],\diffeq[f_2],\dots\diffeq[f_{N}]\}$ of $(\odeClass{\mathcal{F}_{L,K,D}},\allowbreak \hd[\ball{D}][\normFun{\cdot}{L^{\infty}}]{\cdot}{\cdot})$ with 
$N=\covering(\epsilon;\odeClass{\mathcal{F}_{L,K,D}},\hd[\ball{D}][\normFun{\cdot}{L^{\infty}}]{\cdot}{\cdot})$ and corresponding $\{f_1,f_2,\dots,f_{N}\}\subset \mathcal{F}_{L,K,D}$. According to \eqref{eq:l_use_ent_LipODE_2}, this implies that for every $f\in \mathcal{F}_{L,K,D}$, we can find $i\in \{1,2,\dots,N\}$, such that 
\begin{equation*}
    \normFun{f-f_i}{L^{\infty}(\R)} \leq C_4\max \{\epsilon,\epsilon^{\frac{1}{2}}\},
\end{equation*}
which in turn gives that $\{f_1,f_2,\dots,f_N\}$ is a $(C_4\max \{\epsilon,\epsilon^{\frac{1}{2}}\})$-covering of $(\mathcal{F}_{L,K,D},\normFun{\cdot}{L^{\infty}(\R)})$ and thus 
\begin{equation}\label{eq:l_ent_LipODE_est}
    \begin{aligned}
        \log^{(2)}\covering(\epsilon;\odeClass{\mathcal{F}_{L,K,D}},\hd[\ball{D}][\normFun{\cdot}{L^{\infty}}]{\cdot}{\cdot}) &\geq \log^{(2)}\covering(C_4\max \{\epsilon,\epsilon^{\frac{1}{2}}\};\mathcal{F}_{L,K,D} \normFun{\cdot}{L^{\infty}(\R)})\\
        \overset{\text{Lemma \ref{lm:ent_lip_FLKD}}}&{\gtrsim}\frac{1}{2}\log\left(\epsilon^{-1}\right).
    \end{aligned}
\end{equation}
On the other hand, consider an $\epsilon/C_3$-covering $\{f_1,f_2\dots,f_N\}$ of $(\mathcal{F}_{L,K,D},\normFun{\cdot}{L^{\infty}(\R)})$ with $N=\covering(\epsilon/C_3,\mathcal{F}_{L,K,D},\normFun{\cdot}{L^{\infty}(\R)})$. By \eqref{eq:u_use_ent_LipODE}, the corresponding $\{\diffeq[f_1],\diffeq[f_2],\dots\diffeq[f_{N}]\}$ is an $\epsilon$-covering of $(\odeClass{\mathcal{F}_{L,K,D}},\allowbreak \hd[\ball{D}][\normFun{\cdot}{L^{\infty}}]{\cdot}{\cdot})$, which gives
\begin{equation}\label{eq:u_ent_LipODE_est}
    \begin{aligned}
        \log^{(2)}\covering(\epsilon;\odeClass{\mathcal{F}_{L,K,D}},\hd[\ball{D}][\normFun{\cdot}{L^{\infty}}]{\cdot}{\cdot}) &\leq \log^{(2)}\covering(\epsilon/C_3;\mathcal{F}_{L,K,D} \normFun{\cdot}{L^{\infty}(\R)})\\
        \overset{\text{Lemma \ref{lm:ent_lip_FLKD}}}&{\lesssim}\log\left(\epsilon^{-1}\right).
    \end{aligned}
\end{equation}
Combining \eqref{eq:l_ent_LipODE_est} and \eqref{eq:u_ent_LipODE_est} concludes the proof.

\end{appendices}

\clearpage
\printbibliography[
heading=bibintoc,
title={References}
]
\end{document}